\documentclass[twoside,11pt]{article}

\usepackage{jmlr2e}
\usepackage{hhline,bm,epsfig,amsmath,array,xcolor,booktabs,upgreek,subfigure}
\usepackage[font=small,labelfont=bf]{caption}
\usepackage{mathptmx}
\usepackage{floatrow}
\usepackage{mathtools}
\DeclarePairedDelimiter\ceil{\lceil}{\rceil}
\DeclarePairedDelimiter\floor{\lfloor}{\rfloor}

\usepackage{hyperref}

\newcommand{\overbar}[1]{\mkern 1.5mu\overline{\mkern-1.5mu#1\mkern-1.5mu}\mkern 1.5mu}
\def\Dice{{\rm Dice}}
\def\mDice{{\rm mDice}}
\def\mIoU{{\rm mIoU}}

\def\KL{{\rm KL}}
\def\IoU{{\rm IoU}}
\def\argmax{\mathop{\rm argmax}}
\def\argmin{\mathop{\rm argmin}}

\def\Var{\mathop{\rm Var}}

\usepackage[linesnumbered,ruled]{algorithm2e}
\usepackage{amsfonts} 
\DeclareMathSymbol{\sm}{\mathbin}{AMSa}{"39}
\usepackage{amsmath}
\usepackage{lastpage}
\def\independenT#1#2{\mathrel{\setbox0\hbox{$#1#2$}%
\copy0\kern-\wd0\mkern4mu\box0}}

\newcommand{\gray}[1]{\textcolor{gray}{#1}}
\DeclareMathAlphabet{\mathcal}{OMS}{cmsy}{m}{n}
\makeatletter
\newcommand*\bigcdot{\mathpalette\bigcdot@{.7}}
\newcommand*\bigcdot@[2]{\mathbin{\vcenter{\hbox{\scalebox{#2}{$\m@th#1\bullet$}}}}}
\makeatother
\jmlrheading{1}{2020}{1-48}{4/00}{10/00}{meila00a}{Ben Dai and Chunlin Li}

\ShortHeadings{RankSEG}{Ben Dai and Chunlin Li}
\firstpageno{1}
\newcommand{\mb}[1]{\mathbf{\bm{#1}}}

\begin{document}

\jmlrheading{24}{2023}{1-\pageref{LastPage}}{6/22; Revised 5/23}{5/23}{22-0712}{Ben Dai and Chunlin Li}
\ShortHeadings{RankSEG: A Consistent Ranking-based Framework for Segmentation}{Dai and Li}

\title{RankSEG: A Consistent Ranking-based Framework for Segmentation}

\author{\name Ben Dai 
       \email bendai@cuhk.edu.hk \\
       \addr Department of Statistics\\
        The Chinese University of Hong Kong\\
        Hong Kong SAR
      \AND
      \name Chunlin Li 
      \email li000007@umn.edu \\ 
      \addr School of Statistics \\ 
      University of Minnesota \\
      MN 55455 USA}


\editor{Christoph Lampert}

\maketitle

\begin{abstract}
Segmentation has emerged as a fundamental field of computer vision and natural language processing, which assigns a label to every pixel/feature to extract regions of interest from an image/text. To evaluate the performance of segmentation, the Dice and IoU metrics are used to measure the degree of overlap between the ground truth and the predicted segmentation. 
In this paper, we establish a theoretical foundation of segmentation with respect to the Dice/IoU metrics, including the Bayes rule and Dice-/IoU-calibration, analogous to classification-calibration or Fisher consistency in classification. 
We prove that the existing thresholding-based framework with most operating losses are not consistent with respect to the Dice/IoU metrics, and thus may lead to a suboptimal solution. 
To address this pitfall, we propose a novel consistent ranking-based framework, namely \textit{RankDice}/\textit{RankIoU}, inspired by plug-in rules of the Bayes segmentation rule. 
Three numerical algorithms with GPU parallel execution are developed to implement the proposed framework in large-scale and high-dimensional segmentation. 
We study statistical properties of the proposed framework. We show it is Dice-/IoU-calibrated, and its excess risk bounds and the rate of convergence are also provided.
The numerical effectiveness of \textit{RankDice/mRankDice} is demonstrated in various simulated examples and \textit{Fine-annotated CityScapes}, \textit{Pascal VOC} and \textit{Kvasir-SEG} datasets with state-of-the-art deep learning architectures. Python module and source code are available on \textsc{Github} at \url{https://github.com/statmlben/rankseg}.
\end{abstract}


\section{Introduction}
Segmentation is one of the key tasks in the field of computer vision and natural language processing, which groups together similar pixels/features of an input that belong to the same class \citep{ronneberger2015u,badrinarayanan2017segnet}. It has become an essential part of image and text understanding with applications in autonomous vehicles \citep{assidiq2008real}, medical image diagnostics \citep{wang2018deepigeos}, face/fingerprint recognition \citep{xin2018multimodal}, and video action localization \citep{shou2017cdc}. 

The primary aim of segmentation is to label each foreground feature/pixel of an input with a corresponding class. Specifically, for a feature vector or an image $\mb{X} \in \mathbb{R}^d$, a \textit{segmentation function} $\pmb{\delta}: \mathbb{R}^d \to \{0,1\}^d$ yields a predicted segmentation $\pmb{\delta}(\mb{X}) = ( \delta_1(\mb{X}), \cdots, \delta_d(\mb{X}) )^\intercal$, where $\delta_j(\mb{X})$ represents the predicted segmentation for the $j$-th feature $X_j$, and $I(\pmb{\delta}(\mb{X})) = \{j: \delta_j(\mb{X}) = 1; \text{ for } j = 1, \cdots, d \}$ is the index set of the segmented features of $\mb{X}$ provided by $\pmb{\delta}$. Correspondingly, $\mb{Y} \in \{0,1\}^d$ is a feature-wise label of a ground truth segmentation, where $Y_j = 1$ indicates that the $j$-th feature/pixel $X_j$ is segmented, and $I(\mb{Y}) = \big\{ j: \mb{Y}_j = 1; \text{ for } j = 1, \cdots, d \big\}$ is the index set of the ground-truth features. 

To access the performance for a segmentation function $\pmb{\delta}$, the Dice and IoU metrics are introduced and widely used in the literature \citep{milletari2016v}, both of which measure the overlap between the ground truth and the predicted segmentation:
\begin{align}
{\Dice}_\gamma(\pmb{\delta}) & = \mathbb{E} \Big( \frac{2 \big| I(\mb{Y}) \cap I(\pmb{\delta}(\mb{X})) \big| + \gamma }{ | I(\mb{Y}) | + | I(\pmb{\delta}(\mb{X})) | + \gamma } \Big) = \mathbb{E} \Big( \frac{2 \mb{Y}^\intercal \pmb{\delta}(\mb{X}) + \gamma }{ \| \mb{Y} \|_1 + \| \pmb{\delta}(\mb{X}) \|_1 + \gamma } \Big), \nonumber \\
{\IoU}_\gamma(\pmb{\delta}) & = \mathbb{E} \Big( \frac{\big| I(\mb{Y}) \cap I({\pmb{\delta}}(\mb{X})) \big| + \gamma }{ \big| I(\mb{Y}) \cup I({\pmb{\delta}}(\mb{X})) \big| + \gamma } \Big) = \mathbb{E} \Big( \frac{ \mb{Y}^\intercal \pmb{\delta}(\mb{X}) + \gamma }{ \| \mb{Y} \|_1 + \| \pmb{\delta}(\mb{X}) \|_1 - \mb{Y}^\intercal \pmb{\delta}(\mb{X}) + \gamma } \Big), 
\label{eqn:dice_loss}
\end{align}
where $|\cdot|$ is the cardinality of a set, and $\gamma \geq 0$ is a smoothing parameter. When $\gamma = 0$, $\Dice_\gamma(\pmb{\delta}) = \mathbb{E}\big( 2\text{TP} / (2\text{TP} + \text{FP} + \text{FN}) \big)$, $\IoU_\gamma(\pmb{\delta}) = \mathbb{E}\big( \text{TP} / ( \text{TP} + \text{FP} + \text{FN}) \big)$ where TP is the true positive, FP is the false positive, and FN is the false negative. 
Both metrics are similar and will be treated in a unified fashion; however, as to be seen in the sequel, searching for the optimal segmentation function with respect to the IoU metric may require extra computation than its Dice counterpart. Thus, for clarity of presentation, we first focus on the Dice metric and postpone the discussion on the relationships between the Dice and IoU metrics in Section \ref{sec:IoU}.

The recent success of fully convolutional networks has enabled significant progress in segmentation. In literature, the mainstream of recent works devoted to designing and developing neural network architectures under different segmentation scenarios, including FCN \citep{long2015fully}, U-Net \citep{ronneberger2015u}, DeepLab \citep{chen2018encoder}, and PSPNet \citep{zhao2017pyramid}. 
Despite their remarkable performance, most existing models primarily focus on predicting segmentation using a classification framework, without considering the inherent disparities between classification and segmentation (as discussed in Section \ref{sec:existing_frameworks}). 
We find this framework leading to an inconsistent solution and suboptimal performance with respect to the Dice/IoU metrics, and we address this pitfall by developing a novel consistent ranking-based framework, namely \textit{RankSEG} (\textit{RankDice} to the Dice metric and \textit{RankIoU} to the IoU metric), to improve the segmentation performance.

\subsection{Existing methods}
\label{sec:existing_frameworks}
Most existing segmentation methods are developed under a threshold-based framework with two types of loss functions.
\begin{figure}[h]
    \centering
    \includegraphics[scale=0.45]{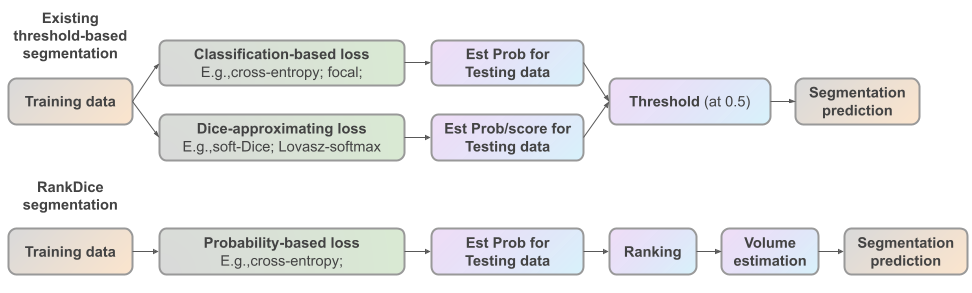}
    \caption{The existing and the proposed (\textit{RankDice}) frameworks for segmentation. The upper panel is the existing \textit{threshold-based segmentation} framework, and the lower panel is the proposed \textit{RankDice} framework.}
    \label{fig:framework}
\end{figure}

As indicated in Figure \ref{fig:framework}, the existing threshold-based segmentation framework, inspired by binary classification, provides a predicted segmentation via a two-stage procedure: (i) estimating a decision function or a probability function based on a loss; (ii) predicting feature-wise labels by thresholding the estimated decision function or probabilities. Specifically, given a training dataset $(\mb{x}_i, \mb{y}_i)_{i=1}^n$, the prediction provided by the threshold-based framework for a new instance $\mb{x}$ can be formulated as:
\begin{equation}
        \label{eqn:existing_framework}
        \widehat{\mb{q}} = \argmin_{\mb{q} \in \mathcal{Q}} \frac{1}{n} \sum_{i=1}^n l\big( \mb{y}_i, \mb{q}(\mb{x}_i) \big) + \lambda \| \mb{q} \|^2,  \qquad \widehat{\pmb \delta}(\mb{x}) = \mb{1}( \widehat{\mb{q}}(\mb{x}) \geq 0.5 ),
\end{equation}
where $l(\cdot, \cdot)$ is an operating loss, $\mb{q}: \mathbb{R}^d \to [0,1]^d$ is a candidate probability function with $q_j$ being the candidate probability of the $j$-th pixel, $\mathcal{Q}$ is a class of candidate probability functions, $\| \mb{q} \|$ is a regularization term, $\lambda>0$ is a tuning parameter to balance the overfitting and underfitting, and $\mb{1}(\cdot)$ is an indicator function. 
For ease of presentation, $q_j(\mb{x})$ is specified as a probability function and a predicted segmentation is produced by thresholding at 0.5, yet it can be equally extended to a general decision function. 
For example, we may formulate $q_j(\mb{x})$ as a decision function with range in $\mathbb{R}$, and the prediction is produced by thresholding at $0$, analogous to SVMs in classification \citep{cortes1995support}.
Next, under the framework \eqref{eqn:existing_framework}, two different types of operating loss functions are considered, namely the \textit{classification-based losses} and the \textit{Dice-approximating losses}.

\noindent \textbf{Classification-based losses} completely characterize segmentation as a classification problem, with examples such as the cross-entropy (CE; \cite{cox1958regression}) and the focal loss (Focal; \citet{lin2017focal}):
\begin{align}
    \label{eqn:classification_loss}
    & (\text{CE}) & l_{\text{CE}} \big(\mb{y}, \mb{q}(\mb{x}) \big) & = - \sum_{j=1}^d \Big( y_{j} \log \big( \mb{q}_j(\mb{x}) \big) + (1 - y_{j}) \log \big( 1 - \mb{q}_j(\mb{x}) \big) \Big), \\
    & (\text{Focal}) & l_{\text{focal}} \big(\mb{y}, \mb{q}(\mb{x}) \big) & = - \sum_{j=1}^d \Big( y_{j} (1 - \mb{q}_j(\mb{x}))^\vartheta \log \big( \mb{q}_j(\mb{x}) \big) + (1 - y_{j}) \mb{q}^\vartheta_j(\mb{x}) \log \big( 1 - \mb{q}_j(\mb{x}) \big) \Big), \nonumber
\end{align}
where $\vartheta \geq 0$ is a hyperparameter for the focal loss \citep{lin2017focal}.
Other margin-based losses such as the hinge loss, in principle, can be included as classification-based losses with a decision function ranged in $\mathbb{R}$ thresholding at 0, although they are less frequently used in a multiclass problem \citep{tewari2007consistency}. Therefore, we focus on the probability-based classification loss in the sequel.


\noindent \textbf{Dice-approximating losses} aim to approximate the Dice/IoU metric and conduct a direct optimization. Typical examples are the soft-Dice \citep{sudre2017generalised} and the Lovasz-softmax loss \citep{berman2018lovasz}:
\begin{align}
    & (\text{Soft-Dice}) & l_{\text{SoftD}} \big(\mb{y}, \mb{q}(\mb{x}) \big) & = 1 - \frac{ 2\mb{y}^\intercal \mb{q}(\mb{x}) }{ \| \mb{y} \|_1 + \| \mb{q}(\mb{x}) \|_1 },  \nonumber \\
    & (\text{Lovasz-softmax}) & l_{\text{l-softmax}} \big(\mb{y}, \mb{q}(\mb{x}) \big) & = \overbar{V}\big( \mb{y} \circ (1 - \mb{q}(\mb{x})) + (1 - \mb{y}) \circ \mb{q}(\mb{x}) \big),  \nonumber
\end{align}
where $\overbar{V}(\cdot)$ is the Lovasz extension of the mis-IoU error \citep{berman2018lovasz}, and $\circ$ is the element-wise product. Specifically, the soft-Dice loss replaces the binary segmentation indicator $\pmb{\delta}(\mb{x}) \in \{0,1\}^d$ in the Dice metric by a candidate probability function $\mb{q}(\mb{x}) \in [0,1]^d$ to make the computation feasible. The Lovasz-softmax directly takes a convex extension of IoU based on a softmax transformation.
Moreover, other losses including the Tversky loss \citep{salehi2017tversky}, the Lovasz-hinge loss \citep{berman2018lovasz}, and the log-Cosh Dice loss \citep{jadon2021semsegloss} can be also categorized as Dice-approximating losses. 

The threshold-based framework \eqref{eqn:existing_framework} with a classification-based loss or a Dice-approximating loss is a commonly used approach for segmentation. Although encouraging performance is delivered, we show that the minimizer from \eqref{eqn:existing_framework} (based on the cross-entropy and the focal loss) is \textbf{inconsistent} (or suboptimal) to the Dice metric, see Lemma \ref{lem:Dice-calibrated}. For Lovasz convex losses, it is still unclear if they are able to yield an optimal segmentation (convex closure is usually not enough to ensure the consistency \citep{bartlett2006convexity}). Moreover, in practice, only a small number of pixel predictions are taken into account in one stochastic gradient step. Therefore, the Lovasz-softmax loss cannot directly optimize the segmentation metric (see Section 3.1 in \cite{berman2018lovasz}). 

Another line of work \citep{bao2020calibrated,nordstrom2020calibrated,lipton2014optimal} has been centered on a linear-fractional approximation of the Dice and IoU metrics which are not decomposable per instance. 
In particular, \cite{bao2020calibrated} and \cite{nordstrom2020calibrated} proposed to approximate the utility by tractable surrogate functions with a sample-splitting procedure and showed that their methods are consistent in optimizing the target utility. Yet, splitting the sample may undermine the efficiency. \cite{lipton2014optimal} indicated that the optimal threshold maximizing the $F_1$ metric is equal to half of the optimal $F_1$ value. However, their definitions of Dice and IoU metrics may overlook small instances, which is undesirable in many applications; see Appendix \ref{sec:E-Dice} for technical discussion.

To summarize, the current frameworks with existing losses may either yield a suboptimal solution or suffer from an inappropriate target metric, demanding efforts to further improve the performance, robustness and sustainability of the existing segmentation framework.

\subsection{Our contribution}
In this paper, we propose a novel Dice-calibrated ranking-based segmentation framework, namely \textit{RankDice}, to address the inconsistency of the existing framework. \textit{RankDice} is primarily inspired by the Bayes rule of Dice-segmentation. 
We summarize our major contribution as follows: 

\begin{enumerate}
    \item To our best knowledge, the proposed ranking-based segmentation framework \textit{RankDice}, is the first consistent segmentation framework with respect to the Dice metric (Dice-calibrated).
    \item Three numerical algorithms with GPU parallel execution are developed to implement the proposed framework in large-scale and high-dimensional segmentation.
    \item We establish a theoretical foundation of segmentation with respect to the Dice metric, such as the Bayes rule and Dice-calibration. Moreover, we present Dice-calibrated consistency (Lemma \ref{lem:RankDice-calibrated}) and a convergence rate of the excess risk (Theorem \ref{thm:risk_bound}) for the proposed \textit{RankDice} framework, and indicate inconsistent results for the existing methods (Lemma \ref{lem:Dice-calibrated}).
    \item Our experiments in three simulated examples and three real datasets (CityScapes dataset, Pascal VOC 2021 dataset, and Kvasir-SEG dataset) suggest that {the improvement of \textit{RankDice} over the existing framework is significant for various loss functions and network architectures (see Tables \ref{tab:sim}-\ref{tab:VOC_binary}).}
\end{enumerate}
It is worth noting that the results are equally applicable to the proposed \textit{RankIoU} framework in terms of the IoU metric.

\section{RankDice}
It is reasonable to assume that all information on a feature-wise label is solely based on input features, that is, $Y_i \perp Y_j | \mb{X}$ for any $i \neq j$. In Appendix \ref{sec:dep}, we provide a probabilistic perspective to suggest the necessity of this assumption in segmentation tasks. Without loss of generality, we further assume that $p_j(\mb{X}) := \mathbb{P}(Y_j = 1|\mb{X})$ are distinct for $j=1, \cdots, d$ with probability one. 

\subsection{Bayes segmentation rule}
To begin with, we call segmentation with respect to the Dice metric as ``Dice-segmentation''. Then, we discuss Dice-segmentation at the population level, and present its Bayes (optimal) segmentation rule in Theorem \ref{thm:Dice_bayes} akin to the Bayes classifier for classification.

\begin{theorem}[The Bayes rule for Dice-segmentation]
\label{thm:Dice_bayes}
A segmentation rule $\pmb{\delta}^*$ is a global maximizer of $\Dice_\gamma(\pmb{\delta})$ if and only if it satisfies that
\[ \delta_j^*(\mb{x}) = 
   \begin{cases} 
      1 & \text{if } p_j(\mb{x}) \text{ ranks top }\tau^*(\mb{x}), \\
      0 & \text{otherwise}.
   \end{cases}
\]
The optimal volume (the optimum total number of segmented features) $\tau^*(\mb{x})$ is given as
\begin{equation}
    \label{eqn:vol_est_true_pro}
    \tau^*(\mb{x}) = \argmax_{\tau \in \{0, 1, \cdots, d\} } \ \Big( \sum_{j \in J_\tau(\mb{x}) } \sum_{l=0}^{d-1} \frac{2}{\tau + l + \gamma + 1}  p_j(\mb{x}) \mathbb{P} \big( \Gamma_{\sm j}(\mb{x}) = l \big) + \sum_{l=0}^d \frac{\gamma}{\tau + l + \gamma} \mathbb{P}\big( \Gamma(\mb{x}) = l \big) \Big),
\end{equation}
where $J_\tau(\mb{x}) = \big \{ j : \sum_{j'=1}^d \mb{1} \big( p_{j'}(\mb{x}) \geq p_j(\mb{x}) \big) \leq \tau \big\}$ is the index set of the $\tau$-largest conditional probabilities with $J_0(\mb{x}) = \emptyset$, $\Gamma(\mb{x}) = \sum_{j=1}^d {B}_{j}(\mb{x})$, and ${\Gamma}_{\sm j}(\mb{x}) = \sum_{j' \neq j} {B}_{j'}(\mb{x})$ are Poisson-binomial random variables, and ${B}_j(\mb{x})$ is a Bernoulli random variable with the success probability $p_{j}(\mb{x})$. See the definition of the Poisson-binomial distribution in Appendix \ref{sec:PBD}.
\end{theorem}

Two remarkable observations emerge from Theorem \ref{thm:Dice_bayes}.  First, the Bayes segmentation operator can be obtained via a two-stage procedure: (i) ranking the conditional probability $p_j(\mb{x})$, and (ii) searching for the optimal volume of the segmented features $\tau(\mb{x})$. Second, both the Bayes segmentation rule $\pmb{\delta}^*(\mb{x})$ and the optimal volume function $\tau^*(\mb{x})$ are achievable when the conditional probability $\mb{p}(\mb{x}) = ( p_1(\mb{x}), \cdots, p_d(\mb{x}) )^\intercal$ is well-estimated. Therefore, our proposed framework \textit{RankDice} is directly inspired by a general \textit{plug-in rule} of the Bayes segmentation rule. 

Moreover, Lemma \ref{lem:Dice-calibrated} (and examples provided in its proof) indicates that the segmentation rule produced by existing frameworks, such as the threshold-based framework, can significantly differ from the optimum in Theorem \ref{thm:Dice_bayes}. In fact, Lemma \ref{lem:Dice-calibrated} further proves that the cross-entropy loss, the focal loss, and even a general classification-calibrated loss \citep{zhang2004statistical,bartlett2006convexity}, are not Dice-calibrated. See the definition of Dice-calibrated (Definition \ref{def:Dice-calibrated}), and the negative results for existing frameworks (Lemma \ref{lem:Dice-calibrated}) in Section \ref{sec:Dice-calibrated}. Besides, the Bayes rule for IoU-segmentation is presented in Lemma \ref{lem:IoU_bayes}.

\begin{remark}[Suboptimal of a fixed-thresholding framework]
    \label{rk:T}
    The Bayes rule of Dice-segmentation can be also regarded as \textit{adaptive} thresholding of conditional probabilities. Specifically, for each input $\mb{x}$, the optimal segmentation rule can be rewritten as:
    $$
    \delta_j^*(\mb{x}) = \mb{1}\big( p_j(\mb{x}) \geq p_{j_{\tau^*(\mb{x})}}(\mb{x}) \big), \ \text{where } p_{j_{\tau^*(\mb{x})}}(\mb{x}) \text{ is the top-} \tau^*(\mb{x}) \text{ largest probability over } \mb{p}(\mb{x}). 
    $$
    Alternatively, Theorem \ref{thm:Dice_bayes} indicates that the Bayes rule for Dice-segmentation is unlikely to be obtained by a \textit{fixed} thresholding framework, since the optimal threshold $p_{j_{\tau^*(\mb{x})}}(\mb{x})$ varies greatly across different inputs. Therefore, (tuning a threshold on) a fixed thresholding-based framework leads to a suboptimal solution.
\end{remark}

Remark \ref{rk:T} also explains the heterogeneity of optimal thresholds in various datasets indicated in \cite{bice2021sensitivity}, and the suboptimality of a fixed-thresholding framework is also empirically supported by Table \ref{tab:sim2} and Figure \ref{fig:sim2} for simulated examples, and Table \ref{tab:Dthold} and Figure \ref{fig:VOC_heatmap} for real datasets. 
Furthermore, the fact that fixed-thresholding is suboptimal for Dice-segmentation should be compared with the existing results in classification. For binary classification, the optimal threshold maximizing the $F_1$ metric is equal to half of the optimal $F_1$ value, which is fixed \citep{lipton2014optimal}. This disparity stems from a different definition of the Dice metric (or $F_1$) for binary classification, where $F_1$-score (for binary classification) can be regard as a linear fractional utility of Dice defined in \eqref{eqn:dice_loss}; see more discussion in Appendix \ref{sec:E-Dice}.

\subsection{Proposed framework}
\label{sec:binary_rankdice}
Suppose a training dataset $(\mb{x}_i, \mb{y}_i)_{i=1}^n$ is given, where $\mb{x}_i \in \mathbb{R}^d$ and $\mb{y}_i \in \{0,1\}^d$ are the input features and the true label for the $i$-th instance. Inspired by Theorem \ref{thm:Dice_bayes}, we develop a ranking-based framework \textit{RankDice} for Dice-segmentation (Steps 1-3). 

\noindent \textbf{Step 1 (Conditional probability estimation)}: Estimate the conditional probability based on logistic regression (the cross-entropy loss):
\begin{equation}
\label{eqn:prob_est}
\widehat{\mb{q}}(\mb{x}) = \argmin_{\mb{q} \in \mathcal{Q}} - \sum_{i=1}^n \sum_{j=1}^d \Big( y_{ij} \log \big( q_j(\mb{x}_i) \big) + (1 - y_{ij}) \log \big( 1 - q_j(\mb{x}_i) \big) \Big) + \lambda \| \mb{q} \|^2,
\end{equation}
where $\mathcal{Q}$ is a class of candidate probability functions, $\| \mb{q} \|$ is a regularization for a candidate function, and $\lambda>0$ is a hyperparameter to balance the loss and regularization. For example, $\mb{q} \in \mathcal{Q}$ is usually a deep convolutional neural network for image segmentation, and $\| \mb{q} \|$ can be a matrix norm of weight matrices in the network.

\noindent \textbf{Step 2 (Ranking)}: Given a new instance $\mb{x}$, rank its estimated conditional probabilities decreasingly, and denote the corresponding indices as $j_1, \cdots, j_d$, that is, $\widehat{{q}}_{j_1}(\mb{x}) \geq \widehat{{q}}_{j_2}(\mb{x}) \geq \cdots \geq \widehat{{q}}_{j_d}(\mb{x})$.

\noindent \textbf{Step 3 (Volume estimation)}: From \eqref{eqn:vol_est_true_pro}, we estimate the volume $\widehat{\tau}(\mb{x})$ by replacing the true conditional probability $\mb{p}(\mb{x})$ by the estimated one $\widehat{\mb{q}}(\mb{x})$:
\begin{align}
\label{eqn:vol_est}
\widehat{\tau}(\mb{x}) 
& = \argmax_{\tau \in \{0, \cdots, d\} } \sum_{s = 1}^\tau \sum_{l=0}^{d-1} \frac{2}{\tau + l + \gamma + 1} \widehat{{q}}_{j_s}(\mb{x}) \mathbb{P} \big( \widehat{\Gamma}_{\sm j_s}(\mb{x}) = l \big) + \sum_{l=0}^d \frac{\gamma}{\tau + l + \gamma} \mathbb{P}\big( \widehat{\Gamma}(\mb{x}) = l \big) ,
\end{align}
where $\widehat{\Gamma}(\mb{x}) = \sum_{j =1}^d \widehat{B}_{j}(\mb{x})$ and $\widehat{\Gamma}_{\sm j_s}(\mb{x}) = \sum_{j \neq j_s} \widehat{B}_{j}(\mb{x})$ are Poisson-binomial random variables, and $\widehat{B}_j(\mb{x})$ are independent Bernoulli random variables with success probabilities $\widehat{q}_{j}(\mb{x})$; for $j=1,\cdots,d$.

Finally, the predicted segmentation $\widehat{\pmb{\delta}}(\mb{x}) = (\widehat{\delta}_{1}(\mb{x}), \cdots, \widehat{\delta}_{d}(\mb{x}))^\intercal$ is given by selecting the indices of the top-$\widehat{\tau}(\mb{x})$ conditional probabilities:
\begin{equation}
    \label{eqn:rankdice-pred}
    \widehat{\delta}_{j}(\mb{x}) =  1, \text{ if } j \in \{ j_1, \cdots, j_{\widehat{\tau}(\mb{x})} \}; \quad \widehat{\delta}_{j}(\mb{x}) =  0, \text{ otherwise.}
\end{equation}

The proposed \textit{RankDice} framework (Steps 1-3) provides a feasible solution to the Bayes segmentation rule in terms of the Dice metric. Note that \textbf{Step 1} is a standard conditional probability estimation, and \textbf{Step 2} simply ranks the estimated conditional probabilities. Next, we focus on developing a scalable computing scheme for \textbf{Step 3}.

\subsection{Scalable computing schemes}
\label{sec:compute}
This section develops scalable computing schemes for \textit{volume estimation} in \eqref{eqn:vol_est}. Note that \eqref{eqn:vol_est} can be rewritten as:
\begin{align}
\label{eqn:vol_opt}
    \widehat{\tau}(\mb{x}) & = \argmax_{\tau \in \{0, \cdots, d\} } \overbar{\pi}_\tau(\mb{x}); \qquad \overbar{\pi}_\tau(\mb{x}) = \overbar{\omega}_\tau(\mb{x}) + \overbar{\nu}_\tau(\mb{x}), \nonumber \\
    \overbar{\omega}_\tau(\mb{x}) & = \sum_{l=0}^{d-1} \frac{ 2 \omega_{\tau,l}(\mb{x})}{\tau + l + \gamma + 1}, \quad \omega_{\tau, l}(\mb{x}) = \sum_{s=1}^\tau \widehat{q}_{j_s}(\mb{x}) \mathbb{P}(\widehat{\Gamma}_{ \sm j_s}(\mb{x}) = l), \quad \overbar{\nu}_\tau(\mb{x}) = \sum_{l=0}^{d} \frac{ \gamma  \mathbb{P}(\widehat{\Gamma}(\mb{x}) = l)}{\tau + l + \gamma}.
\end{align}
The computational complexity of solving \eqref{eqn:vol_opt} is intimately related to the dimension of input features. Therefore, we develop numerical algorithms for low- and high-dimensional segmentation separately. 

\subsubsection{Exact algorithms for low-dimensional segmentation}
\textbf{Exact algorithm based on FFT.} When the dimension is low ($d \leq 500$), we consider an exact algorithm to evaluate $\overbar{\omega}_\tau$ and $\overbar{\nu}_\tau$. 
 According to the definition of $\omega_{\tau,l}$, it can be computed by the following recursive formula ($\tau = 1, \cdots, d$):
\begin{equation}
\label{eqn:algo_exact}
    \pmb{\omega}_{\tau}(\mb{x}) = \pmb{\omega}_{\tau-1}(\mb{x}) + \widehat{q}_{j_\tau}(\mb{x}) \big(  \mathbb{P}( \widehat{\Gamma}_{-j_\tau}(\mb{x}) = 0 ), \cdots, \mathbb{P}( \widehat{\Gamma}_{-j_\tau}(\mb{x}) = d-1 ) \big)^\intercal,
    \quad \pmb{\omega}_{0}(\mb{x}) = \mb{0},
\end{equation}
where $\pmb{\omega}_{\tau}(\mb{x}) = (\omega_{\tau,0}(\mb{x}), \cdots, \omega_{\tau,d-1}(\mb{x}))^\intercal$.
On this ground, it suffices to evaluate $\mathbb{P}(\widehat{\Gamma}(\mb{x}) = l)$ and $\mathbb{P}(\widehat{\Gamma}_{ \sm j_\tau}(\mb{x}) = l)$, which are the probability mass functions of Poisson-binomial random variables $\widehat{\Gamma}(\mb{x})$ and $\widehat{\Gamma}_{ \sm j_\tau}(\mb{x})$, respectively. As indicated in \cite{hong2013computing}, they can be efficiently evaluated by a fast Fourier transformation (FFT). Based on the numerical results in \cite{hong2013computing}, the computing time for FFT evaluation with $d \leq 500$ is generally negligible (less than ten milliseconds). Moreover, it is worth noting that our algorithm in \eqref{eqn:algo_exact} needs not store the entire auxiliary matrix $(\pmb{\omega}_1(\mb{x}), \cdots, \pmb{\omega}_d(\mb{x}))$, since the $\tau$-th row $\pmb{\omega}_\tau$ can be computed from the previous row $\pmb{\omega}_{\tau-1}$. Hence, only $O(d)$ storage is required in \eqref{eqn:algo_exact}. The detailed algorithm is summarized in Algorithm \ref{algo:rankdice} with \texttt{approx=False}.

\subsubsection{Approximation algorithms for high-dimensional segmentation}
For high-dimensional segmentation, especially for image segmentation, it is challenging to solve \eqref{eqn:vol_est} by a grid searching over $\tau \in \{0, \cdots, d\}$. To address this issue, we use shrinkage and approximation techniques to reduce the computational complexity of \eqref{eqn:vol_est}. First, Lemma \ref{lem:shrinkage} is developed to shrink the searching range of $\tau$.

\begin{lemma}[Shrinkage]
    \label{lem:shrinkage}
    If $\sum_{s=1}^{\tau} \widehat{q}_{j_s}(\mb{x}) \geq (\tau + \gamma + d) \widehat{q}_{j_{\tau+1}}(\mb{x})$, then $\overbar{\pi}_\tau(\mb{x}) \geq \overbar{\pi}_{\tau'}(\mb{x})$ for all $\tau' >\tau$.
\end{lemma}

Lemma \ref{lem:shrinkage} can be viewed as an early stopping of the grid searching, which draws from the property of Poisson binomial distribution (c.f. Lemma \ref{lem:de_score}). Accordingly, we can shrink the grid search in \eqref{eqn:vol_est} from $\{0, \cdots, d\}$ to $\{0, \cdots, d_0(\mb{x})\}$, which significantly reduces the computational complexity. Specifically,
\begin{equation}
    \label{eqn:vol_opt_shrink}
        \widehat{\tau}(\mb{x}) = \argmax_{\tau \in \{0, \cdots, d_0(\mb{x})\} } \overbar{\pi}_\tau(\mb{x}); \quad d_0(\mb{x}) = \min \big\{ \tau = 1, \cdots, d : \sum_{s=1}^\tau \widehat{q}_{j_s}(\mb{x}) / \widehat{q}_{j_{\tau+1}}(\mb{x}) \geq \tau + \gamma + d \big \}.
\end{equation}
In many applications, $d_0(\mb{x})$ is upper bounded by a small integer, that is, $d_0(\mb{x})< d_U \ll d$, called the \textit{well-separated segmentation}. For example, there exists a small number of features/pixels whose probabilities (close to 1) are much larger than the others.

\noindent \textbf{Truncated refined normal approximation (T-RNA).} Note that the cumulative distribution function (CDF), and thus the probability mass function, of a Poisson-binomial random variable can be efficiently evaluated by a refined normal approximation when the dimension is large \citep{hong2013computing,neammanee2005refinement}. For instance,
\begin{align}
    \label{eqn:RNA}
    \mathbb{P}( \widehat{\Gamma}(\mb{x}) \leq l ) & \ \leftarrow \ \widetilde{\mathbb{P}}( \widehat{\Gamma}(\mb{x}) \leq l ) \coloneqq \Psi \big( \widehat{\sigma}^{-1}(l + 1/2 - \widehat{\mu}(\mb{x})) \big), \nonumber \\
    \mathbb{P}( \widehat{\Gamma}_{\sm j}(\mb{x}) \leq l ) &  \ \leftarrow \ \widetilde{\mathbb{P}}( \widehat{\Gamma}_{\sm j}(\mb{x}) \leq l ) \coloneqq \Psi_{\sm j} \big( \widehat{\sigma}_{\sm j}^{-1}(l + 1/2 - \widehat{\mu}_{\sm j}(\mb{x})) \big),
\end{align}
where $\Psi(u) = \Phi(u) + \widehat{\eta}(\mb{x}) (1 - u^2) \phi(u) /6$ and $\Psi_{\sm j}(u) = \Phi(u) + \widehat{\eta}_{\sm j}(\mb{x}) (1 - u^2) \phi(u) /6$ are two skew-corrected refined normal CDFs, $\Phi(\cdot)$ is the CDF of the standard normal distribution, and $(\widehat{\mu}(\mb{x}), \widehat{\mu}_{\sm j}(\mb{x}))$, $(\widehat{\sigma}^2(\mb{x}), \widehat{\sigma}_{\sm j}^2(\mb{x}))$, $(\widehat{\eta}(\mb{x}), \widehat{\eta}_{\sm j}(\mb{x}))$ are mean, variance and skewness of the Poisson-binomial random variables $\widehat{\Gamma}(\mb{x})$ and $\widehat{\Gamma}_{\sm j}(\mb{x})$, respectively. See the definitions of variance and skewness of the Poisson-binomial distribution in Appendix \ref{sec:PBD}.

On this ground, it is unnecessary to compute all $\mathbb{P}(\widehat{\Gamma}_{\sm j}(\mb{x}) = l)$ and $\mathbb{P}(\widehat{\Gamma}(\mb{x}) = l)$ for $l=1, \cdots, d$, since they are negligibly close to zero when $l$ is too small or too large. In other words, many $\mathbb{P}(\widehat{\Gamma}_{\sm j}(\mb{x}) = l)$ and $\mathbb{P}(\widehat{\Gamma}(\mb{x}) = l)$ can be ignored when evaluating $\overbar{\omega}_\tau$ and $\overbar{\nu}_\tau$. Therefore, according to the refined normal approximation in \eqref{eqn:RNA}, $\overbar{\omega}_\tau$ and $\overbar{\nu}_\tau$ can be approximated by only taking a partial sum over a subset of $l = 0, \cdots, d$:
\begin{align}
    \label{eqn:app_vol_opt}
    \widetilde{\omega}_\tau(\mb{x}) & = \sum_{l \in \mathcal{L}(\epsilon)} \frac{2 \widetilde{\omega}_{\tau,l}(\mb{x})}{\tau + l + \gamma + 1} , \quad \widetilde{\nu}_\tau(\mb{x}) = \sum_{l \in \mathcal{L}(\epsilon)} \frac{\gamma \widetilde{\mathbb{P}}( \widehat{\Gamma}(\mb{x}) = l ) }{\tau + l + \gamma} , \quad
    \widetilde{\omega}_{\tau,l}(\mb{x}) = \sum_{s=1}^\tau \widehat{q}_{j_s}(\mb{x}) \widetilde{\mathbb{P}}( \widehat{\Gamma}_{\sm j_s}(\mb{x}) = l ), \nonumber \\
    \mathcal{L}(\epsilon) & = \Big\{ \floor{\widehat{\sigma}(\mb{x}) \Psi^{-1}(\epsilon) + \widehat{\mu}(\mb{x}) -\frac{3}{2} }, \cdots, \floor{ \widehat{\sigma}(\mb{x}) \Psi^{-1}(1 - \epsilon) + \widehat{\mu}(\mb{x}) - \frac{1}{2}} \Big\} \bigcap \{ 0, \cdots, d \},
\end{align}
where $\widetilde{\mathbb{P}}( \widehat{\Gamma}(\mb{x}) = l ) \coloneqq \widetilde{\mathbb{P}}( \widehat{\Gamma}(\mb{x}) \leq l ) - \widetilde{\mathbb{P}}( \widehat{\Gamma}(\mb{x}) \leq l -1 )$, $\widetilde{\mathbb{P}}( \widehat{\Gamma}_{\sm j}(\mb{x}) = l ) \coloneqq \widetilde{\mathbb{P}}( \widehat{\Gamma}_{\sm j}(\mb{x}) \leq l ) - \widetilde{\mathbb{P}}( \widehat{\Gamma}_{\sm j}(\mb{x}) \leq l -1 )$, $\Psi^{-1}(\cdot)$ is the quantile function of the refined normal distribution, and $\epsilon$ is a custom tolerance error. The detailed algorithm is summarized in Algorithm 1 with \texttt{approx=`T-RNA'}. Moreover, the following lemma quantifies the approximation error of the proposed approximate algorithm.

\begin{lemma}
    \label{lem:app}
    For $(\overbar{\omega}_\tau, \overbar{\nu}_\tau)$ and $(\widetilde{\omega}_\tau, \widetilde{\nu}_\tau)$ defined in \eqref{eqn:vol_opt} and \eqref{eqn:app_vol_opt} respectively, if $\widehat{\sigma}^2(\mb{x}) \geq 25$, then for any $\tau \in \{0, \cdots, d\}$, we have
    \begin{align*}
        | \widetilde{\omega}_\tau - \overbar{\omega}_\tau | &\leq \frac{4\tau}{ \tau + \gamma + 1} (\epsilon + \frac{C_0}{ \widehat{\sigma}^2(\mb{x}) }) + \frac{C_0\min(\widehat{\mu}(\mb{x}), \tau)}{ \widehat{\sigma}^2(\mb{x}) - 1/4 }\Big( \log\big( 1 + d \big) + 1\Big), \\
        | \widetilde{\nu}_\tau - \overbar{\nu}_\tau | & \leq \frac{2 \gamma}{ \tau + \gamma} (\epsilon + \frac{C_0}{ \widehat{\sigma}^2(\mb{x}) }) + \frac{C_0 \gamma }{ \widehat{\sigma}^2(\mb{x}) }\Big( \log\big( 1 + d \big) + 1\Big),
    \end{align*}
    where $C_0=0.1618$ if $\widehat{\sigma}^2(\mb{x}) \geq 100$, and $C_0=0.3056$ if $\widehat{\sigma}^2(\mb{x}) \geq 25$.  Moreover, when $d \to \infty$, we have
    $$
    | \widetilde{\pi}_\tau - \overbar{\pi}_\tau | \leq \big( \frac{4\tau}{ \tau + \gamma + 1} + \frac{2 \gamma}{ \tau + \gamma} \big) \epsilon + O( \frac{\min( \widehat{\mu}(\mb{x}), \tau ) \log(d) }{ \widehat{\sigma}^2(\mb{x}) } ).
    $$
    Here we define $0/0 \coloneqq 0$ for $\gamma / (\tau + \gamma)$ when $\tau = \gamma = 0$.
\end{lemma}
Note that $\widehat{\sigma}^2(\mb{x}) = \sum_{j=1}^d \widehat{q}_j(\mb{x}) (1 - \widehat{q}_j(\mb{x}))$ is the variance of $\widehat{\Gamma}(\mb{x})$, which often tends to infinity as $d \to \infty$.
In image segmentation, the dimension $d$ varies from 1024 (32x32) to 262144 (512x512). Therefore, the error bound is practical for high-dimensional segmentation. 
Moreover, $\epsilon$ is the tolerance error to balance the approximation error and computation complexity. For instance, when $\epsilon$ becomes smaller, the approximation error decreases, the computation complexity increases with an enlarged $\mathcal{L}(\epsilon)$. Typically, $\Psi^{-1}(1 - \epsilon) - \Psi^{-1}(\epsilon) \leq 7.438$ when $\epsilon = 1\mathrm{e}{-4}$. Based on the approximation formula \eqref{eqn:app_vol_opt}, the worst-case computational complexity is reduced to $O( d \widehat{\sigma}(\mb{x}) )$ for general segmentation and $O( \widehat{\sigma}(\mb{x}) )$ for \textit{well-separated segmentation} ($d_0(\mb{x}) \leq d_U$). The computational complexity is summarized in Table \ref{tab:cc} (based on $\epsilon = 1\mathrm{e}{-4}$). 

\begin{algorithm}
    \SetKwInOut{Input}{Input}
    \SetKwInOut{Output}{Output}
    \Input{Training set: $(\mb{x}_i, \mb{y}_i)_{i=1}^n$; A new testing instance: $\mb{x}$; Approx method: \texttt{approx}}
    \Output{The predicted segmentation for the testing instance $\widehat{\pmb{\delta}}(\mb{x})$ }
    \textbf{Conditional prob est.} Estimate conditional prob $\widehat{\mb{q}}$ via \eqref{eqn:prob_est} based on training set $(\mb{x}_i, \mb{y}_i)_{i=1}^n$\;
    \textbf{Ranking.} Rank estimated conditional probabilities for $\mb{x}$, and denote the indices in sorted order as $j_1, \cdots, j_d$, that is, $\widehat{q}_{j_1}(\mb{x}) \geq \widehat{q}_{j_2}(\mb{x}) \geq \cdots \geq \widehat{q}_{j_d}(\mb{x})$\;
    \textbf{Volume estimation.}\\
    Compute and store the cumulative sum of sorted estimated probabilities $\widehat{\mb{s}}(\mb{x})$\;
    Compute $d_0(\mb{x})$ based on $\widehat{\mb{s}}(\mb{x})$ via \eqref{eqn:vol_opt_shrink}\;
    \eIf{\texttt{approx} is \texttt{None}}
        {$\mathcal{L} = \{0, \cdots, d\}$\;}
        {$\mathcal{L} = \mathcal{L}(\epsilon) = \{l_L, \cdots, l_U\}$ based on \eqref{eqn:app_vol_opt}\;}
        Compute and store $ \mathbb{P}(\widehat{\Gamma}(\mb{x}) = l)$ for $l \in \mathcal{L}$\;
    \eIf{\texttt{approx} is \texttt{BA}}
    {
        Compute and store $\widetilde{\mb{\omega}}_{1:d_0(\mb{x})}^b$ and $\widetilde{\mb{\nu}}_{1:d_0(\mb{x})}^b$ based on \eqref{eqn:blind_approx}\;
    }
    {
    Initialize $\mb{\omega}_{old} = \mb{0}$\;
    \For{$\tau = 1, \cdots, d_0(\mb{x})$}{
    Update $\mb{\omega}_{new}$ as
    $$
    \mb{\omega}_{new, l} \leftarrow \mb{\omega}_{old, l} + \widehat{q}_{j_\tau}(\mb{x}) \mathbb{P}(\widehat{\Gamma}_{-j_\tau}(\mb{x}) = l), \text{ for } l \in \mathcal{L}, \quad \mb{\omega}_{old}  \leftarrow \mb{\omega}_{new},
    $$
    where $\widehat{\Gamma}_{\sm j}(\mb{x})$ is Poisson binomial r.v. with the success probabilities $\widehat{\mb{q}}_{\sm j}(\mb{x})$\;
    Compute and store $\overbar{\pi}_\tau = \overbar{\omega}_\tau + \overbar{\nu}_\tau = \sum_{l \in \mathcal{L}} \frac{1}{\tau + l + 1} \omega_{new,l} + \sum_{l \in \mathcal{L}} \frac{\gamma}{\tau + l + \gamma} \mathbb{P}(\widehat{\Gamma}(\mb{x}) = l)$\;
    }
    }
	Estimate $\widehat{\tau}(\mb{x}) = \argmax_{\tau = 0, \cdots, d_0(\mb{x})} \overbar{\pi}_\tau$ via \eqref{eqn:vol_opt}\;
    \textbf{Prediction.} The predicted segmentation is provided as:
    $$\widehat{\delta}_{j}(\mb{x}) =  1, \text{ if } j \in \{ j_1, \cdots, j_{\widehat{\tau}(\mb{x})} \}; \quad \widehat{\delta}_{j}(\mb{x}) =  0, \text{ otherwise.}$$
    \Return{The predicted segmentation $ \widehat{\pmb{\delta}}(\mb{x}) = (\widehat{\delta}_{1}(\mb{x}), \cdots, \widehat{\delta}_{d}(\mb{x}))^\intercal$.}
    \caption{Computing schemes for the proposed RankDice framework.} 
    \label{algo:rankdice}
\end{algorithm}

Although T-RNA significantly reduces the computational complexity, 
yet this algorithm uses recursive updates, making it difficult in parallel computing on GPUs.
For example, due to the memory issues, recursive function calls are restricted in the CUDA (Compute Unified Device Architecture) kernels \citep{nickolls2008scalable}. Next, we propose the blind approximation algorithm to exploit GPU-computing for the proposed \textit{RankDice}.

\noindent \textbf{Blind approximation (BA).} In high-dimensional segmentation, the difference in distributions between $\widehat{\Gamma}(\mb{x})$ and $\widehat{\Gamma}_{\sm j}(\mb{x})$ is negligible. 
Inspired by this fact, we propose a novel approximation algorithm, namely the blind approximation, to simultaneously evaluate the score functions with a small error tolerance.
Specifically, based on the refined normal approximation, we further simplify the evaluation formulas by replacing $\widetilde{\mathbb{P}}( \widehat{\Gamma}_{\sm j_s}(\mb{x}) = l )$ as $\widetilde{\mathbb{P}}( \widehat{\Gamma}(\mb{x}) = l )$:
\begin{align}
    \widetilde{\omega}^b_\tau(\mb{x}) & = 2 \sum_{s=1}^\tau \widehat{q}_{j_s}(\mb{x}) \sum_{l \in \mathcal{L}(\epsilon)} \frac{ \widetilde{\mathbb{P}}( \widehat{\Gamma}(\mb{x}) = l ) }{\tau + l + \gamma + 1}, 
    \quad \widetilde{\nu}^b_\tau(\mb{x}) = \widetilde{\nu}_\tau(\mb{x}) = \sum_{l \in \mathcal{L}(\epsilon)} \frac{\gamma \widetilde{\mathbb{P}}( \widehat{\Gamma}(\mb{x}) = l ) }{\tau + l + \gamma}, \nonumber
\end{align}
where $\mathcal{L}(\epsilon) = \{l_L, \cdots, l_U\}$ is defined in \eqref{eqn:app_vol_opt}. Then, for any $1 \leq d_0 \leq d$, $\widetilde{\pmb{\omega}}_{1:T}^b = ( \widetilde{\omega}_{1}^b, \cdots, \widetilde{\omega}_{T}^b)^\intercal$ and $\widetilde{\pmb{\nu}}_{1:T}^b = ( \widetilde{\nu}_{1}^b, \cdots, \widetilde{\nu}_{T}^b)^\intercal$ can be simultaneously computed over $\tau = 0, \cdots, d_0$ based on cross-correlation (flipped convolution \citep{tahmasebi2012multiple}):
\begin{align}
    \label{eqn:blind_approx}
    \widetilde{\pmb{\omega}}_{1:d_0}^b & = 2 \widehat{\mb{s}}_{1:d_0}(\mb{x}) \circ \Big( \big( \widetilde{\mathbb P}( \widehat{\Gamma}(\mb{x}) = l_L ), \cdots, \widetilde{\mathbb P}( \widehat{\Gamma}(\mb{x}) = l_U ) \big)^\intercal \mb{\star} \big( \frac{1}{l_l + \gamma + 1}, \cdots, \frac{1}{l_U + \gamma + d_0 + 1} \big) \Big), \nonumber \\
    \widetilde{\pmb{\nu}}_{1:d_0}^b & = \gamma \Big( \big( \widetilde{\mathbb P}( \widehat{\Gamma}(\mb{x}) = l_L ), \cdots, \widetilde{\mathbb P}( \widehat{\Gamma}(\mb{x}) = l_U ) \big)^\intercal \mb{\star} \big( \frac{1}{l_l + \gamma}, \cdots, \frac{1}{l_U + \gamma + d_0} \big) \Big),
\end{align}
where $\widehat{\mb{s}}(\mb{x}) = (\widehat{s}_{1}(\mb{x}), \cdots, \widehat{s}_{d}(\mb{x}))^\intercal$ and $\widehat{s}_K(\mb{x}) = \sum_{k=1}^K \widehat{q}_{j_k}(\mb{x})$ is the cumulative sum of sorted estimated probabilities, $\circ$ is element-wise product of two vectors, and $\mb{\star}$ is the cross-correlation operator (flipped convolution) of two vectors \citep{tahmasebi2012multiple}. Notably, the cross-correlation can be efficiently implemented by a fast Fourier transform \citep{bracewell1986fourier} with time complexity $O( (d_0 + \widehat{\sigma}(\mb{x})) \log(d_0 + \widehat{\sigma}(\mb{x})) )$. Besides, the overall time complexity with CUDA implementation on GPUs can be further reduced by parallel computing. The detailed algorithm is summarized in Algorithm 1 with \texttt{approx=`BA'}. Next, Lemma \ref{lem:blind_approx} shows the approximation error of the proposed blind approximation algorithm.

\begin{lemma}
    \label{lem:blind_approx}
    For $(\overbar{\omega}_\tau, \overbar{\nu}_\tau)$ and $(\widetilde{\omega}^b_\tau, \widetilde{\nu}^b_\tau)$ defined in \eqref{eqn:vol_opt} and \eqref{eqn:blind_approx} respectively, if $\widehat{\sigma}^2(\mb{x}) \geq 25$, then for any $\tau \in \{0, \cdots, d\}$, and any $\gamma \geq 0$, we have
    \begin{align*}
        \big| \widetilde{\omega}^b_\tau - \overbar{\omega}_\tau \big| & \leq \frac{4\tau}{ \tau + \gamma + 1} (\epsilon + \frac{C_0}{ \widehat{\sigma}^2(\mb{x}) }) + \frac{C_0\min(\widehat{\mu}(\mb{x}), \tau)}{ \widehat{\sigma}^2(\mb{x}) - 1/4 }\Big( \log\big( 1 + d \big) + 1\Big) \\ 
        & \quad + \frac{1}{4 \sqrt{2\pi}} \Big( \frac{1/(2\sqrt{e})}{ \widehat{\sigma}^2(\mb{x}) - 1/4 } + \frac{4}{\sqrt{\widehat{\sigma}^2(\mb{x}) - 1/4}} + \frac{4 \widehat{m}_3(\mb{x})}{ (\widehat{\sigma}^2(\mb{x}) - 1/4)^{3/2} } \Big) \Big( \log\big( 1 + d \big) + 1 \Big),
    \end{align*}
    where $\widehat{m}_3(\mb{x}) = \sum_{j=1}^d \widehat{q}_j(\mb{x})( 1 - \widehat{q}_j(\mb{x}) )(1 - 2 \widehat{q}_j(\mb{x}))$, $C_0=0.1618$ if $\widehat{\sigma}^2(\mb{x}) \geq 100$, and $C_0=0.3056$ if $\widehat{\sigma}^2(\mb{x}) \geq 25$. Specifically, when $d \to \infty$, we have
    $$
    | \widetilde{\pi}_\tau - \overbar{\pi}_\tau | \leq \big( \frac{4\tau}{ \tau + \gamma + 1} + \frac{2 \gamma}{ \tau + \gamma} \big) \epsilon + O\Big( \big(\frac{\min( \widehat{\mu}(\mb{x}), \tau )}{ \widehat{\sigma}^2(\mb{x}) } + \frac{1}{\widehat{\sigma}(\mb{x})}\big) \log(d) \Big).
    $$
\end{lemma}
In contrast to the truncated refined normal approximation, the blind approximation algorithm significantly reduces the time complexity and enables GPU parallel execution. On the other hand, the blind approximation leads to an additional $\widehat{\sigma}^{-1}(\mb{x})$ in Lemma \ref{lem:blind_approx} compared with that of T-RNA in Lemma \ref{lem:app}, yet the error bound is still acceptable in practice.

Taken together, we summarize the foregoing computational schemes in Algorithm \ref{algo:rankdice}, and their inference (after obtaining conditional probabilities) computational complexity (worst-case) is indicated in Table \ref{tab:cc}. 
For the threshold-based framework, thresholding the estimated probabilities takes $O(d)$ time complexity. For the proposed method, Step 2 takes $O(d \log(d))$ based on the merge sort \citep{ajtai19830}, and Step 3 takes $O(d_0(\mb{x}) \widehat{\sigma}(\mb{x}) )$ based on T-RNA in \eqref{eqn:vol_opt_shrink} and \eqref{eqn:RNA}, and takes $O\big( d_0(\mb{x}) \log( d_0(\mb{x}) )\big)$ based on BA in \eqref{eqn:blind_approx}.

\begin{table*}[h]\centering
    \scalebox{.9}{
    \begin{tabular}{@{}clcccccccccccccccc@{}} \toprule
     \phantom{a} & SEG framework & \phantom{a} & Time & \phantom{a} & Time (well-separated) & \phantom{a} & Memory \\
    \midrule
    & Threshold-based SEG && $O(d)$ && $O(d)$ && $O(d)$  \\
    & RankDice (our) && $O( d \log(d) + d d_0(\mb{x}) )$ && $O( d \log(d) )$ && $O(d)$ \\
    & RankDice-RNA (our) && $O( d \log(d) + \widehat{\sigma}(\mb{x}) d_0(\mb{x}) )$ && $O(d\log(d))$ && $O(d)$ \\
    & RankDice-BA (our) && $O( d_0(\mb{x}) \log(d_0(\mb{x})) )$ && $O(d\log(d))$ && $O(d)$ \\
    \bottomrule
    \end{tabular}}
    \caption{Inference (prediction) computational complexity for the segmentation frameworks. Here ``Time'' denotes the time complexity, ``Memory'' denotes the amount of storage needed including probability outcomes, ``Time (well-separated)'' denotes the time complexity on \textit{well-separated segmentation} ($d_0(\mb{x}) \leq d_U$), and $d_0(\mb{x}) \leq d$ is the reduced dimension defined in \eqref{eqn:vol_opt_shrink}, $\widehat{\sigma}(\mb{x})$ is the standard deviation of $\widehat{\Gamma}(\mb{x})$ with at most an order of $O(\sqrt{d})$. Note that the time complexity of RankDice-BA can be further reduced by GPU implementation. See more detailed discussion in Section \ref{sec:compute}.}
    \label{tab:cc}
\end{table*}

\section{mDice-segmentation and mRankDice}
In this section, we discuss the extension of the proposed \emph{RankDice} framework to segmentation with multiclass/multilabel outcomes evaluated by the mDice metric. Overall, multiclass/multilabel segmentation differs from Dice-segmentation in a number of important aspects. First, a new evaluation metric called mDice is introduced. Second, multiclass/multilabel outcomes provide two different ways of probabilistic modeling. Third, whether (or not) to allow for overlapping segmentation results among different classes leads to different problem setups. These aspects will be discussed in detail in the following subsections.

\subsection{The mDice metric}
For multiclass/multilabel segmentation, $(\mb{X}, \mb{Y}_{\cdot 1}, \cdots, \mb{Y}_{\cdot K} )$ is available, where $\mb{X} \in \mathbb{R}^d$ represents a feature vector, $K$ is the total number of classes of interest, $\mb{Y}_{\cdot k} \in \{0,1\}^d$ is the true feature-wise segmentation label for the $k$-th class, where $Y_{jk} = 1$ indicates that the $j$-th feature $X_j$ is segmented under the $k$-th class, and $I(\mb{Y}_{\cdot k}) = \big\{ j: Y_{jk} = 1; \text{ for } j = 1, \cdots, d \big\}$ is the class-specific index set for all segmented features.  

On this ground, a class-specific segmentation operator $\pmb{\Delta}_k: \mathbb{R}^d \to \{0,1\}^d (k=1, \cdots, K)$ is introduced, where $\Delta_{jk}(\mb{x}) \in \{0,1\}$ is the predicted segmentation for the class $k$ of the $j$-th feature, and $I(\pmb{\Delta}_k(\mb{X})) = \{j: \Delta_{jk}(\mb{X}) = 1; \text{ for } j = 1, \cdots, d \}$ is the class-specific index set of features for the predicted segmentation. Then, given a segmentation operator $\pmb{\Delta} = (\pmb{\Delta}_1, \cdots, \pmb{\Delta}_K)$, the multi-Dice (mDice) metric is defined as:
\begin{align}
    \mDice_\gamma( \pmb{\Delta} ) & = \sum_{k=1}^K \alpha_k {\Dice}_{\gamma, k}(\pmb{\Delta}_k) \nonumber \\ 
    & = \sum_{k=1}^K \alpha_k \mathbb{E} \Big( \frac{2 \big| I(\mb{Y}_{\cdot k}) \cap I(\pmb{\Delta}_k(\mb{X})) \big| + \gamma }{ | I(\mb{Y}_{\cdot k}) | + | I(\pmb{\Delta}_k(\mb{X})) | + \gamma } \Big) = \sum_{k=1}^K \alpha_k \mathbb{E} \Big( \frac{2 \mb{Y}_{\cdot k}^\intercal \pmb{\Delta}_k(\mb{X}) + \gamma }{ \| \mb{Y}_{\cdot k} \|_1 + \| \pmb{\Delta}_k(\mb{X}) \|_1 + \gamma } \Big),
    \label{eqn:mdice_loss}
\end{align}
where ${\Dice}_{\gamma, k}(\cdot)$ is the Dice metric under the $k$-th class, $\pmb{\alpha} = (\alpha_1, \cdots, \alpha_K)^\intercal \geq \mb{0}_K$ is a weight vector for classes with $\| \pmb{\alpha} \|_1 = 1$. For example, $\alpha_k = 1/K$ yields that each class has the same contribution to segmentation performance. More generally, $\pmb{\alpha} \geq \mb{0}_K$ can be a custom weight vector indicating the relative importance of segmentation classes. In practice, given a new instance $(\mb{x}, \mb{y})$, the weight is an average over classes excluding those are not present and not predicted, that is,
\begin{equation}
    \label{eqn:Dice_weight}
    \alpha_k = 0, \text{ if } \| \mb{y}_{\cdot k} \|_1 = \| \pmb{\Delta}_k(\mb{x}) \|_1 = 0; \qquad \alpha_k = \frac{1}{\sum_{k=1}^K \mb{1}(\| \mb{y}_{\cdot k} \|_1 + \| \pmb{\Delta}_k(\mb{x}) \|_1 > 0  ) }, \text{ otherwise}.
\end{equation}

Following our convention, we shall call multiclass/multilabel segmentation with respect to the mDice metric as ``mDice-segmentation''. As indicated in Figure \ref{fig:mDice-frameworks}, unlike Dice-segmentation, mDice-segmentation provides more flexibility in probabilistic modeling (multiclass or multilabel) and the decision-making in prediction (taking argmax or thresholding at 0.5), resulting in different operating losses and the construction of predictive functions. 

\subsection{Multilabel/multiclass outcomes}
In this section, we describe two probabilistic models (multilabel or multiclass) of $\mb{Y}_{j} | \mb{X}$, where $\mb{Y}_{j} = (Y_{j1}, \cdots, Y_{jK})^\intercal$ is the true label for the $j$-th feature. 


For multilabel modeling, we assume each feature could be assigned to multiple classes, that is, $\| \mb{Y}_{j} \|_1 \in \{0, \cdots, K\}$, for $j = 1, \cdots, d$. As a result, the index sets of segmented features may overlap among classes. In this case, we formulate and estimate $q_{jk}(\mb{x})$ under a multilabel probabilistic model. For example, for deep learning models, we use the sigmoid function as the output layer activation function of a neural network with the \textit{binary cross-entropy} loss. 

For multiclass modeling, each feature is assigned to one and only one class, that is, $\| \mb{Y}_{j} \|_1 = 1$ for $j = 1, \cdots, d$. An instance is the panoptic annotation in image segmentation.
In this case, we formulate and estimate $\mb{q}_j(\mb{x}) = (q_{j1}(\mb{x}), \cdots, q_{jK}(\mb{x}))^\intercal$ under a multiclass model with additional sum-to-one constraints $\| \mb{q}_j(\mb{x}) \|_1 = 1$ for $j=1, \cdots, d$ and $\mb{x} \in \mathbb{R}^d$. Specifically, for deep learning models, we use the softmax function as the output layer activation function of a neural network, which automatically enforces the sum-to-one constraints. Correspondingly, a multiclass loss is used as an operating loss, including the \textit{multiclass cross-entropy}. Note that the Dice-approximating losses, such as the soft-Dice loss, can be adopted into the multilabel/multiclass modeling.

In literature, once the estimation is done, the predicted segmentation is produced by taking \textit{argmax} or \textit{thresholding} at 0.5 on the estimated probabilities. Indeed, \textit{argmax} and \textit{thresholding} are inspired by the decision-making in multiclass and multilabel classification, respectively. In segmentation, it is possible to attempt ad-hoc combinations, such as multiclass modeling + \textit{thresholding}, and multilabel modeling + \textit{argmax}.
The major purpose of \textit{argmax} is to provide \textit{non-overlapping} prediction (e.g., the panoptic prediction in image segmentation). We discuss overlapping/non-overlapping segmentation in the next section.
 
\begin{figure}[h]
    \centering
    \includegraphics[scale=0.40]{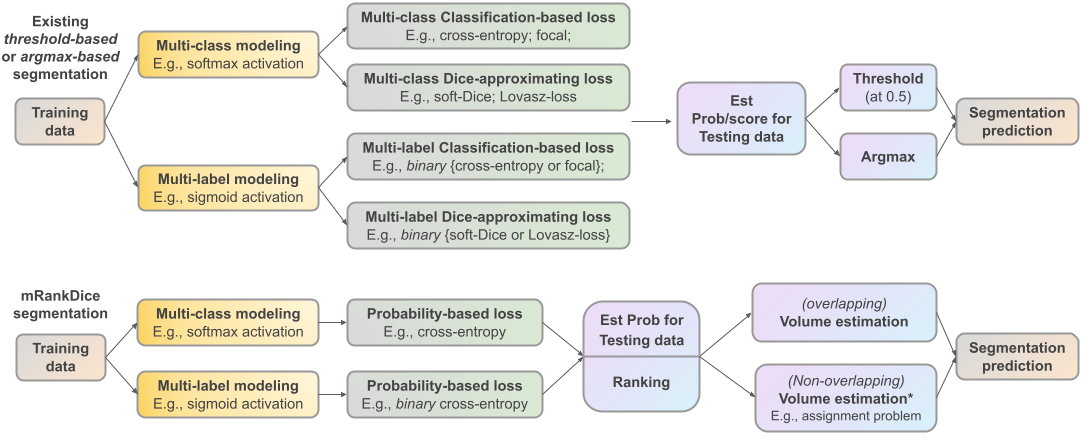}
    \caption{The existing and the proposed frameworks under multiclass/multilabel modeling. The upper panel is \textit{threshold-/argmax-based segmentation}, and the lower panel is the proposed \textit{mRankDice} framework.}
    \label{fig:mDice-frameworks} 
\end{figure}

\subsection{Overlapping or non-overlapping mDice-segmentation}
Specifically, whether (or not) to allow for overlapping results among distinct classes leads to different decision-making procedures, namely overlapping/non-overlapping mDice-segmentation:
\begin{align}
    \label{eqn:overlap_mDice}
    \text{(Overlapping)} \ \ \argmax_{ \pmb{\Delta} } \ \mDice_\gamma(\pmb{\Delta}), \quad \text{(Non-overlapping)} \ \ \argmax_{ \pmb{\Delta} } \ \mDice_\gamma(\pmb{\Delta}), \quad \sum_{k=1}^K \pmb{\Delta}_k(\cdot) = \mb{1}_d.
\end{align}
In the overlapping setting, there is no restriction on a segmentation operator, thus the predicted segmentation for different classes may overlap. On the other hand, the overlapping is ruled out in non-overlapping formulation due the additional sum-to-one constraint. Lemma \ref{lem:overlapping} presents the Bayes rule for overlapping segmentation, yielding that mDice-segmentation is reduced to class-specific Dice-segmentation subproblems if overlapping is allowed.

\begin{lemma}[The Bayes rule for overlapping mDice-segmentation]
    \label{lem:overlapping}
    An overlapping (allowing) segmentation operator $\pmb{\Delta}^*$ is a global maximizer of $\mDice_\gamma(\pmb{\Delta})$ if and only if $\pmb{\Delta}^*_k$ is the Bayes rule (global maximizer) of ${\Dice}_{\gamma,k}(\pmb{\Delta}_k)$ on $(\mb{X}, \mb{Y}_{\cdot k})$ for all $k \in \{ 1 \leq k \leq K: \alpha_k > 0 \}$.
\end{lemma}

Therefore, it suffices to consider Dice-segmentation of each class separately in overlapping mDice-segmentation, and the proposed \textit{RankDice} is readily extended to \textit{mRankDice}; see Section \ref{sec:mRankDice}. 

Next, we investigate the Bayes rule for non-overlapping mDice-segmentation. To proceed, we denote $\pmb{\Delta}^*$ as the Bayes rule (global maximizer) of non-overlapping mDice-segmentation in \eqref{eqn:overlap_mDice}, and $\pmb{\tau}^*(\cdot)$ as the volume function of the Bayes segmentation rule:
$
\pmb{\tau}^*(\mb{x}) = (\tau^*_1(\mb{x}), \dots, \tau^*_K(\mb{x}))^\intercal = ( \| \pmb{\Delta}^*_1(\mb{x}) \|_1, \cdots, \| \pmb{\Delta}^*_K(\mb{x}) \|_1 )^\intercal.
$
\begin{lemma}
    \label{lem:nonoverlap_bayes_rule}
    Suppose $\pmb{\tau}^*(\cdot)$ is pre-known, then solving the Bayes rule for non-overlapping mDice-segmentation in \eqref{eqn:overlap_mDice} is equivalent to an assignment problem.
\end{lemma}
As indicated in Lemma \ref{lem:nonoverlap_bayes_rule}, even when the optimal volume function is pre-given, solving the Bayes rule for non-overlapping segmentation is nontrivial. For example, the most successful assignment algorithms, including the Hungarian method \citep{kuhn1955hungarian,edmonds1972theoretical, tomizawa1971some} and its variants Jonker-Volgenant algorithm \citep{crouse2016implementing}, generally achieves an $O(d^3)$ running time complexity in our content, which is time-consuming for high-dimensional segmentation. In sharp contrast, for the overlapping case, when $\tau^*(\mb{x})$ is given, the Bayes rule is simply ranking all the conditional probabilities with $O(d \log(d))$. Moreover, when $\pmb{\tau}^*(\mb{x})$ is unknown, the optimization for non-overlapping segmentation becomes a nonlinear integer optimization which is NP-hard in general \citep{murty1985some,d2020lower}. Therefore, a fast $O(d \log(d))$ greedy approximation algorithm is more feasible in practical implementation. We leave pursuing this topic as future work.  

Next, we summarize the proposed \textit{mRankDice} framework under three scenarios.

\subsection{mRankDice}
\label{sec:mRankDice}
In this section, we present the proposed \textit{mRankDice} framework for mDice-segmentation. Before we proceed, we  highlight the different roles of multiclass/multilabel modeling and the overlapping/non-overlapping constraint. Multiclass/multilabel modeling determines the probabilistic models (say the softmax or the sigmoid activation in a neural network) and an operating loss in probability estimation (say the cross-entropy or the binary cross-entropy). Meanwhile, the overlapping/non-overlapping constraint affects the segmentation prediction after the probabilities are estimated. 

Therefore, we consider following segmentation three cases: ``multilabel + overlapping'',  ``multiclass + overlapping'', and ``multilabel/multiclass + non-overlapping''.

\noindent \textbf{mRankDice (multilabel + overlapping segmentation)} is a straightforward extension of \textit{RankDice} in  Dice-segmentation (inspired by Lemma \ref{lem:overlapping}). Given a training dataset $(\mb{x}_i, \mb{y}_{i, 1:d, 1}, \cdots, \mb{y}_{i, 1:d, K})_{i=1}^n$, with the same manner, the conditional probability function is estimated under a multilabel logistic regression (the binary cross-entropy loss):
\begin{equation}
    \label{eqn:multi_label_prob_est}
    \widehat{\mb Q}(\mb{x}) = \argmin_{\mb{Q} \in \mathcal{Q}}  \sum_{i=1}^n \sum_{j=1}^d \sum_{k=1}^K \Big( y_{ijk} \log \big( q_{jk}(\mb{x}_i) \big) + (1 - y_{ijk}) \log \big( 1 - q_{jk}(\mb{x}_i) \big) \Big) + \lambda \| \mb{Q} \|^2,
    \end{equation}
where $\mb{Q} = (q_{jk}): \mathbb{R}^d \to [0,1]^{d \times K}$ is a matrix function, and $q_{jk}(\mb{x})$ is a candidate estimator of $p_{jk}(\mb{x})$. Then, given a new instance $\mb{x}$, the class-specific segmentation  $\widehat{\pmb{\Delta}}_k(\mb{x})$ is generated based on \textbf{Steps 2-3} in Section \ref{sec:binary_rankdice} with the estimated conditional probabilities $\widehat{\mb{Q}}_k(\mb{x})$ (the $k$-th column of $\widehat{\mb{Q}}(\mb{x})$). We summarize the foregoing computational scheme in Algorithm \ref{algo:mrankdice_ol}.

\noindent \textbf{mRankDice (multiclass + overlapping segmentation)} yields a different probability estimation procedure, where the conditional probability function is estimated under a multiclass logistic regression (the multiclass cross-entropy loss):
\begin{equation}
    \label{eqn:multi_class_prob_est}
    \widehat{\mb Q}(\mb{x}) = \argmin_{\mb{Q} \in \mathcal{Q}}  \sum_{i=1}^n \sum_{j=1}^d \sum_{k=1}^K  y_{ijk} \log \big( q_{jk}(\mb{x}_i) \big) + \lambda \| \mb{Q} \|^2, \ \ \text{s.t.} \ \sum_{k=1}^K q_{jk}(\cdot) = 1; \text{ for } j=1,\cdots d,
\end{equation}
where $\mb{Q} = (q_{jk}): \mathbb{R}^d \to [0,1]^{d \times K}$ is a matrix function, and $\mb{q}_{j}(\mb{x})$ is a candidate estimator of $\mb{p}_j(\mb{x})$. Although the probability estimation \eqref{eqn:multi_class_prob_est} differs from \eqref{eqn:multi_label_prob_est}, the downstream ranking and volume estimation remain the same according to Lemma \ref{lem:overlapping}. We also summarize the foregoing computational scheme in Algorithm \ref{algo:mrankdice_ol}.

\begin{algorithm}
    \SetKwInOut{Input}{Input}
    \SetKwInOut{Output}{Output}
    \Input{Training set: $(\mb{x}_i, \mb{y}_{i,1:d,1}, \cdots, \mb{y}_{i,1:d,K})_{i=1}^n$; A new testing instance: $\mb{x}$}
    \Output{The predicted segmentation for the testing instance $\widehat{\pmb{\Delta}}(\mb{x})$ }
    \If{multilabel outcome}
    {
    \textbf{Multilabel Prob Est}. Estimate the prob function $\widehat{\mb{Q}}$ via \eqref{eqn:multi_label_prob_est}\;
    }
    \If{multiclass outcome}
    {
    \textbf{Multiclass Prob Est}. Estimate the prob function $\widehat{\mb Q}$ via \eqref{eqn:multi_class_prob_est}\;
    }
    
    \For{$k=1, \cdots, K$}
    {
        \textbf{Class-specific RankDice}. Obtain $\widehat{\pmb{\Delta}}_k(\mb{x})$ from Lines 2-22 in Algorithm \ref{algo:rankdice} based on the estimated prob $\widehat{\mb{Q}}_k(\mb{x})$.
    }
    \Return{The predicted segmentation $ \widehat{\pmb{\Delta}}(\mb{x}) = (\widehat{\pmb{\Delta}}_{1}(\mb{x}), \cdots, \widehat{\pmb{\Delta}}_{K}(\mb{x})) $}

    \caption{\textbf{mRankDice} for overlapping mDice-segmentation.}
    \label{algo:mrankdice_ol}
\end{algorithm}

\noindent \textbf{mRankDice (multiclass/multilabel + non-overlapping segmentation)} is a quite difficult scenario for developing \textit{mRankDice} from \textit{RankDice} in Dice-segmentation. As indicated in Lemma \ref{lem:nonoverlap_bayes_rule}, searching for an optimal non-overlapping mDice-segmentation operator is NP-hard in general. We leave pursuing this topic as future work.  

\section{Theory}
\label{sec:theory}
In this section, we establish a theoretical foundation of segmentation, including the concept of Dice-calibration, the excess risk of the Dice metric, and the rate of convergence with respect to the excess risk of the proposed \textit{RankDice} framework. For illustration, we focus on Dice-segmentation, and the results can be extended to mDice-segmentation, and \textit{RankIoU} in terms of the IoU metric.

\subsection{Dice-calibrated segmentation}
\label{sec:Dice-calibrated}

In Theorem \ref{thm:Dice_bayes}, the Bayes rule of Dice-segmentation is obtained. To carry this agenda further, we propose concept of ``Dice-calibrated'', following the same philosophy of Fisher consistency or classification-calibration in \citet{lin2004note,zhang2004statistical,bartlett2006convexity}. Intuitively, Dice-calibration is the weakest possible condition on a consistent segmentation method with respect to the Dice metric, this is, at population level, a method ultimately searches for the Bayes rule that achieves the optimal Dice metric. Figure \ref{fig:theory} indicates the overview and logical relations among the theoretic results in this section.

\begin{figure}[h]
    \centering
    \includegraphics[scale=.45]{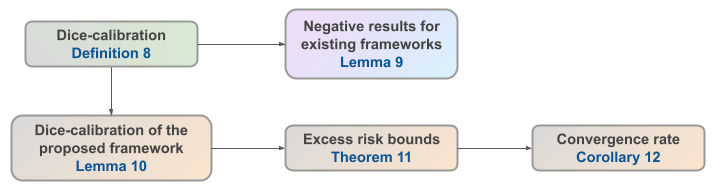}
    \caption{Flowchart of the theory for \textit{RankDice} in Section \ref{sec:Dice-calibrated}, indicating the logical relations among the developed theorems. }
    \label{fig:theory}
\end{figure}


To proceed, let $\mathcal P$ be the class of all probability measures $\mathbb{P}_{\mb{X},\mb{Y}}$ on (Borel) measurable subsets of $\mathbb{R}^d \times \{0,1\}^d$ such that $(\mb{X},\mb{Y}) \sim \mathbb{P}_{\mb{X},\mb{Y}}$, $(\mb{X},\mb{Y})\in \mathbb{R}^d\times\{0,1\}^d$, and $Y_{i} \perp Y_{j} | \mb{X}$ for $i\neq j$.
Denote $\mathcal{D}$ as the class of all (Borel) measurable segmentation operators $\pmb{\delta}: \mb{x} \in \mathbb{R}^d \to \pmb{\delta}(\mb{x}) \in \{0,1\}^d$. The definition of Dice-calibrated segmentation is given as follows.

\begin{definition}[Dice-calibrated segmentation]
    \label{def:Dice-calibrated}
    Given $\gamma \geq 0$, a (population) segmentation method $\mathcal{M}_\gamma: \mathcal{P} \to \mathcal{D}$ is Dice-calibrated if, for any $\mathbb{P}_{\mb{X}, \mb{Y}} \in \mathcal{P}$,
    $$
    \Dice_\gamma\big( \mathcal{M}_\gamma(\mathbb{P}_{\mb{X}, \mb{Y}} ) \big) = \Dice_\gamma( \pmb{\delta}^*),
    $$
    where $\pmb{\delta}^*$ is the Bayes rule for Dice-segmentation defined in Theorem \ref{thm:Dice_bayes}.
\end{definition}

Now, we show that most existing loss functions, under the framework \eqref{eqn:existing_framework}, are not Dice-calibrated.

\begin{lemma}
    \label{lem:Dice-calibrated} 
    Given a loss function $l(\cdot, \cdot)$, define $\mathcal{M}_\gamma(\mathbb{P}_{\mb{X}, \mb{Y}})$ under the framework \eqref{eqn:existing_framework}, that is,
    \begin{align*}
        \mathcal{M}_\gamma(\mathbb{P}_{\mb{X}, \mb{Y}})(\mb{x}) := \mb{1}( \widetilde{\mb{q}}_l(\mb{x}) \geq 0.5 ), \quad \widetilde{\mb{q}}_l = \argmin_{\mb{q}} \ \mathbb{E}\big(l(\mb{Y}, \mb{q}(\mb{X})) \big).
    \end{align*}
    Then, $\mathcal{M}_\gamma(\mathbb{P}_{\mb{X}, \mb{Y}})$ is not Dice-calibrated for $\gamma = 0$ when $l(\cdot, \cdot)$ is 
    any classification-calibrated loss, including the cross-entropy loss and the focal loss.
\end{lemma}

Lemma \ref{lem:Dice-calibrated} indicates that the existing framework \eqref{eqn:existing_framework} with most losses ultimately yields a suboptimal solution to Dice-segmentation, even if a ``classification-calibrated'' loss, such as the cross-entropy loss or the focal loss \citep{charoenphakdee2021focal}, is used. 
Meanwhile, as indicated in \citet{bertels2019optimization}, the optimization with the soft-Dice loss can introduce a volumetric bias for tasks with high inherent uncertainty. In sharp contrast, the proposed \textit{RankDice} method is Dice-calibrated (Lemma \ref{lem:RankDice-calibrated}) and its asymptotic convergence rate in terms of the Dice metric is provided in Theorem \ref{thm:risk_bound}.

To proceed, we give the definition of \textit{RankDice} at population level in Appendix \ref{sec:pop_rankdice}, which replaces the average in \eqref{eqn:prob_est} by the population expectation. Moreover, the cross-entropy loss in \eqref{eqn:prob_est} can be extended to an arbitrary \textit{strictly proper} loss \citep{gneiting2007strictly}. The most common strictly proper losses are the cross-entropy loss and the squared error loss. 

\begin{lemma}[Dice-calibrated]
    \label{lem:RankDice-calibrated}
    The proposed RankDice framework with a \textit{strictly proper} loss is Dice-calibrated.
\end{lemma}

Next, we present an excess risk bound in terms of the Dice metric, that is, $\Dice(\pmb{\delta}^*)- \Dice(\widehat{\pmb{\delta}})$.

\begin{theorem}[Excess risk bounds]
    \label{thm:risk_bound}
    Given $\gamma \geq 0$, let $\widehat{\mb{q}}(\cdot)$ be an estimated probability of $\mb{p}(\cdot)$, and $\widehat{\pmb{\delta}}(\cdot)$ be the RankDice segmentation function defined in \eqref{eqn:rankdice-pred} based on $\widehat{\mb{q}}(\cdot)$, then
    \begin{equation}
        \label{eqn:risk_bound}
        \Dice_\gamma(\pmb{\delta}^*)- \Dice_\gamma(\widehat{\pmb{\delta}}) \leq C_1 \mathbb{E}_\mb{X} { \|\widehat{\mb{q}}(\mb{X}) - \mb{p}(\mb{X}) \|_1 },
    \end{equation}
    where $C_1 > 0$ is a universal constant depending only on $\gamma$. 
\end{theorem}
As indicated in Theorem \ref{thm:risk_bound}, the excess risk of the Dice metric for the proposed \textit{RankDice} framework is upper bounded by the total variation (TV) distance between the estimated probability $\widehat{\mb{q}}$ and the true probability $\mb{p}$. Note that the Kullback-Leibler divergence (the excess risk for the cross-entropy) dominates the TV distance. It follows that if the KL divergence between $p_j$ and $\widehat{q}_j$ goes to 0, then $\widehat{\mb{q}}$ converges to $\mb{p}$ in the TV sense, and so does $\Dice_\gamma(\widehat{\pmb{\delta}})$ to $\Dice_\gamma(\pmb{\delta}^*)$.


Taken together, we present the rate of convergence for the empirical estimator obtained from the proposed \textit{RankDice} framework (Steps 1-3) in Section \ref{sec:binary_rankdice}.

\begin{corollary}[Convergence rate]
    \label{thm:rate}
    Let $\widehat{\mb{q}}(\cdot)$ and $\widehat{\pmb{\delta}}(\cdot)$ be obtained by the proposed \textit{RankDice} framework (Steps 1-3) in Section \ref{sec:binary_rankdice}, and 
    \begin{equation*}
        \mathcal{E}_{\text{CE}}(\widehat{\mb{q}}) := \mathbb{E}\Big( l_{\text{CE}}\big(\mb{Y}, \widehat{\mb{q}}(\mb{X}) \big) \Big) - \mathbb{E}\Big( l_{\text{CE}}\big(\mb{Y}, \mb{p}(\mb{X}) \big) \Big) = O_P( \epsilon_n ),
    \end{equation*}
    where $l_{\text{CE}}(\cdot, \cdot)$ is defined in \eqref{eqn:classification_loss}.
    Then,
    \begin{equation}
        \label{eqn:rate_Dice}
        \Dice_\gamma(\pmb{\delta}^*)- \Dice_\gamma(\widehat{\pmb{\delta}}) = O_P( \sqrt{d} \epsilon^{1/2}_n ).
    \end{equation}
\end{corollary}
Note that $\mathcal{E}_{\text{CE}}$ is the excess risk of the cross-entropy loss or the negative conditional log-likelihood in \eqref{eqn:prob_est}, and its asymptotics as well as a rate of convergence can be established based on statistical learning theory of empirical risk minimization \citep{pollard1984convergence,shen1997methods,bartlett2005local,gine2006concentration,cucker2007learning}, which depends on the sample size and the complexity of the probability class. Then, the rate of convergence of the excess risk in terms of the Dice metric is obtained via \eqref{eqn:rate_Dice}. Note that both \eqref{eqn:risk_bound} and \eqref{eqn:rate_Dice} are derived for a fixed dimension, and the upper bounds can be extended and improved when the dimension of segmentation grows with the sample size. 

Finally, we briefly discuss the connections of the developed theory (i.e., Lemma \ref{lem:RankDice-calibrated}, Theorem \ref{thm:risk_bound} and Corollary \ref{thm:rate}) with the existing results. For example, \cite{popordanoska2021relationship} derived an upper bound for the volume bias $\|\mathbb{E}_{\mb{X}} (\widehat{\mb{q}}(\mb{X}) - \mb{p}(\mb{X}))\|_1$ in terms of the TV distance. It is worth noting that the volume bias focuses on conditional probability estimation and a small volume bias may not necessarily yield a consistent segmentation rule in terms of the Dice metric. In contrast, our result on the excess risk $\Dice(\pmb{\delta}^*) - \Dice(\widehat{\pmb{\delta}})$ characterizes the performance of segmentation rule $\widehat{\pmb{\delta}}$. Besides, \cite{bao2020calibrated} proved the consistency of their method under a linear fractional approximation of Dice metric (see Appendix \ref{sec:E-Dice}), which seems not directly comparable to ours.

\subsection{Relation between Dice and IoU metrics}
\label{sec:IoU}

In this section, we consider the relation and difference between Dice and IoU metrics, and present the Bayes rule for IoU-segmentation. 

\begin{lemma}
    \label{lem:IoU_bayes}
A segmentation rule $\pmb{\delta}^*$ is a global maximizer of ${\IoU}_\gamma(\pmb{\delta})$ if and only if it satisfies that
\[ \delta_j^*(\mb{x}) = 
   \begin{cases} 
      1 & \text{if } p_j(\mb{x}) \text{ ranks top }\tau^*(\mb{x}), \\
      0 & \text{otherwise}.
   \end{cases}
\]
The optimal volume $\tau^*(\mb{x})$ is given as
\begin{equation}
    \label{eqn:IoU_bayes}
    \tau^*(\mb{x}) = \argmax_{\tau \in \{0, 1, \cdots, d\} } \ \Big( \sum_{j \in J_\tau(\mb{x}) } p_j(\mb{x}) + \gamma \Big) \sum_{l=0}^{d-\tau} \frac{1}{\tau + l + \gamma} \mathbb{P}\big( \Gamma_{\sm J_{\tau}(\mb{x})}(\mb{x}) = l \big),
\end{equation}
where $J_\tau(\mb{x}) = \big \{ j : \sum_{j'=1}^d \mathbb{I} \big( p_{j'}(\mb{x}) \geq p_j(\mb{x}) \big) \leq \tau \big\}$ is the index set of the $\tau$-largest conditional probabilities with $J_0(\mb{x}) = \emptyset$, and ${\Gamma}_{\sm J_{\tau}(\mb{x})}(\mb{x}) = \sum_{j' \notin J_{\tau}(\mb{x})} {B}_{j'}(\mb{x})$ is a Poisson-binomial random variable, and ${B}_j(\mb{x})$ is a Bernoulli random variable with the success probability $p_{j}(\mb{x})$.
\end{lemma}



In view of Lemma \ref{lem:IoU_bayes}, IoU-segmentation shares a substantial similarity with Dice-segmentation in terms of the Bayes rule. On this ground, a consistent \textit{RankIoU} framework is also developed based on a \textit{plug-in} rule by replacing $\mb{p}(\mb{x})$ as $\widehat{\mb{q}}(\mb{x})$.
Specifically, \textit{RankIoU} comprises three steps, where \textbf{Steps 1-2} are the same as in \textit{RankDice}; see Section \ref{sec:binary_rankdice}.

\noindent \textbf{Step 3$\mb{'}$ (IoU volume estimation)}: From \eqref{eqn:vol_est_true_pro}, we estimate the volume $\widehat{\tau}(\mb{x})$ by replacing the true conditional probability $\mb{p}(\mb{x})$ with the estimated one $\widehat{\mb{q}}(\mb{x})$:
\begin{equation*}
    \widehat{\tau}(\mb{x}) = \argmax_{\tau \in \{0, 1, \cdots, d\} } \ \Big( \sum_{j \in J_\tau(\mb{x}) } \widehat{q}_j(\mb{x}) + \gamma \Big) \sum_{l=0}^{d-\tau} \frac{1}{\tau + l + \gamma} \mathbb{P}\big( \widehat{\Gamma}_{\sm J_{\tau}(\mb{x})}(\mb{x}) = l \big),
\end{equation*}
where $\widehat{\Gamma}_{\sm J_{\tau}(\mb{x})}(\mb{x}) = \sum_{j \notin J_{\tau}(\mb{x})} \widehat{B}_{j}(\mb{x})$ is a Poisson-binomial random variable, and $\widehat{B}_j(\mb{x})$ are independent Bernoulli random variables with success probabilities $\widehat{q}_{j}(\mb{x})$; for $j=1,\cdots,d$.

Similar to \textit{RankDice}, the predicted IoU-segmentation $\widehat{\pmb{\delta}}(\mb{x}) = (\widehat{\delta}_{1}(\mb{x}), \cdots, \widehat{\delta}_{d}(\mb{x}))^\intercal$ is produced by taking the top-$\widehat{\tau}(\mb{x})$ conditional probabilities:
\begin{equation*}
    \widehat{\delta}_{j}(\mb{x}) =  1, \text{ if } j \in \{ j_1, \cdots, j_{\widehat{\tau}(\mb{x})} \}; \quad \widehat{\delta}_{j}(\mb{x}) =  0, \text{ otherwise.}
\end{equation*}

For multiclass/multilabel segmentation, the conditional probability estimation \eqref{eqn:multi_label_prob_est} and \eqref{eqn:multi_class_prob_est} are carried over into \textit{mRankIoU} and the subsequent ranking and volume estimation remain the same as \textbf{Step 2} and \textbf{Step 3$\mb{'}$} in binary segmentation. 

Computationally, \textit{RankIoU} involves the evaluation of $\mathbb{P}\big( \widehat{\Gamma}_{\sm J_{\tau}(\mb{x})}(\mb{x}) = l \big)$ in \textbf{Step 3$\mb{'}$}.
The FFT algorithm and the truncated refined normal approximation (T-RNA) are applicable after minor modifications; however, the blind approximation (BA) may not be appropriate due to the discrepancy of $\widehat{\Gamma}(\mb{x})$ and $\widehat{\Gamma}_{\sm J_{\tau}(\mb{x})}(\mb{x})$, especially when the size of $J_{\tau}(\mb{x})$ is large; see Section \ref{sec:compute}. Thus, the computation scheme of \textit{RankIoU} might be relatively expensive in high-dimensional segmentation.
Here, we present a parallel result of Lemma \ref*{lem:shrinkage} to narrow down the searching range in \textbf{Step 3$\mb{'}$} of \textit{RankIoU}.

\begin{lemma}\label{lem:shrinkage_IoU}
    If 
    \begin{equation*}
    \sum_{s=1}^{\tau}\widehat{q}_{j_s}(\mb{x}) + \gamma 
    \geq \frac{\widehat{q}_{j_{\tau+1}}(\mb{x})}{1 - \widehat{q}_{j_{\tau+1}}(\mb{x})} \max\Big(d + \gamma, \frac{( (d-\tau)\widehat{q}_{j_{\tau+1}}(\mb{x}) + \tau + \gamma)^2}{\tau + \gamma}\Big),
    \end{equation*}
    then 
    $\varpi_{\tau}(\mb{x}) \geq \varpi_{\tau'}(\mb{x})$ for all $\tau'>\tau$,
    where 
    \begin{equation*}
        \varpi_{\tau}(\mb{x}) = \Big( \sum_{j \in J_\tau(\mb{x}) } \widehat{q}_j(\mb{x}) + \gamma \Big) \sum_{l=0}^{d-\tau} \frac{1}{\tau + l + \gamma} \mathbb{P}\big( \widehat{\Gamma}_{\sm J_{\tau}(\mb{x})}(\mb{x}) = l \big).
    \end{equation*}
\end{lemma}

Theoretically, the concept of ``IoU-calibrated'' can be established (by replacing \Dice{} as \IoU{} in Definition \ref{def:Dice-calibrated}) and the excess risk bounds can be derived in parallel to Dice-segmentation.

\begin{theorem}
    \label{thm:risk_bound_IoU}
    Given $\gamma \geq 0$, let $\widehat{\mb{q}}(\cdot)$ be an estimated probability of $\mb{p}(\cdot)$, and $\widehat{\pmb{\delta}}(\cdot)$ be the RankIoU segmentation function based on $\widehat{\mb{q}}(\cdot)$, then 
    \begin{equation*}
        \IoU_\gamma(\pmb{\delta}^*)- \IoU_\gamma(\widehat{\pmb{\delta}}) \leq C_2 \mathbb{E}_\mb{X} { \|\widehat{\mb{q}}(\mb{X}) - \mb{p}(\mb{X}) \|_1 },
    \end{equation*}
    where $C_2 > 0$ is a universal constant depending only on $\gamma$. Consequently, if $\mathcal{E}_{\text{CE}}(\widehat{\mb{q}}) = O_P( \epsilon_n )$, then
    \begin{equation*}
        \IoU_\gamma(\pmb{\delta}^*)- \IoU_\gamma(\widehat{\pmb{\delta}}) = O_P( \sqrt{d} \epsilon^{1/2}_n ).
    \end{equation*} 
\end{theorem}

\section{Numerical experiments}
\label{sec:num}
This section describes a set of simulations and real datasets that demonstrate the segmentation performance of the proposed \textit{RankDice} and \textit{mRankDice} frameworks compared with the existing \textit{argmax}- and \textit{thresholding}-based frameworks using various loss functions and network architectures. For illustration, the segmentation performances for all numerical experiments are evaluated by empirical Dice/IoU metrics with $\gamma = 0$, see Appendix \ref{sec:E-Dice}. For the mDice/mIoU metric, the class-specific weight is defined as in \eqref{eqn:Dice_weight}. All experiments are conducted using PyTorch and CUDA on an NVIDIA GeForce RTX 3080 GPU. All Python codes are available for download at our GitHub repository (\url{https://github.com/statmlben/rankseg}).

\subsection{Simulation}
\label{sec:sim}

In this section, we mainly compare the proposed \textit{RankDice} framework with the \textit{thresholding}-based framework \eqref{eqn:existing_framework} in various simulated examples. Note that for Dice-segmentation with binary outcomes, \textit{threshold}- and \textit{argmax}-based frameworks yield the same solution.
Both frameworks require an estimation of conditional probability function in the first stage. Therefore, in order to convincingly demonstrate the difference between two frameworks, in our simulation, we assume the true conditional probabilities $p_j(\mb{x}) = \mathbb{P}(Y_j | \mb{x}); j = 1, \cdots, d$ are perfectly estimated, and report the Dice metric of the downstream segmentation produced by two different frameworks.

\noindent \textbf{Example 1.} To mimic the spatial smoothness in practical segmentation problem, especially for image segmentation, the simulated dataset based on matrix response ($d = W \times H$) is generated as follows. First, the true probability matrix $\mb{P} = (p_{wh})_{W \times H}$ are generated by two patterns:
\begin{itemize}
    \item Step decay: $p_{wh} \sim U(0.5,1)$, if $w \leq \floor{\rho W}$ and $h \leq \floor{\rho H}$; $p_{wh} \sim N_{[0,1]}(\beta,0.1)$, otherwise.
    \item Exponential decay: $p_{wh} = \exp\big(- \beta(w + h)\big)$, as visualized in the upper panel of Figure \ref{fig:sim}.
    \item Linear decay: $p_{wh} = 1 -  \beta(w + h) / (W + H)$, as visualized in the lower panel of Figure \ref{fig:sim}.
\end{itemize}
Here $\beta$ is a decay parameter, $U(0,1)$ is the uniform distribution, and $N_{[0,1]}(\beta,0.1)$ is the truncated normal distribution with mean $\beta$ and standard deviation 0.1. In our simulation, we consider $\beta = 0.1, 0.3, 0.5$ with $\rho = 0.1$ for step decay, $\beta = 1.01, 1.05, 1.10$ for exponential decay, and $ \beta = 1, 2, 4$ for linear decay. For each case, four different dimensions are considered: $W=H=28, 64, 128, 256$, and therefore $d$ increases from 784 to 65,536. Then, the proposed \textit{RankDice} framework and the \textit{thresholding}-based framework (at 0.5) are conducted on the true probability matrix $\mb{P}$. Both decay scenarios are replicated 100 times, and the averaged Dice metrics and its standard errors are summarized in Table \ref{tab:sim}. 

\noindent \textbf{Example 2.} As indicated in Theorem \ref{thm:Dice_bayes} and Remark \ref{rk:T}, the optimal segmentation volume (or the optimal threshold) varies significantly across different inputs. This example aims to illustrate the \textit{suboptimality of  (tuning a threshold of) a fixed thresholding} based on step decay in Example 1. To this end, the conditional probability matrices $(\mb{P}_i)_{i=1, \cdots, n}$ of inputs are generated as follows. First, $\rho_i \sim U(0,1)$ to mimic the different segmentation scales/patterns over images in real applications. Next, $\mb{P}_i$ is generated by step decay based on $\beta=0.1$ and $\rho=\rho_i$. Then, the proposed \textit{RankDice} framework and the fixed \textit{thresholding}-based framework (at 0.1. $\cdots$, 0.9) are applied as the same manner in Example 1. All methods are applied with  $n = 2000$ and $W = H = 64$, and the averaged Dice metrics and its standard errors are summarized in Table \ref{tab:sim2}. In this example, the heterogeneous conditional probabilities yield different optimal segmentation volumes (or thresholds) for images, thus (tuning a threshold on) a fixed thresholding leads to a suboptimal solution. To better illustrate the adaptiveness over optimal thresholds, we also present the optimal thresholds for different images with various generating parameters $(\beta, \rho)$ on segmentation patterns of Example 2 in Figure \ref{fig:sim2}, and similar results are demonstrated in Figure \ref{fig:VOC_heatmap} for Pascal VOC 2012 dataset.

The major conclusions on the simulated examples are listed as follows.
\begin{itemize}
    \item It is evident that RankDice significantly outperforms the \textit{thresholding}-based (at 0.5) framework in both decay scenarios with various dimensions (Table \ref{tab:sim}). 
    The substantial improvement is consistent with the findings of Lemma \ref{thm:Dice_bayes}, which indicates that segmentation and classification are entirely distinct problems.
    \item Interestingly, the amount of improvement is gradually increased when the decay of probabilities becomes progressively faster (for exponential decay and linear decay). This suggests that the proposed \textit{RankDice} might be even more advantageous in \textit{well-separated segmentation}.
    \item As suggested in Table \ref{tab:sim2}, the proposed \textit{RankDice} generally outperforms a fixed thresholding framework (with any threshold). This because that the optimal threshold can vary greatly across inputs, as indicated in Figure \ref{fig:sim2}. Moreover, tuning the threshold may improve the performance of the thresholding-based framework, yet it still leads to a suboptimal solution compared with the proposed \textit{RankDice}.
\end{itemize}




\begin{figure}[h!]
    \centering
    \includegraphics[scale=0.35]{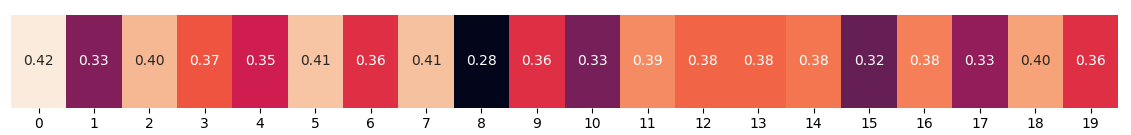}
    \caption{The heatmap for averaged (over all validation samples) optimal thresholding probabilities (i.e. the probability cutpoint $\widehat{q}_{j_{\widehat{\tau}}}$) provided by the proposed \textit{RankDice}, see Remark \ref{rk:T} for more details. Here, the $x$-axis indicates the class, and the results are provided by a PSPNet trained with the cross-entropy loss. }
    \label{fig:VOC_heatmap}
\end{figure}

\subsection{Real datasets}
\label{sec:app}
This section examines the performance of the proposed \textit{RankDice} framework in the PASCAL VOC 2012 \citep{pascal-voc-2012}, the fine-annotated CityScapes semantic segmentation benchmark \citep{cordts2016cityscapes}, and Kvasir SEG dataset polyp segmentation dataset \citep{ajtai19830}. Three different neural network architectures are considered: DeepLab-V3+ \citep{chen2018encoder} with resnet101 as the backbone, PSPNet with resnet50 as the backbone, and FCN8 with resnet101 as the backbone. We report both mDice and mIoU metrics of the segmentation produced by \textit{Threshold}, \textit{Argmax} and \textit{RankDice} on the \textit{same} estimated network/probability. 
Note that only overlapping segmentation is considered.

\noindent \textbf{Fine-annotated Cityscapes dataset} contains 5,000 high quality pixel-level annotated images. For all methods, we employ SGD on the learning rate (\texttt{lr}) schedule \texttt{lr\_schedule=`poly'}, and the initial learning rate \texttt{initial\_lr=0.01}, \texttt{weight\_decay=100}, \texttt{momentum=0.9}, crop size 512x512, batch size 6, and 300 epochs. The performance on validation set is measured in terms of the mDice and mIoU averaged across 19 object classes (Table \ref{tab:CityScapes}).

\noindent \textbf{Pascal VOC 2012 dataset} contains 20 foreground object classes and one background class. The dataset contains 1,464 training and 1,449 validation pixel-level annotated images. We augment the dataset by using the additional annotations provided by \citet{hariharan2011semantic}. For all methods, we employ SGD on \texttt{lr\_schedule=`poly'}, and the initial learning rate \texttt{initial\_lr=0.01}, \texttt{weight\_decay=100}, \texttt{momentum=0.9}, crop size 480x480, batch size 8, and an early stop with patient 10 based on validation loss. The performance on validation set is measured in terms of the mDice and mIoU averaged across the 20 object classes (Table \ref{tab:VOC}). In this dataset, we also present a heatmap (Figure \ref{fig:VOC_heatmap}) for averaged minimal estimated probabilities for segmented features (i.e. $\widehat{q}_{j_{\widehat{\tau}}}$) by the proposed \textit{RankDice}, to highlight its difference to the \textit{thresholding}-based framework.

\noindent \textbf{Kvasir SEG dataset} contains 1000 polyp images and their ground truth segmentation (a single class) from the Kvasir dataset. The scale of the images varies from 332x487 to 1920x1072 pixels. For all methods, we employ SGD on \texttt{lr\_schedule=`poly'}, and the initial learning rate \texttt{initial\_lr=0.01}, \texttt{weight\_decay=100}, \texttt{momentum=0.9}, crop size 320x320, batch size 8, and 140 epochs. The performance on testing set is measured in terms of the Dice and IoU (Table \ref{tab:kvasir}). 

\noindent \textbf{Multiclass/multilabel loss.} In both datasets, we use six loss functions (including multiclass and multilabel losses) in the implementation, including the cross-entropy (CE), the focal loss (Focal), the binary cross-entropy (BCE), the soft-Dice loss (Soft-Dice), the binary soft-Dice loss (B-Soft-Dice), and the LovaszSoftmax loss (LovaszSoftmax). For multiclass losses, including CE, Focal, Soft-Dice, and LovaszSoftmax, we use the softmax function as the output layer activation function of a neural network. For multilabel losses, including BCE and B-Soft-Dice, we use the sigmoid function as the output layer activation function of a neural network.

\noindent \textbf{Dice-segmentation based on a single class.} For the first two datasets, when we focus on a single object class, it reduces to a Dice-segmentation with binary outcomes. To examine the performance for the proposed \textit{RankDice} in binary segmentation, we also report the Dice/IoU metric for each label separately. In principle, we need to train a model only on the binary label for an object class, then produce the segmentation prediction by \textit{thresholding} (or \textit{argmax}) and \textit{RankDice}. However, based on our empirical study, a model trained from full labels is significantly better than the one trained from a single binary label. We thus train a model based the same procedure in the aforementioned segmentation with multiclass/multilabel losses on full labels, and then produce the prediction based on the estimated probability for each object class separately. The best two performances (``PSPNet + CE'' and ``PSPNet + BCE'') for both datasets are summarized in Tables \ref{tab:CityScapes_binary} and \ref{tab:VOC_binary}. The performance for other models and losses can be found in the supplementary. In Figure \ref{fig:VOC_example}, we present the segmentation results on illustrative examples for all methods to demonstrate the difference between the proposed \textit{RankDice} and the existing methods.

\noindent \textbf{The fixed-thresholding framework with different thresholds.} As indicated in Remark \ref{rk:T} and Example 2 in Section \ref{sec:sim}, the optimal segmentation volume (or the optimal thresholding) varies significantly across different images, thus (tuning on) a fixed-thresholding yields a suboptimal solution. In Tables \ref{tab:Dthold} and \ref{tab:sim2}, we also report the numerical performance based on different thresholds for real datasets and Example 2, respectively. 

\noindent \textbf{Probability calibration via Temperature scaling (TS).} According to Theorem \ref{thm:risk_bound}, the consistency of the proposed method holds if the estimated conditional probabilities are calibrated. Alternatively, improving the probability calibration may improve the segmentation performance for the proposed framework. Therefore, in Table \ref{tab:calibration}, we examine the numerical results of the proposed method via TS (with different tuning temperatures), which is one of the most effective probability calibration methods as suggested by \cite{guo2017calibration}.

\begin{table*}[h]\centering
    \scalebox{.75}{
    \begin{tabular}{@{}clllcccccccccccccc@{}} \toprule
     & Model & \phantom{a} & Loss & \phantom{a} & Threshold (at 0.5) & \phantom{a} & Argmax & \phantom{a} & mRankDice (our)  \\
     & &&  && (mDice, mIoU) ($\times .01$) &&  (mDice, mIoU) ($\times .01$) && (mDice, mIoU) ($\times .01$) \\
     \midrule
    & DeepLab-V3+ && CE && (56.00, 48.40) && (54.20, 46.60) && (\textbf{57.80}, \textbf{49.80}) \\
    & (resnet101) && Focal && (54.10, 46.60) && (53.30, 45.60) && \gray{({56.50}, {48.70})} \\
    & && BCE && (49.80, 24.90) && (44.20, 22.10) && ({54.00}, {27.00}) \\
    & && Soft-Dice && (39.50, 35.90) && (39.50, 35.90) && \gray{(39.50, 35.90)} \\
    & && B-Soft-Dice && (41.00, 20.50) && (27.60, 13.80) && \gray{(41.10, 20.50)} \\
    & && LovaszSoftmax && (55.20, 47.60) && (52.30, 45.10) && \gray{(55.50, 47.80)} \\
    \midrule
    & PSPNet && CE && (57.50, 49.60) && (56.50, 48.50) && (\textbf{59.30}, \textbf{51.00}) \\
    & (resnet50) && Focal && (56.00, 48.20) && (55.80, 47.70) && \gray{({58.20}, {50.00})} \\
    & && BCE && (51.40, 25.70) && (47.60, 23.80) && ({55.10}, {27.60}) \\
    & && Soft-Dice && (49.10, 43.50) && (48.70, 43.20) && \gray{(49.30, 43.60)} \\
    & && B-Soft-Dice && (46.30, 23.10) && (32.70, 16.40) && \gray{(46.20, 23.10)} \\
    & && LovaszSoftmax && (56.80, 48.90) && (55.40, 47.70) && \gray{(56.70, 49.10)} \\
    \midrule
    & FCN8 && CE && (51.40, 43.70) && (50.50, 42.60) && (\textbf{53.50}, \textbf{45.30}) \\
    & (resnet101) && Focal && (48.50, 41.20) && (49.60, 41.60) && \gray{({51.50}, {43.70})} \\
    & && BCE && (39.40, 19.70) && (39.40, 19.70) && ({41.30}, {20.60}) \\
    & && Soft-Dice && (28.30, 24.30) && (28.30, 24.30) && \gray{(28.30, 24.80)} \\
    & && B-Soft-Dice && (29.10, 14.60) && (29.10, 14.60) && \gray{(29.10, 14.60)} \\
    & && LovaszSoftmax && (48.10, 40.40) && (42.90, 35.80) && \gray{(48.90, 40.90)} \\
    \bottomrule
    \end{tabular}}
    \caption{Averaged mDice and mIoU metrics of \textit{Threshold}, \textit{Argmax}, and the proposed \textit{mRankDice} based on state-of-the-art models/losses on \textbf{Fine-annotated CityScapes} \textit{val} set. \gray{Gray color} indicates that \textit{RankDice}/\textit{mRankDice} is inappropriately applied to a loss function which is not strictly proper. The best performance in each model is bold-faced.}
    \label{tab:CityScapes}
\end{table*}

\begin{table*}[h]\centering
    \scalebox{.75}{
    \begin{tabular}{@{}clllcccccccccccccc@{}} \toprule
     & Model & \phantom{a} & Loss & \phantom{a} & Threshold (at 0.5) & \phantom{a} & Argmax & \phantom{a} & mRankDice (our)  \\
     & &&  && (mDice, mIoU) ($\times .01$) &&  (mDice, mIoU) ($\times .01$) && (mDice, mIoU) ($\times .01$) \\
     \midrule
    & DeepLab-V3+ && CE && (63.60, 56.70) && (61.90, 55.30) && ({64.01}, {57.01}) \\
    & (resnet101) && Focal && (62.70, 55.01) && (60.50, 53.20) && \gray{({62.90}, {55.10})} \\
    & && BCE && (63.30, 31.70) && (59.90, 29.90) && (\textbf{64.60}, \textbf{32.30}) \\
    & && Soft-Dice && --- && --- && --- \\
    & && B-Soft-Dice && --- && --- && --- \\
    & && LovaszSoftmax && (57.70, 51.60) && (56.20, 50.30) && \gray{(57.80, 51.60)} \\
    \midrule
    & PSPNet && CE && (64.60, 57.10) && (63.20, 55.90) && ({65.40}, {57.80}) \\
    & (resnet50) && Focal && (64.00, 56.10) && (63.90, 56.10) && \gray{({66.60}, {58.50})} \\
    & && BCE && (64.20, 32.10) && (65.20, 32.60) && (\textbf{67.10, 33.50}) \\
    & && Soft-Dice && (59.60, 54.00) && (58.80, 53.20) && \gray{(60.00, 54.30)} \\
    & && B-Soft-Dice && (63.30, 31.60) && (54.00. 27.00) && \gray{(64.30, 32.20)} \\
    & && LovaszSoftmax && (62.00, 55.20) && (60.80, 54.10) && \gray{(62.20, 55.40)} \\
    \midrule
    & FCN8 && CE && (49.50, 41.90) && (45.30, 38.40) && ({50.40}, {42.70}) \\
    & (resnet101) && Focal && (50.40, 41.80) && (47.20, 39.30) && \gray{\bf ({51.50}, {42.50})} \\
    & && BCE && (46.20, 23.10) && (44.20, 22.10) && ({47.70}, {23.80}) \\
    & && Soft-Dice && --- && --- && --- \\
    & && B-Soft-Dice && --- && --- && --- \\
    & && LovaszSoftmax && (39.80, 34.30) && (37.30, 32.20) && \gray{(40.00, 34.40)} \\
    \bottomrule
    \end{tabular}}
    \caption{Averaged mDice and mIoU of \textit{threshold}, \textit{argmax}, and the proposed \textit{mRankDice} based on state-of-the-art models/losses on \textbf{PASCAL VOC 2012} \textit{val} set. ``---'' indicates that either the performance is significantly worse or the training is unstable. \gray{Gray color} indicates that \textit{RankDice}/\textit{mRankDice} is inappropriately applied to a loss function which is not strictly proper. The best performance in each model is bold-faced.}
    \label{tab:VOC}
\end{table*}

\begin{table*}[h!]\centering
    \scalebox{.75}{
    \begin{tabular}{@{}clllcccccccccccccc@{}} \toprule
     & Model & \phantom{a} & Loss & \phantom{a} & Threshold/Argmax & \phantom{a} & mRankDice (our)  \\
     & &&  && (Dice, IoU) ($\times .01$) &&  (Dice, IoU) ($\times .01$) \\
     \midrule
    & DeepLab-V3+ && CE && (87.9, 80.7) && (\textbf{88.3}, \textbf{80.9}) \\
    & (resnet101) && Focal && (86.5, 87.3) && \gray{({83.1}, {73.2})} \\
    & && Soft-Dice && (85.7, 77.8) && \gray{({85.8}, {77.9})} \\
    & && LovaszSoftmax && (84.3, 77.3) && \gray{(84.5, 77.4)} \\
    \midrule
    & PSPNet && CE &&  (86.3, 79.2) && (\textbf{87.1}, \textbf{79.8}) \\
    & (resnet50) && Focal  && (83.8, 75.4) && \gray{({81.8}, {72.4})} \\
    & && Soft-Dice && (83.5, 75.9) && \gray{(83.7, 76.1)} \\
    & && LovaszSoftmax && (86.0, 79.2) && \gray{(86.0, 79.2)} \\
    \midrule
    & FCN8 && CE && (81.9, 73.5) && (\textbf{82.1}, \textbf{73.6}) \\
    & (resnet101) && Focal  && (78.5, 69.0) && \gray{\bf ({70.3}, {58.3})} \\
    & && Soft-Dice && --- && --- \\
    & && LovaszSoftmax && (82.0, 73.4) && \gray{(82.0, 73.4)} \\
    \bottomrule
    \end{tabular}}
    \caption{Averaged mDice and mIoU of \textit{threshold}/\textit{argmax}, and the proposed \textit{mRankDice} based on state-of-the-art models/losses on \textbf{Kvasir SEG} dataset (with a single class segmentation). ``---'' indicates that either the performance is significantly worse. \gray{Gray color} indicates that \textit{RankDice} is inappropriately applied to a loss function which is not strictly proper. The best performance in each model is bold-faced.}
    \label{tab:kvasir}
\end{table*}

Overall, the empirical results show that the proposed \textit{RankDice}/\textit{mRankDice} framework yields good performance in three segmentation benchmarks. The major empirical conclusions on the proposed RankDice are listed as follows. 

\begin{itemize}
    \item As suggested in Tables \ref{tab:CityScapes} and \ref{tab:VOC}, the proposed RankDice framework \textit{consistently} outperforms the \textit{threshold}-based and \textit{argmax}-based frameworks based on the same trained model/network. The percentage of improvement on the best performance (for each framework) are 3.13\% (over \textit{threshold}) and 4.96\% (over \textit{argmax}) for CityScapes dataset (PSPNet + CE), and 3.87\% (over \textit{threshold}) and 2.91\% (over \textit{argmax}) for Pascal VOC 2012 dataset (PSPNet + CE/BCE), and 0.926\% for Kvasir SEG dataset (PSPNet + CE).
    \item For Dice-segmentation based on a single class, as suggested in Tables \ref{tab:CityScapes_binary} and \ref{tab:VOC_binary}, the proposed \textit{RankDice} framework \textit{consistently} outperforms the \textit{threshold}-/\textit{argmax}-based framework for each class. The largest percentage of class-specific improvement in terms of the Dice metric on the best performance (for each framework) is 23.6\% for CityScapes dataset, and 26.9\% for Pascal VOC 2021 dataset. 
    \item The proposed \textit{RankDice} works significantly better in ``difficult'' segmentation. As indicated in Tables \ref{tab:CityScapes_binary} and \ref{tab:VOC_binary}, the improvements by \textit{RankDice} are negatively correlated with the resulting Dice/IoU metrics. It is also suggested by Figure \ref{fig:VOC_example}, where we illustrate three images from classes \texttt{cat} (no improvement) and \texttt{chair} (26.9\% improvement). As presented in all examples of \texttt{chair} and the last example of \texttt{cat}, the improvement is significant when segmentation is difficult (the target object is either occluded or similar with other objects).
    \item As suggested in Table \ref{tab:Dthold}, the empirical results are in line with Remark \ref{rk:T} (theoretically) and Table \ref{tab:sim2} (numerically), suggesting that the proposed \textit{RankDice} generally outperforms a fixed thresholding framework (with any threshold). This because that the optimal threshold can vary greatly across images, as indicated in Figure \ref{fig:VOC_heatmap}. Moreover, tuning the threshold may improve the performance of the fixed-thresholding framework, yet it still leads to a suboptimal solution compared with the proposed \textit{RankDice}.
    \item As suggested in Table \ref{tab:calibration}, the segmentation performance can be potentially improved by TS probability calibration method, especially  4.59\% improvement for CE loss and 3.28\% improvement for BCE in Pascal VOC 2021 dataset. Yet, the TS demands an additional validation dataset to tune the optimal temperature, thus more numerical experiments are required to suggest the effectiveness of this promising method. We leave pursuing this topic as future work.  
    \item As expected, the proposed \textit{RankDice}/\textit{mRankDice} performs well for strictly proper loss functions, including CE and BCE. In addition, we show that the performance is continuously improved compared with existing frameworks for some classification calibrated (only) losses, such as the focal loss. It is possible that this phenomenon is due to the relationship between the estimated scores (from focal loss) and the true conditional probabilities (cf. \cite{charoenphakdee2021focal,liu2021deep}). We leave this topic as future work.
    \item Although the \textit{RankDice}/\textit{mRankDice} framework is developed for Dice/mDice optimization, the performance in terms of the IoU/mIoU metric is also consistently improved. 
\end{itemize}

Moreover, we also present some important observations based on our experiments about losses, frameworks, and models. 
\begin{itemize}
    \item CE, Focal and BCE are the top three losses for Dice-segmentation. While some Dice approximating losses, such as Soft-Dice and binary Soft-Dice, usually lead to suboptimal solutions. 
    \item It seems that the multiclass modeling and multiclass losses are more preferred for both datasets. Moreover, the \textit{threshold}-based framework usually outperforms the \textit{argmax}-based framework for both multiclass and multilabel losses.
    \item For multiclass losses, including CE Focal, Soft-Dice, and LovaszSoftmax, the Dice and IoU metrics are consistent, i.e., a higher Dice yields a higher IoU score. For multilabel losses, including BCE and B-Soft-Dice, there is a significant difference between Dice and IoU metrics.
\end{itemize}

\begin{table*}[h!]
    \centering
    \scalebox{.73}{
    \begin{tabular}{@{}lrccccccccccccccc@{}} \toprule
     & Object class & \phantom{a} & \multicolumn{3}{c}{Threshold/Argmax}  & \phantom{a} &\multicolumn{3}{c}{RankDice (our)} && \textbf{imps.} \\
     & &&  \multicolumn{3}{c}{(Dice, IoU) ($\times .01$)} &&  \multicolumn{3}{c}{(Dice, IoU) ($\times .01$)} && (best vs. best) \\
     & && CE & Focal & BCE  &&  CE & Focal & BCE && (Dice) \\
     \midrule
     & \texttt{road}          && (85.7, 77.2) & (86.1, 77.9) & (92.2, 46.1) && (85.6, 77.1) & (86.0, 77.7) & (92.2, 46.1) && $\sim$ \\
     & \texttt{sidewalk}      && (57.3, 47.8) & (53.6, 43.8) & (43.8, 21.9) && (60.8, 50.8) & (58.7, 48.4) & (54.0, 27.0) && 6.1\% \\
     & \texttt{building}      && (84.6, 76.2) & (83.4, 74.8) & (79.4, 39.7) && (85.1, 76.7) & (83.6, 74.7) & (82.1, 41.0) && $\sim$ \\
     & \texttt{wall}          && (17.4, 13.6) & (16.1, 12.4) & (04.9, 02.4) && (21.0, 16.4) & (21.5, 16.8) & (08.3, 04.2) && 23.6\% \\
     & \texttt{fence}         && (14.7, 10.9) & (12.4, 08.9) & (15.2, 07.6) && (15.8, 11.7) & (13.7, 09.8) & (19.1, 09.5) && 25.7\% \\
     & \texttt{pole}          && (41.9, 29.0) & (34.7, 23.4) & (27.1, 13.5) && (46.0, 31.7) & (35.6, 23.1) & (36.4, 18.2) && 9.8\% \\
     & \texttt{traffic light} && (34.9, 26.5) & (31.5, 24.0) & (18.7, 09.4) && (37.4, 28.3) & (33.5, 24.7) & (21.3, 10.6) && 7.2\% \\
     & \texttt{traffic sign}  && (49.9, 39.0) & (45.9, 35.1) & (35.3, 17.6) && (51.4, 40.1) & (46.6, 35.1) & (39.6, 19.8) && 3.0\% \\
     & \texttt{vegetation}    && (90.2, 84.1) & (90.2, 83.8) & (89.0, 44.5) && (90.3, 84.1) & (89.6, 82.8) & (89.4, 44.7) && $\sim$ \\
     & \texttt{terrain}       && (25.7, 20.1) & (24.1, 18.5) & (19.8, 09.9) && (29.4, 23.1) & (28.7, 22.7) & (25.3, 12.7) && 14.4\% \\
     & \texttt{sky}           && (83.6, 77.0) & (82.0, 75.2) & (80.1, 40.0) && (84.5, 77.8) & (83.1, 76.2) & (80.7, 40.3) && 1.1\% \\
     & \texttt{person}        && (45.1, 36.3) & (42.6, 34.1) & (32.8, 16.4) && (49.5, 40.0) & (47.6, 38.2) & (38.6, 19.3) && 9.8\% \\
     & \texttt{rider}         && (35.1, 27.3) & (31.2, 24.0) & (18.6, 09.3) && (37.2, 29.2) & (33.9, 26.3) & (24.0, 12.0) && 6.0\% \\
     & \texttt{car}           && (84.1, 76.9) & (83.4, 76.2) & (80.8, 40.4) && (84.0, 76.6) & (81.8, 74.0) & (81.2, 40.6) && $\sim$ \\
     & \texttt{truck}         && (24.7, 21.9) & (25.6, 22.7) & (21.8, 10.9) && (26.6, 23.3) & (28.1, 24.8) & (26.8, 13.4) && 9.8\% \\
     & \texttt{bus}           && (46.8, 42.2) & (48.8, 43.8) & (36.3, 18.2) && (51.3, 46.5) & (51.5, 46.8) & (39.2, 19.6) && 5.5\% \\
     & \texttt{train}         && (34.9, 30.7) & (36.3, 31.0) & (33.8, 16.9) && (35.8, 31.5) & (37.3, 32.2) & (34.7, 17.4) && 2.8\% \\
     & \texttt{motorcycle}    && (19.7, 15.8) & (20.4, 16.1) & (07.0, 03.5) && (22.2, 17.7) & (21.1, 16.8) & (08.7, 04.4) && 8.8\% \\
     & \texttt{bicycle}       && (41.4, 32.5) & (42.1, 32.9) & (32.9, 16.5) && (41.9, 32.6) & (42.0, 32.5) & (36.7, 18.4) && $\sim$ \\
    \bottomrule
    \end{tabular}}
    \caption{Averaged class-specific Dice and IoU metrics of \textit{Threshold}/\textit{Argmax}, and the proposed \textit{RankDice} based on various losses of PSPNet + resnet50 on \textbf{Fine-annotated CityScapes} \textit{val} set. The class-specific improvement in terms of the Dice metric of the proposed RankDice framework is computed in ``imps.'', where $\sim$ indicates that the difference between \textit{Threshold}/\textit{Argmax} and \textit{RankDice} is smaller than 1.0\%.}
    \label{tab:CityScapes_binary}
\end{table*}

\begin{table*}[h!]
    \centering
    \scalebox{.73}{
    \begin{tabular}{@{}lrccccccccccccccc@{}} \toprule
     & Object class & \phantom{a} & \multicolumn{3}{c}{Threshold/Argmax}  & \phantom{a} &\multicolumn{3}{c}{RankDice (our)} && \textbf{imps.} \\
     & &&  \multicolumn{3}{c}{(Dice, IoU) ($\times .01$)} &&  \multicolumn{3}{c}{(Dice, IoU) ($\times .01$)} && (best vs. best) \\
     & && CE & Focal & BCE  &&  CE & Focal & BCE && (Dice) \\
     \midrule
     & \texttt{  Aeroplane} && (71.2, 63.4) & (68.4, 59.2) & (72.9, 36.5) && (71.3, 63.4) & (72.7, 64.1) & (75.3, 37.6) && 3.3\% \\
     & \texttt{    Bicycle} && (37.1, 25.9) & (19.6, 12.4) & (14.6, 7.30) && (38.7, 27.3) & (30.5, 20.6) & (23.1, 11.5) && 4.3\% \\
     & \texttt{       Bird} && (76.0, 68.2) & (74.3, 65.2) & (74.2, 37.1) && (76.6, 68.7) & (75.8, 66.4) & (76.3, 38.1) && $\sim$ \\
     & \texttt{       Boat} && (51.1, 42.7) & (59.5, 49.1) & (55.5, 27.8) && (51.3, 42.9) & (61.9, 51.5) & (61.0, 30.5) && 4.0\% \\
     & \texttt{     Bottle} && (42.8, 35.8) & (36.2, 30.0) & (39.1, 19.6) && (44.2, 36.8) & (37.6, 31.4) & (41.1, 20.6) && 3.3\% \\
     & \texttt{        Bus} && (72.8, 68.3) & (72.3, 67.5) & (74.8, 37.4) && (74.1, 69.6) & (73.5, 68.8) & (75.9, 37.9) && 1.5\% \\
     & \texttt{        Car} && (53.5, 47.5) & (51.1, 45.6) & (48.9, 24.4) && (55.0, 49.0) & (53.6, 47.9) & (51.7, 25.9) && 2.7\%\\
     & \texttt{        Cat} && (75.0, 69.2) & (74.1, 67.9) & (73.1, 36.6) && (75.5, 69.7) & (75.4, 68.7) & (75.1, 37.6) && $\sim$ \\
     & \texttt{      Chair} && (17.5, 12.8) & (16.7, 11.6) & (10.2, 5.10) && (19.6, 14.4) & (22.2, 16.1) & (14.5, 7.30) && 26.9\% \\
     & \texttt{        Cow} && (65.3, 58.6) & (60.1, 53.7) & (64.9, 32.4) && (66.5, 59.9) & (62.3, 56.0) & (68.4, 34.2) && 4.8\% \\
     & \texttt{Diningtable} && (32.9, 27.5) & (33.6, 27.4) & (31.7, 15.9) && (34.5, 29.2) & (38.6, 32.4) & (35.3, 17.6) && 14.9\% \\
     & \texttt{        Dog} && (64.6, 57.9) & (71.0, 63.4) & (71.7, 35.9) && (65.5, 58.7) & (72.5, 64.9) & (74.4, 37.2) && 3.8\% \\
     & \texttt{      Horse} && (63.9, 55.3) & (67.3, 58.3) & (67.0, 33.5) && (65.3, 56.6) & (69.5, 60.1) & (70.9, 35.4) && 5.4\% \\
     & \texttt{  Motorbike} && (69.7, 60.6) & (65.5, 56.7) & (66.9, 33.5) && (71.6, 62.6) & (67.0, 57.9) & (70.1, 35.1) && 2.7\% \\
     & \texttt{     Person} && (67.0, 57.7) & (65.0, 55.4) & (67.4, 33.7) && (67.4, 58.1) & (67.2, 57.6) & (69.7, 34.8) && 3.4\% \\
     & \texttt{Pottedplant} && (26.9, 20.2) & (22.4, 17.3) & (25.5, 12.8) && (29.1, 22.0) & (26.9, 20.7) & (28.6, 14.3) && 8.2\% \\
     & \texttt{      Sheep} && (53.9, 47.4) & (62.8, 55.4) & (62.1, 31.1) && (54.3, 47.9) & (66.0, 58.6) & (66.9, 33.4) && 6.5\% \\
     & \texttt{       Sofa} && (29.8, 25.0) & (29.8, 24.4) & (33.7, 16.8) && (32.0, 26.9) & (34.6, 29.0) & (38.9, 19.4) && 14.5\% \\
     & \texttt{      Train} && (77.7, 71.0) & (75.8, 68.9) & (80.3, 40.1) && (77.9, 71.1) & (77.3, 70.4) & (82.1, 41.1) && 2.2\% \\
     & \texttt{  Tvmonitor} && (48.4, 41.4) & (50.7, 41.6) & (53.7, 26.8) && (49.2, 42.0) & (54.1, 45.4) & (56.4, 28.2) && 5.0\% \\
    \bottomrule
    \end{tabular}}
    \caption{Class-specific Dice and IoU of \textit{Threshold}/\textit{Argmax}, and the proposed \textit{RankDice} based on various losses of PSPNet + resnet50 on \textbf{PASCAL VOC 2012} \textit{val} set. The class-specific improvement in terms of the Dice metric of the proposed \textit{RankDice} framework is computed in ``imps.'', where $\sim$ indicates that the difference between \textit{Threshold}/\textit{Argmax} and the proposed \textit{RankDice} is smaller than 1.0\%.}
    \label{tab:VOC_binary}
\end{table*}

\begin{figure}
    \centering
    \includegraphics[scale=0.275]{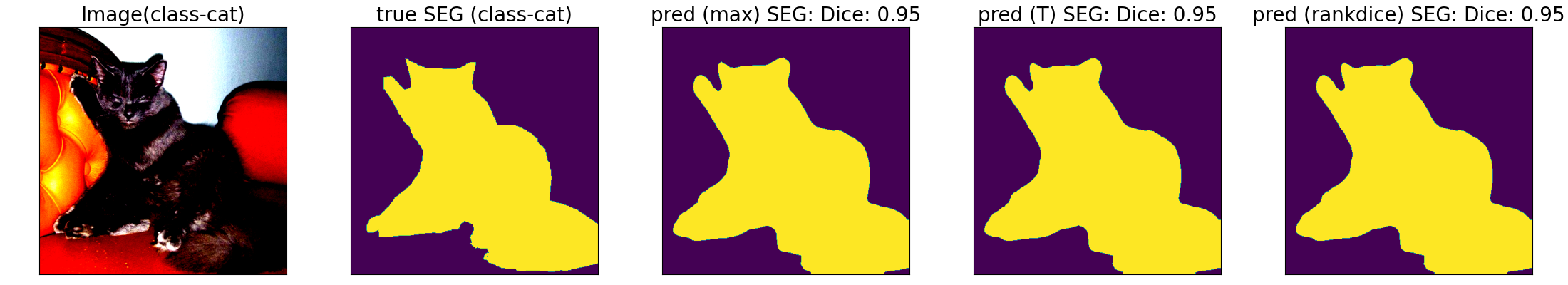}
    \includegraphics[scale=0.275]{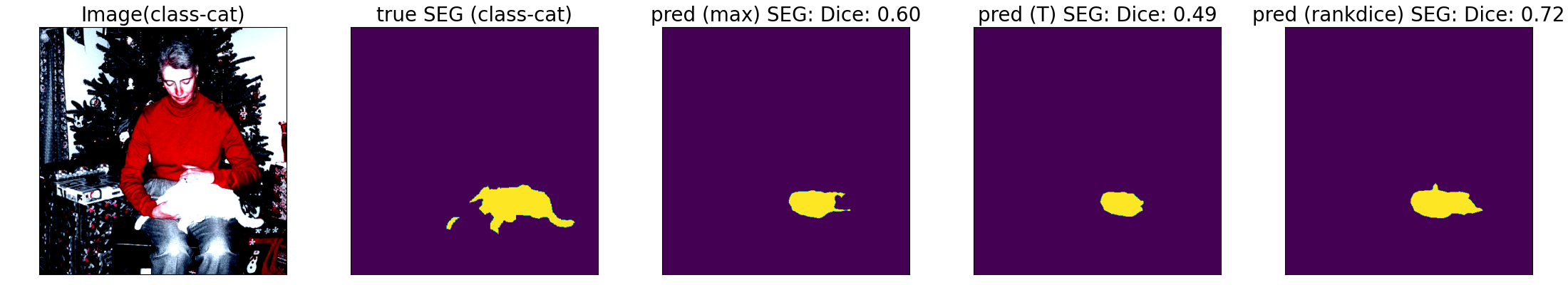}
    \includegraphics[scale=0.275]{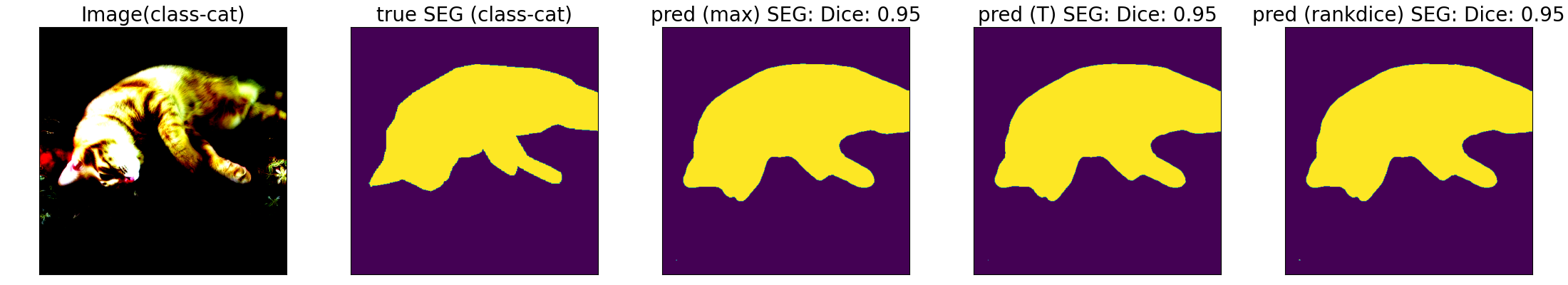}
    \rule[1ex]{\textwidth}{0.1pt}
    \includegraphics[scale=0.275]{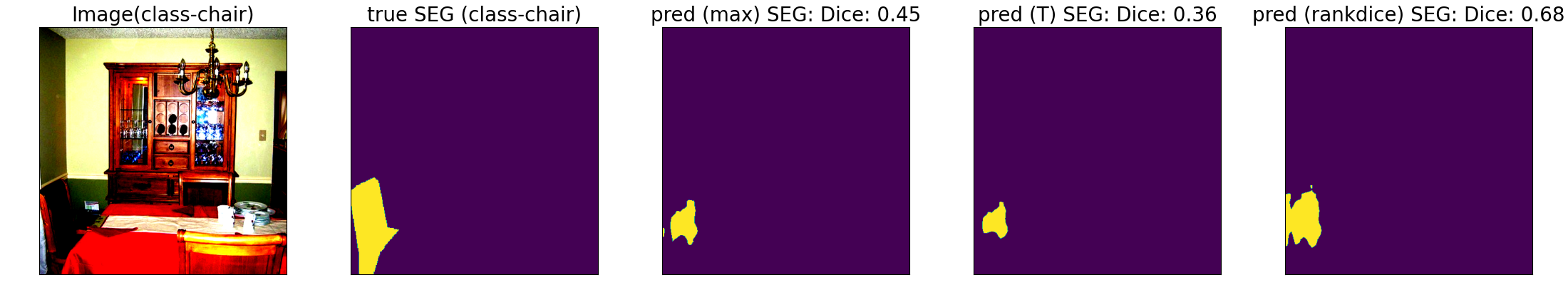}
    \includegraphics[scale=0.275]{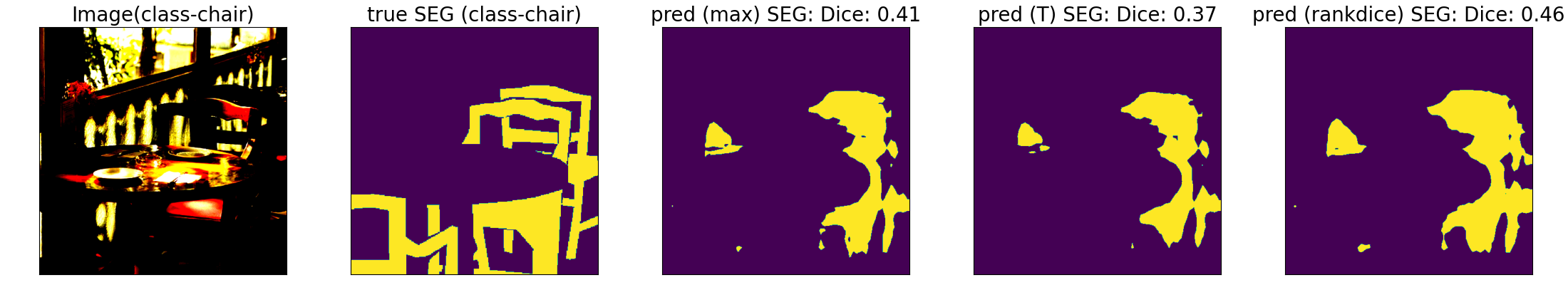}
    \includegraphics[scale=0.275]{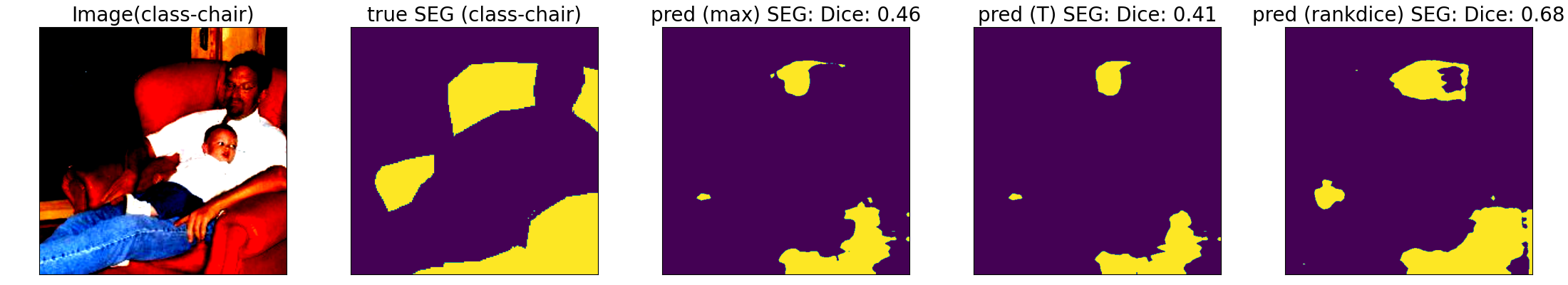}
    \caption{Comparison of segmentation results between the proposed method and existing methods for classes \texttt{cat} (upper panel) and \texttt{chair} (lower panel). Column 1 indicates original images, Column 2 indicates ground truths, and Columns 3-5 indicate the predicted segmentation produced by \textit{argmax}, \textit{thresholding}, and the proposed \textit{RankDice}, respectively. The results are provided by a PSPNet trained with the cross-entropy loss.}
    \label{fig:VOC_example}
\end{figure}

\begin{table}
\begin{minipage}{\textwidth}
    \begin{minipage}[b]{0.49\textwidth}
        \scalebox{.8}{
        \begin{tabular}{@{}cccccccccccccccccc@{}} \toprule
        & Framework & \phantom{a} & Thold & \phantom{a} &  (Dice, IoU) ($\times .01$)  \\
         \midrule
         & \textit{threshold}-based    && 0.1 && (49.10, 24.60) \\
         &                             && 0.2 && (53.00, 26.50) \\
         &                             && 0.3 && (53.60, 26.80) \\
         &                             && 0.4 && (52.90, 26.50) \\
         &                             && 0.5 && (51.40, 25.70) \\
         &                             && 0.6 && (49.60, 24.80) \\
         &                             && 0.7 && (47.00, 23.50) \\
         &                             && 0.8 && (43.40, 21.70) \\
         &                             && 0.9 && (37.40, 18.70) \\
         \cmidrule{1-2}
         & \textit{RankDice}(our)      && --- && \textbf{(\textbf{55.10}, \textbf{27.60})} \\
        \bottomrule
        \end{tabular}}
    \end{minipage}
    \hfill
    \begin{minipage}[b]{0.49\textwidth}
        \scalebox{.8}{
        \begin{tabular}{@{}cccccccccccccccccc@{}} \toprule
        & Framework & \phantom{a} & Thold & \phantom{a} &  (Dice, IoU) ($\times .01$)  \\
         \midrule
         & \textit{threshold}-based    && 0.1 && (56.80, 28.40) \\
         &                             && 0.2 && (63.90, 32.00) \\
         &                             && 0.3 && (65.70, 32.80) \\
         &                             && 0.4 && (65.60, 32.80) \\
         &                             && 0.5 && (64.20, 32.10) \\
         &                             && 0.6 && (62.30, 32.00) \\
         &                             && 0.7 && (59.30, 29.60) \\
         &                             && 0.8 && (54.20, 27.10) \\
         &                             && 0.9 && (43.40, 21.70) \\
         \cmidrule{1-2}
         & \textit{RankDice}(our)      && --- && \textbf{(67.10, 33.50)} \\
        \bottomrule
        \end{tabular}}
    \end{minipage}
\end{minipage}
\caption{The averaged Dice and IoU metrics and their standard errors (in parentheses) of the proposed \textit{RankDice} framework and the fixed-\textit{thresholding} (with different thresholds) framework in \textbf{Fine-annotated CityScapes} (left) and \textbf{PASCAL VOC 2012} (right) datasets. The performance is reported based on PSPNet + resnet50 with the BCE loss.}
\label{tab:Dthold}
\end{table}

\begin{table}
    \scalebox{.9}{
    \begin{tabular}{@{}cccccccccccccccccc@{}} \toprule
    & Dataset    & Temp & \phantom{a} & \multicolumn{2}{c}{Threshold}  & \phantom{a} & \multicolumn{2}{c}{RankDice (our)} \\ 
    &    &      &             & \multicolumn{2}{c}{(Dice, IoU) ($\times .01$)} && \multicolumn{2}{c}{(Dice, IoU) ($\times .01$)}  \\
    &    &      &             & CE & BCE && CE & BCE \\
    \midrule
    & \textbf{CityScapes} & 1.0 && (\textbf{57.50}, \textbf{49.60}) & (51.40, 25.70) && (59.30, 51.00) & (\textbf{55.10}, \textbf{27.60}) \\
                        & & 1.2 && (57.40, 49.60) & /              && (\textbf{59.40}, \textbf{51.20}) & (54.70, 26.30) \\
                        & & 1.5 && (56.90, 49.20) & /              && (59.40, 51.20) & (50.60, 25.30) \\
                        & & 1.7 && (56.20, 48.50) & /              && (59.10, 51.00) & (47.20, 23.60) \\
                        & & 2.0 && (54.40, 46.90) & /              && (58.00, 50.10) & (47.20, 23.60) \\
                        & & 2.2 && (52.80, 45.40) & /              && (57.10, 49.30) & (44.70, 22.40) \\
                        & & 2.5 && (49.80, 42.50) & /              && (55.40, 48.00) & (41.50, 20.70) \\
    \midrule
    & \textbf{VOC 2012} &  1.0  && (64.60, 57.10) & (64.20, 32.10) && (65.40, 57.80) & (67.10, 33.50) \\
                        & & 1.2 && (64.50, 57.50) & /              && (66.10, 58.30) & (68.80, 34.40) \\
                        & & 1.5 && (\textbf{65.30}, \textbf{57.70}) & /              && (66.90, 59.10) & (\textbf{69.30}, \textbf{34.60}) \\
                        & & 1.7 && (65.30, 57.70) & /              && (67.60, 59.80) & (69.00, 34.50) \\
                        & & 2.0 && (64.80, 57.10) & /              && (68.30, 60.50) & (68.00, 34.00) \\
                        & & 2.2 && (63.80, 56.00) & /              && (\textbf{68.40}, \textbf{60.60}) & (67.00, 33.50) \\
                        & & 2.5 && (61.30, 53.40) & /              && (67.90, 60.20) & (65.10, 32.50) \\
            \bottomrule
            \end{tabular}}
    \caption{The averaged Dice and IoU metrics and their standard errors (in parentheses) of the proposed \textit{RankDice} framework and the fixed-\textit{thresholding} framework based on temperature-scaling calibration methods with different temperature tuning parameters in \textbf{Fine-annotated CityScapes} and \textbf{PASCAL VOC 2012} datasets. `/' indicates that the performance of the thresholding-based framework over different temperatures are all the same under BCE loss. The performance is reported based on PSPNet + resnet50.}
    \label{tab:calibration}
    \end{table}

\section{Conclusions and future work}



\noindent\textbf{Summary.}
In this paper, we proposed a ranking-based framework for segmentation called \textit{RankSEG} that  comprises three steps: conditional probability estimation, ranking, and volume estimation.
Specifically, we have focused on the Dice metric and developed \textit{RankDice}, a version of \textit{RankSEG} for optimal Dice-segmentation. 
We introduced a key concept ``Dice-calibrated'' and demonstrated that \textit{RankDice} is able to recover the optimal segmentation rule, as opposed to the existing fixed-thresholding frameworks that are suboptimal with respect to the Dice metric.
Computationally, we have developed efficient exact/approximate numerical methods, including GPU-enabled algorithms, to carry out \textit{RankDice}.
Moreover, we established general theoretical results, including excess risk bounds and a rate of convergence for \textit{RankDice}, showing that \textit{RankDice} is consistent when the conditional probability estimation is well-calibrated.
Empirical experiments suggested that the proposed framework performs consistently well on a variety of segmentation benchmarks and state-of-the-art deep learning architectures. In parallel to \textit{RankDice}, we also developed the framework \textit{RankIoU} for the IoU metric. The theoretical results are similar, while the computation for the optimal IoU-segmentation could be more expensive in high-dimensional situation.

\noindent\textbf{Limitation and future work.} 
(i) For multiclass/multilabel segmentation, our results in this paper cover the overlapping (allowing) case; however, computing the optimal segmentation for the non-overlapping case is NP-hard. Thus, it would be interesting to develop a scalable approximating algorithm to utilize the proposed framework in the non-overlapping setting.
(ii) The conditional independence, $Y_i \perp Y_j | \mb{X}$ for any $i \neq j$, is crucial for Theorem \ref{thm:Dice_bayes} and subsequent theorems in Section \ref{sec:theory}. In some segmentation applications, it is of interest to extend the proposed frameworks and theorems with locally dependent outcomes. (iii) When given the testing features, the proposed method can be extended to maximize the F1-score in classification tasks.

\acks{We would like to acknowledge support for this project from the HK RGC-ECS 24302422 and the CUHK direct fund. Ben Dai developed the main framework (\textit{RankDice}), theory, algorithms and experiments. Chunlin Li mainly extended \textit{RankDice} to \textit{RankIoU} (Section \ref{sec:IoU})}, and partly contributed to the proof of Lemma \ref{lem:shrinkage} and Theorem \ref{thm:risk_bound}. We would like to thank the referees and the Action Editor for constructive feedback which greatly improved this work.

\newpage

\appendix

\section{Empirical evaluation of the Dice metric}
\label{sec:E-Dice}

Recall the definition of the Dice and IoU metrics in \eqref{eqn:dice_loss}, their empirical evaluation based on a validation/testing dataset $(\tilde{\mb{x}}_i, \tilde{\mb{y}}_i)_{i=1, \cdots, m}$, can be written as:
\begin{align}
    \widehat{\Dice}_\gamma(\pmb{\delta}) & = \frac{1}{m} \sum_{i=1}^{m} \frac{2 \tilde{\mb{y}}^\intercal_i \pmb{\delta}(\tilde{\mb{x}}_i)   + \gamma }{ \| \tilde{\mb{y}}_i \|_1 + \| \pmb{\delta}(\tilde{\mb{x}}_{i}) \|_1 + \gamma } = \frac{1}{m} \sum_{i=1}^{m} \frac{2 \text{TP}_i  + \gamma }{ 2 \text{TP}_i + \text{FP}_i + \text{FN}_i + \gamma }, \nonumber \\
    \widehat{\IoU}_\gamma(\pmb{\delta}) & = \frac{1}{m} \sum_{i=1}^{m} \Big( \frac{ \tilde{\mb{y}}_i^\intercal \pmb{\delta}(\tilde{\mb{x}}_{i}) + \gamma }{ \| \tilde{\mb{y}}_i \|_1 + \| \pmb{\delta}(\tilde{\mb{x}}_i) \|_1 - \tilde{\mb{y}}_i^\intercal \pmb{\delta}(\tilde{\mb{x}}_i) + \gamma } \Big) = \frac{1}{m} \sum_{i=1}^{m} \frac{\text{TP}_i + \gamma}{ \text{TP}_i + \text{FP}_i + \text{FN}_i + \gamma },
    \label{eqn:emp_dice}
\end{align}
where $\text{TP}_i$, $\text{FP}_i$ and $\text{FN}_i$ are defined at the instance level. In general, the empirical Dice and IoU metrics are not equal to the evaluation criteria used in some literature:
\begin{align}
    & \widehat{\Dice}_\gamma(\pmb{\delta}) \neq \overbar{\Dice}_\gamma(\pmb{\delta}) := \frac{ \frac{1}{m} \sum_{i=1}^{m} 2 \tilde{\mb{y}}^\intercal_i \pmb{\delta}(\tilde{\mb{x}}_i)   + \gamma }{ \frac{1}{m} \sum_{i=1}^{m} \| \tilde{\mb{y}}_i \|_1 + \frac{1}{m} \sum_{i=1}^{m} \| \pmb{\delta}(\tilde{\mb{x}}_{i}) \|_1 + \gamma } \overset{\mathbb{P}}{\longrightarrow} \frac{ \mathbb{E}\big( 2 \mb{Y}^\intercal \pmb{\delta}(\mb{X}) \big) + \gamma }{ \mathbb{E}\big( \| \mb{Y} \|_1 ) + \mathbb{E}\big( \| \pmb{\delta}(\mb{X}) \|_1 ) + \gamma }, \nonumber \\
    & \widehat{\IoU}_\gamma(\pmb{\delta}) \neq \overbar{\IoU}_\gamma(\pmb{\delta}) := \frac{ \frac{1}{m} \sum_{i=1}^{m} \tilde{\mb{y}}^\intercal_i \pmb{\delta}(\tilde{\mb{x}}_i)  + \gamma }{ \frac{1}{m} \sum_{i=1}^{m} \| \tilde{\mb{y}}_i \|_1 + \frac{1}{m} \sum_{i=1}^{m} \| \pmb{\delta}(\tilde{\mb{x}}_{i}) \|_1 - \frac{1}{m} \sum_{i=1}^m \tilde{\mb{y}}_i^\intercal \pmb{\delta}(\tilde{\mb{x}}_i) + \gamma } \nonumber \\
    & \hspace{5cm} \qquad \overset{\mathbb{P}}{\longrightarrow} \frac{ \mathbb{E}\big( 2 \mb{Y}^\intercal \pmb{\delta}(\mb{X}) \big) + \gamma }{ \mathbb{E}\big( \| \mb{Y} \|_1 ) + \mathbb{E}\big( \| \pmb{\delta}(\mb{X}) \|_1 ) - \mathbb{E}\big( \mb{Y}^\intercal \pmb{\delta}(\mb{X}) \big) + \gamma }.
    \label{eqn:mis_dice}
\end{align}
Here $\overset{\mathbb{P}}{\to}$ denotes convergence in probability following from the law of large numbers and Slutsky's theorem. Clearly, both empirical and population evaluations in \eqref{eqn:mis_dice} do not match with the empirical verisons in \eqref{eqn:emp_dice} and the population Dice in \eqref{eqn:dice_loss}. Although the empirical evaluation in \eqref{eqn:mis_dice} is widely used, it inherently discounts the effects of instances with small segmented features/pixels, leading to bias in the empirical evaluation. The issues of $\overbar{\Dice}_\gamma(\pmb{\delta})$ and $\overbar{\IoU}_\gamma(\pmb{\delta})$ are also indicated in some recent literature, including \cite{cordts2016cityscapes} and \cite{berman2018lovasz}. 

Therefore, it is highly recommended using the empirical Dice in \eqref{eqn:emp_dice} in implementation, and our numerical results in Section \ref{sec:num} are reported based on \eqref{eqn:emp_dice}.

\section{Auxiliary definitions}

\subsection{Population RankSEG}
\label{sec:pop_rankdice}
In this section, we present the definition of population \textit{RankSEG}, including the proposed frameworks \textit{RankDice} and \textit{RankIoU}. In other words, we work with population of $(\mb{X}, \mb{Y}) \in \mathbb{R}^d \times \{0,1\}^d$. Denote $\mathcal{Q}$ as the class of all measurable functions $\mb{q}: \mb{x} \in \mathbb{R}^d \to \mb{q}(\mb{x}) = (q_1(\mb{x}), \cdots, q_d(\mb{x}))^\intercal \in [0,1]^d$.

\noindent \textbf{Step 1 (Conditional probability estimation)}: Estimate the conditional probability based on a strictly proper loss $l(\cdot, \cdot)$:
\begin{equation}
\label{eqn:pop_Rankdice}
\widehat{\mb{q}} = \argmin_{\mb{q} \in \mathcal{Q}}  \mathbb{E} \big( l( \mb{Y}, \mb{q}(\mb{X}) ) \big).
\end{equation}

\noindent \textbf{Step 2 (Ranking)}: Given a new instance $\mb{x}$, sort its estimated conditional probabilities decreasingly, and denote the corresponding indices as $j_1, \cdots, j_d$, that is, $\widehat{{q}}_{j_1}(\mb{x}) \geq \widehat{{q}}_{j_2}(\mb{x}) \geq \cdots \geq \widehat{{q}}_{j_d}(\mb{x})$.

\noindent \textbf{Step 3 (Volume estimation)}: From \eqref{eqn:vol_est_true_pro}, we estimate the volume $\widehat{\tau}(\mb{x})$ by replacing the true conditional probability $\mb{p}(\mb{x})$ by the estimated one $\widehat{\mb{q}}(\mb{x})$:
\begin{align*}
\text{(RankDice)}\ \ \
\widehat{\tau}(\mb{x}) 
& = \argmax_{\tau \in \{0, \cdots, d\} } \sum_{s = 1}^\tau \sum_{l=0}^{d-1} \frac{2}{\tau + l + \gamma + 1} \widehat{{q}}_{j_s}(\mb{x}) \mathbb{P} \big( \widehat{\Gamma}_{\sm j_s}(\mb{x}) = l \big) + \sum_{l=0}^d \frac{\gamma}{\tau + l + \gamma} \mathbb{P}\big( \widehat{\Gamma}(\mb{x}) = l \big), \\
\text{(RankIoU)}\ \ \
\widehat{\tau}(\mb{x}) 
& = \argmax_{\tau \in \{0, 1, \cdots, d\} } \ \Big( \sum_{j \in J_\tau(\mb{x}) } \widehat{q}_j(\mb{x}) + \gamma \Big) \sum_{l=0}^{d-\tau} \frac{1}{\tau + l + \gamma} \mathbb{P}\big( \widehat{\Gamma}_{\sm J_{\tau}(\mb{x})}(\mb{x}) = l \big),
\end{align*}
where $\widehat{\Gamma}(\mb{x}) = \sum_{j=1}^d \widehat{B}_{j}(\mb{x})$, $\widehat{\Gamma}_{\sm j_s}(\mb{x}) = \sum_{j \neq j_s} \widehat{B}_{j}(\mb{x})$, and $\widehat{\Gamma}_{\sm J_{\tau}(\mb{x})}(\mb{x}) = \sum_{j \notin J_{\tau}(\mb{x})} \widehat{B}_{j}(\mb{x})$ denote Poisson-binomial random variables, and $\widehat{B}_j(\mb{x})$ is a Bernoulli random variable with the success probability $\widehat{q}_{j}(\mb{x})$; for $j=1,\cdots,d$.

\subsection{Poisson-binomial distribution}
\label{sec:PBD}
The Poisson binomial distribution is the discrete probability distribution of a sum of independent non-identical Bernoulli trials. Specifically, suppose $B_1, \cdots, B_d$ are independent Bernoulli random variables, with probabilities of success $\mb{p} = (p_1, \cdots, p_d)^\intercal$, then $\Gamma = \sum_{j=1}^d B_j$ is a Poisson-Binomial random variable with parameter $\mb{p}$, denoted as $\Gamma \sim \text{PB}(\mb{p})$, and its probability mass function is:

$$
\mathbb{P}\big ( \Gamma  = l \big)  = \sum_{\mb{b}: \|\mb{b}\|_1 = l} \prod_{j=1}^d \big( b_{j} p_{j} + (1 - b_{j}) \big(1 - p_{j} \big) \big),
$$
where $\mb{b} = (b_1, \cdots, b_d)^\intercal \in \{0,1\}^d$. Moreover, the mean, variance, and skewness for $\Gamma \sim \text{PB}(\mb{p})$ are listed as follows.
\begin{align*}
    \mu \coloneqq \mathbb{E}(\Gamma) = \sum_{j=1}^d p_j, \ \sigma^2 \coloneqq \Var(\Gamma) = \sum_{j=1}^d p_j (1 - p_j), \ \eta \coloneqq \text{Skew}(\Gamma) = \frac{1}{ \sigma^{3}} \sum_{j=1}^d p_j (1 - p_j) (1 - 2p_j).
\end{align*}

\subsection{Conditional independence in segmentation}
\label{sec:dep}
In this section, we adopt a probabilistic perspective on the likelihood of the segmentation task, to suggest that the conditional independence ($Y_{j}\perp Y_{j'} \mid \mb X$ for $j\neq j'$) is implicitly assumed to ensure the validity of the cross-entropy (CE) loss, and widely accepted due to the high dimensional nature of segmentation data.

Suppose $(\mb{X}_i, \mb{Y}_i) \overset{\text{iid}}{\sim} \mathbb{P}_{\mb{X}, \mb{Y}}$, the negative conditional log-likelihood function of $\mb{q}$ for the probabilistic model is:
        \begin{align*}
            \mathcal{L}_n(\mb{q}) &:= - \log \Big( \prod_{i=1}^{n} \mathbb P_{\mb{q}}(\mb Y=\mb{y}_i \mid \mb X = \mb{x}_i) \Big) = - \log \Big( \prod_{i=1}^{n} \prod_{j=1}^{d} q_j(\mb x_i)^{y_{ij}} (1 - q_j(\mb x_i))^{1-y_{ij}} \Big) \\
            &= -\sum_{i=1}^{n} \sum_{j=1}^{d} (y_{ij} \log(q_j(\mb x_i)) + (1 - y_{ij})\log(1 - q_j(\mb x_i)) ) = \sum_{i=1}^n l_{\text{CE}}( \mb{y}_i, \mb{q}(\mb{x}_i) ),
        \end{align*}
    where the second equality follows from the conditional independence assumption $Y_{j}\perp Y_{j'} \mid \mb X$ for $j\neq j'$, which connects the CE loss to the negative conditional log-likelihood function. In this sense, the conditional independence is naturally (and implicitly) assumed by CE to presuppose the probabilistic interpretation of its estimator.

    Moreover, the conditional independence is widely accepted due to the high dimensional nature of segmentation data. For example, given a 512x512 image, it is infeasible to consider $512^4$ pairs of label-dependence, which can even be adaptive with respect to $\mb{x}$. 

\section{Simulation results and Implementation details}
\subsection{Simulation setting and results}
This subsection includes the simulation setting demonstration (Figure \ref{fig:sim}) and the numerical results (Tables \ref{tab:sim} and \ref{tab:sim2}, and Figure \ref{fig:sim2}).

\begin{figure*}[t!]
    \centering
    \begin{subfigure}
        \centering
        \includegraphics[height=1.3in]{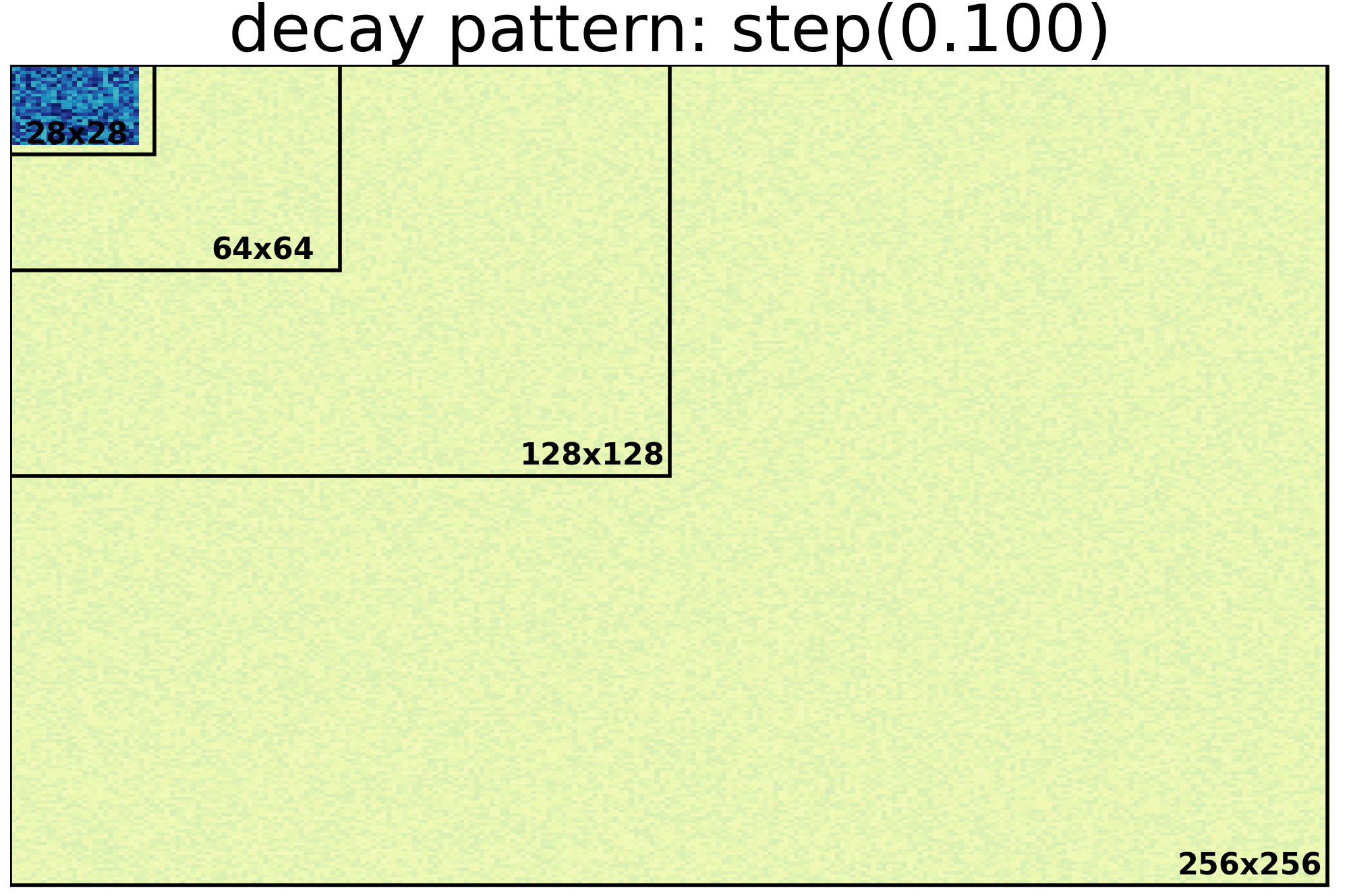}
    \end{subfigure}%
    \begin{subfigure}
        \centering
        \includegraphics[height=1.3in]{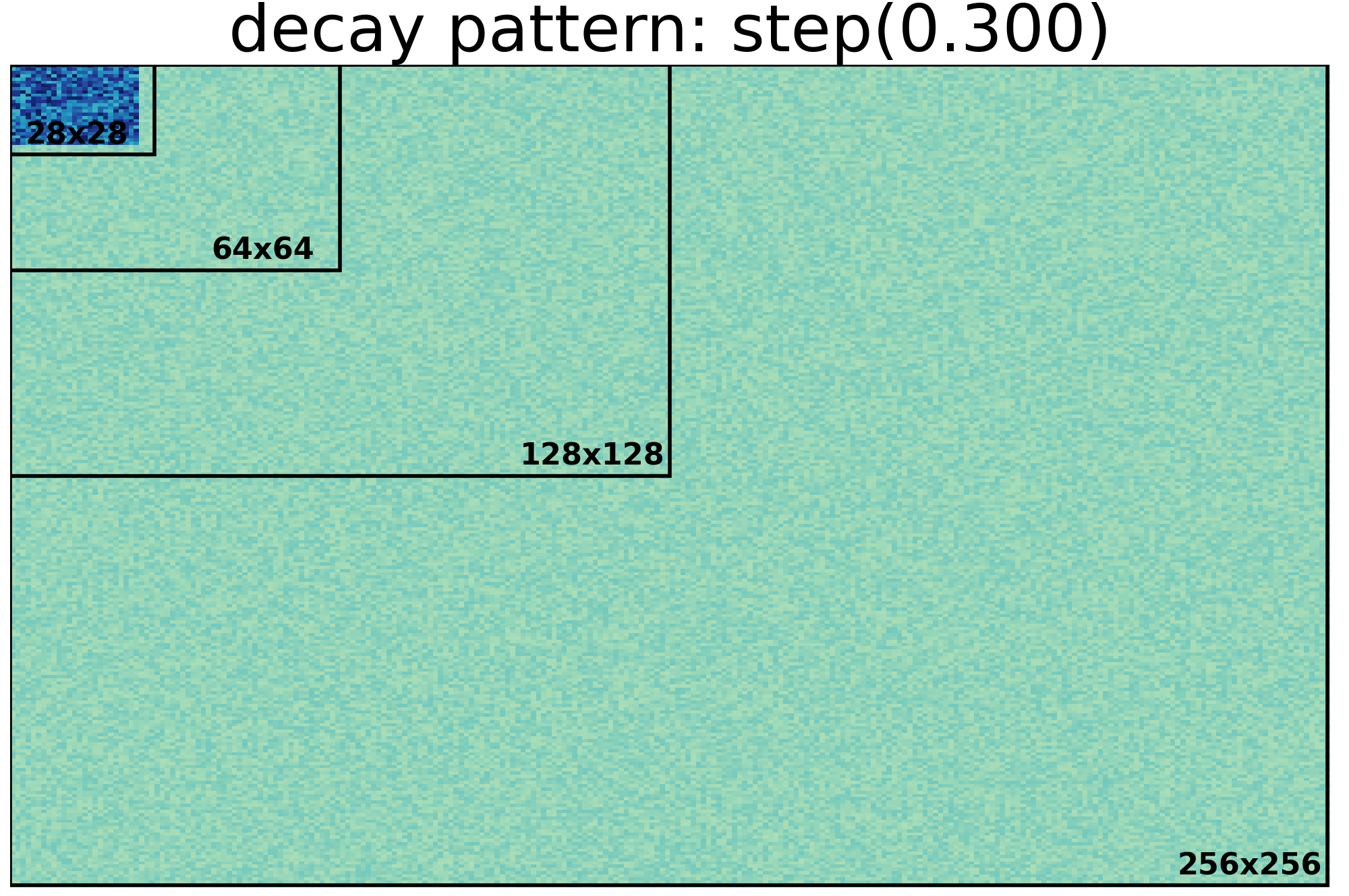}
    \end{subfigure}
    \begin{subfigure}
        \centering
        \includegraphics[height=1.3in]{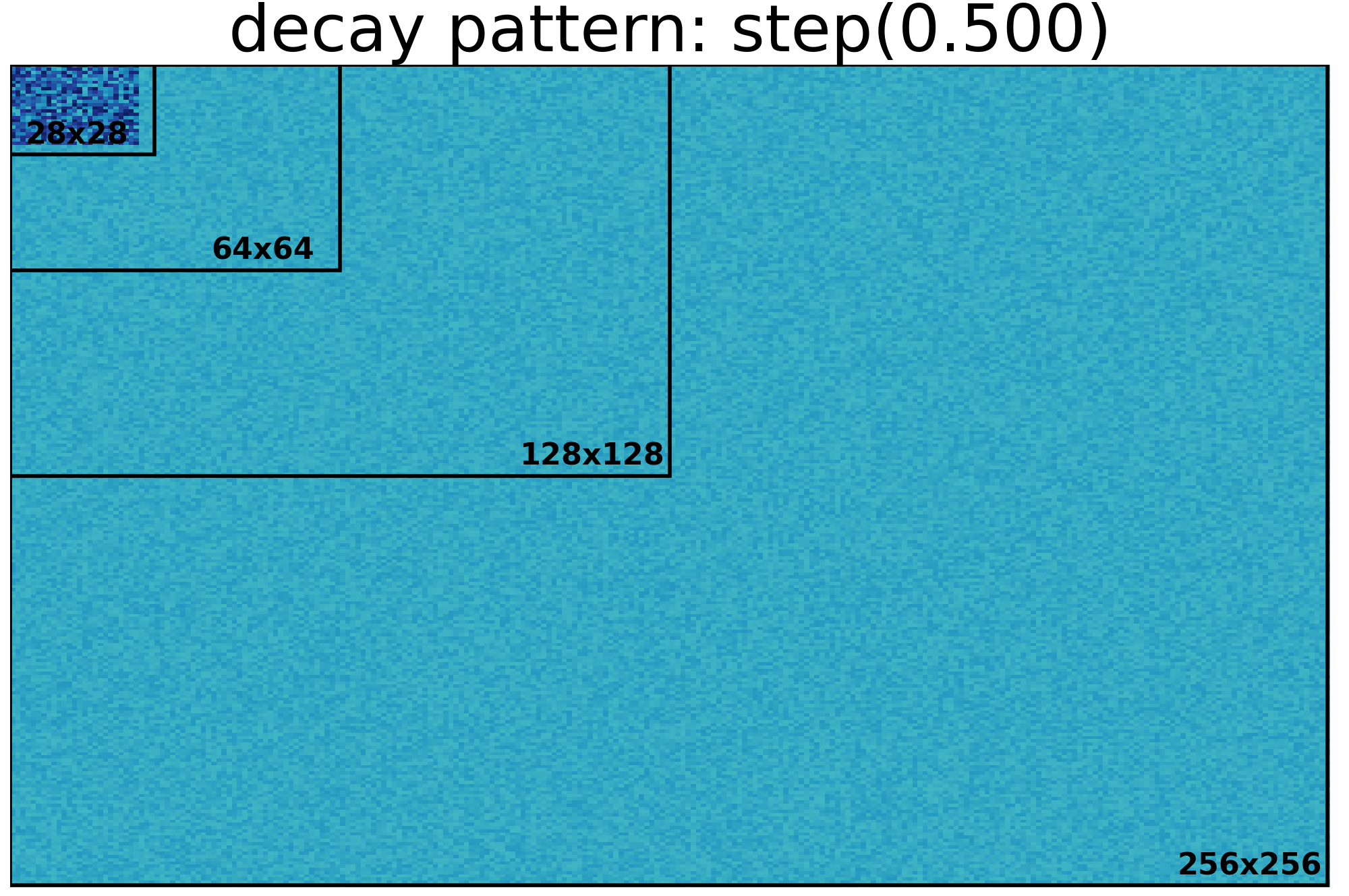}
    \end{subfigure}
    \begin{subfigure}
        \centering
        \includegraphics[height=1.3in]{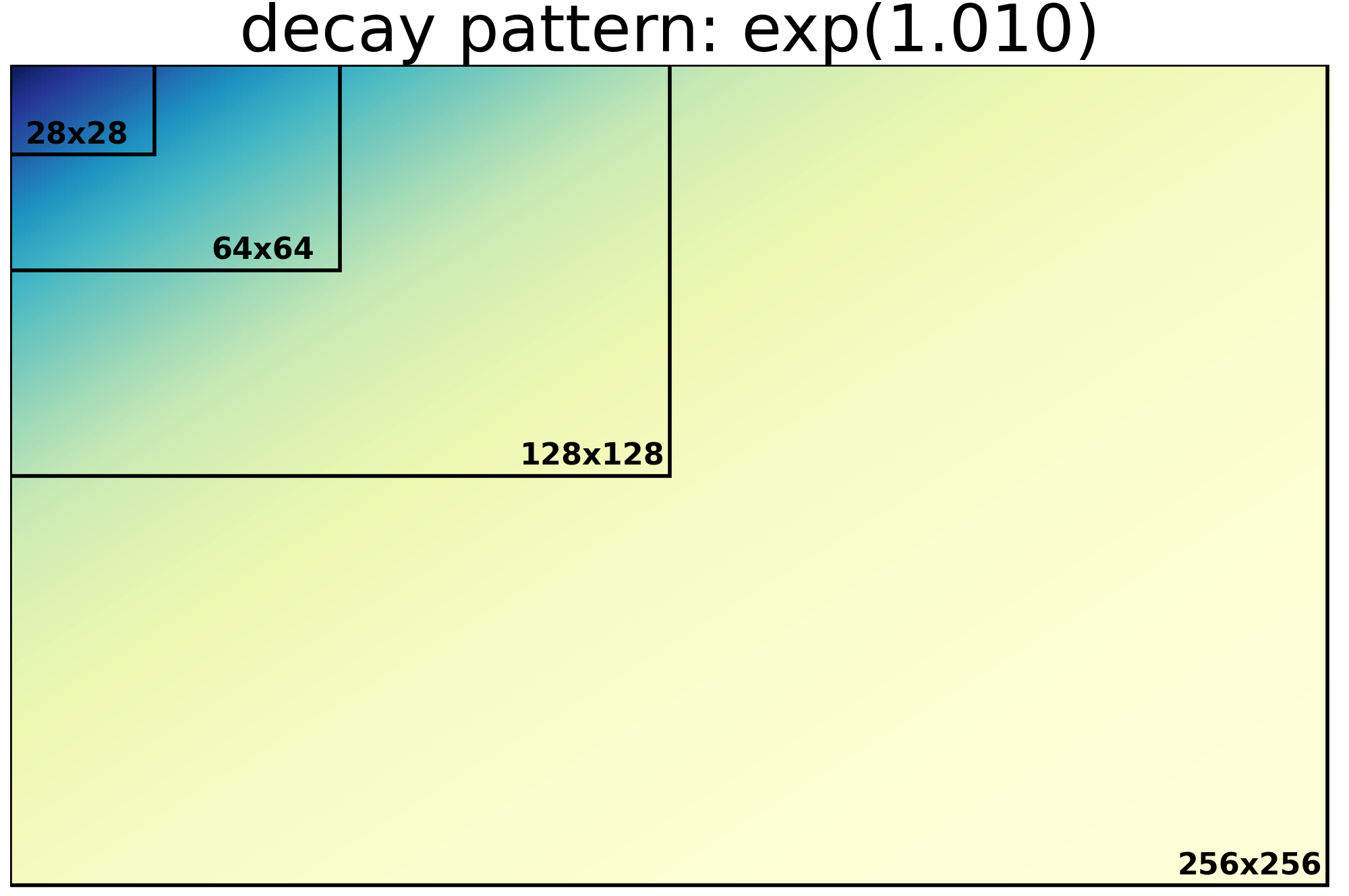}
    \end{subfigure}%
    \begin{subfigure}
        \centering
        \includegraphics[height=1.3in]{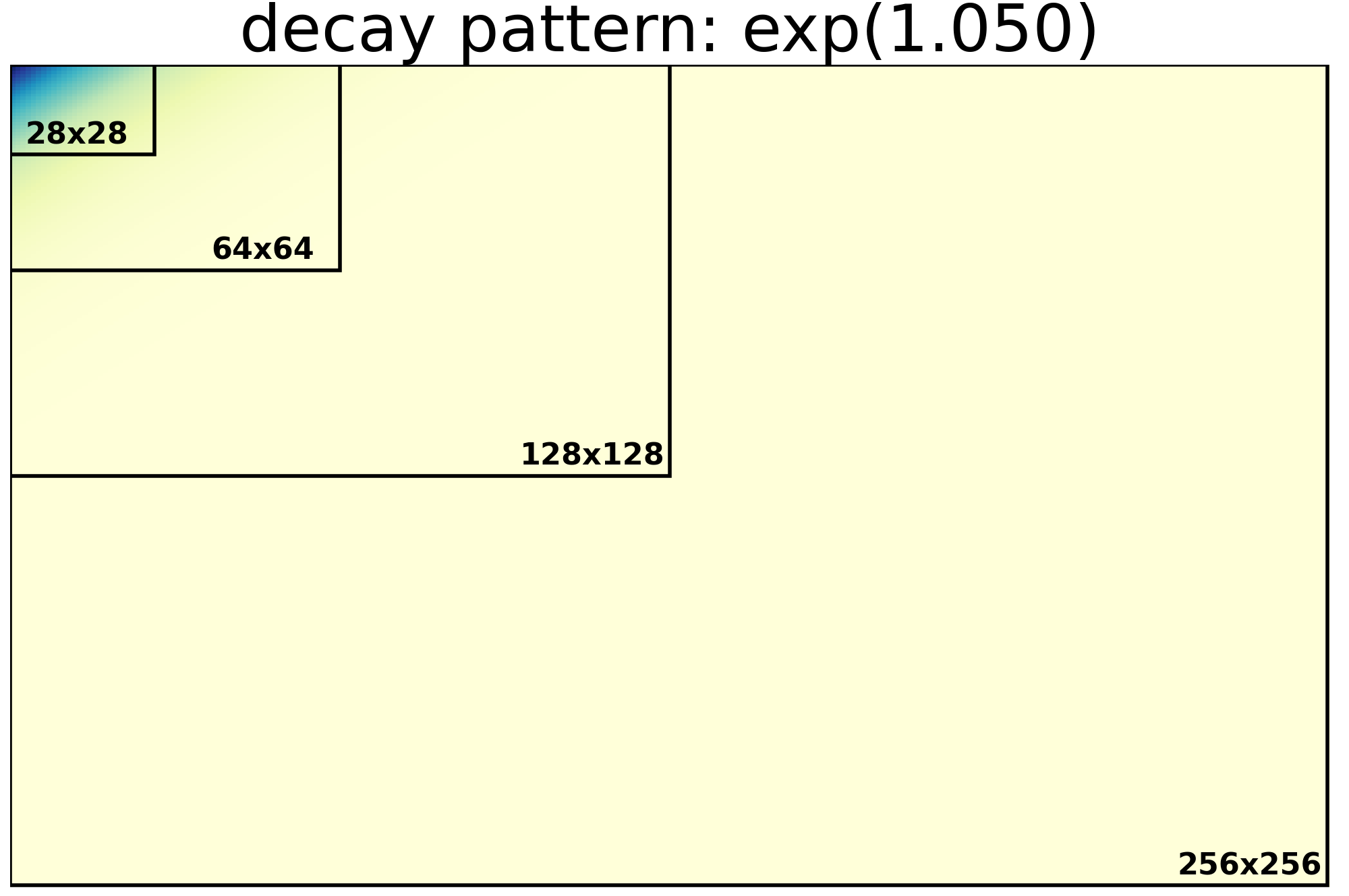}
    \end{subfigure}
    \begin{subfigure}
        \centering
        \includegraphics[height=1.3in]{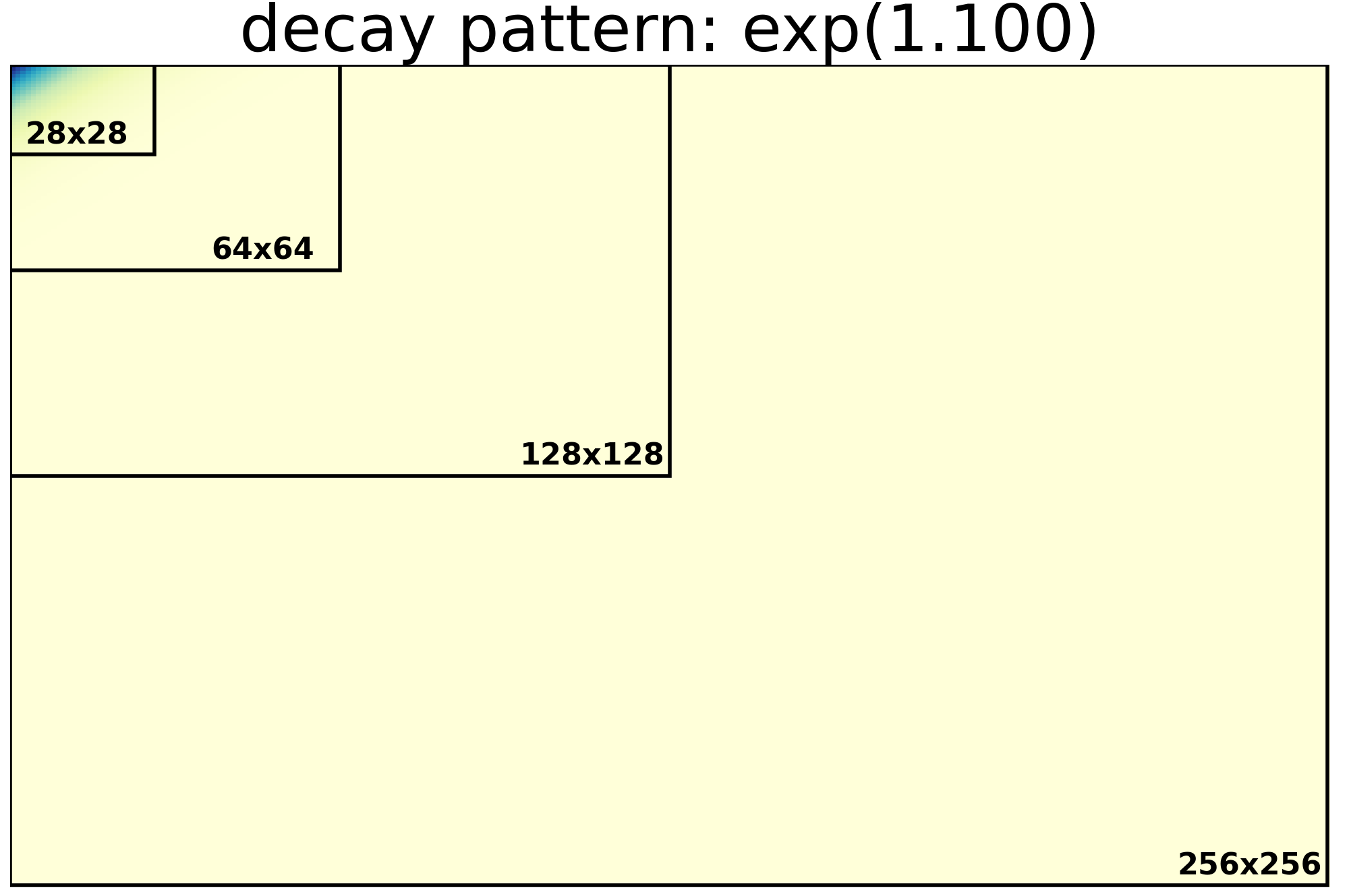}
    \end{subfigure}
    \begin{subfigure}
        \centering
        \includegraphics[height=1.3in]{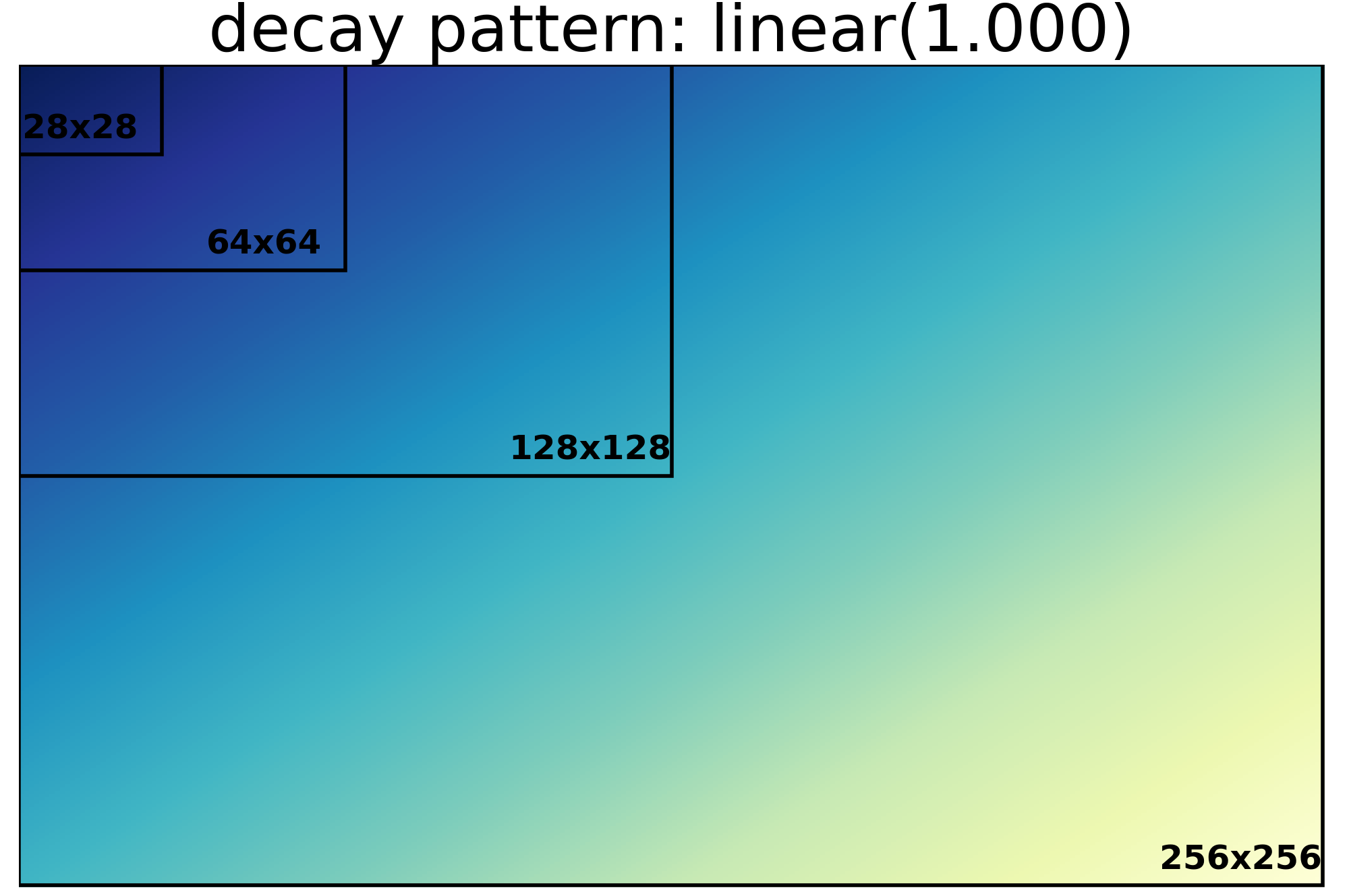}
    \end{subfigure}%
    \begin{subfigure}
        \centering
        \includegraphics[height=1.3in]{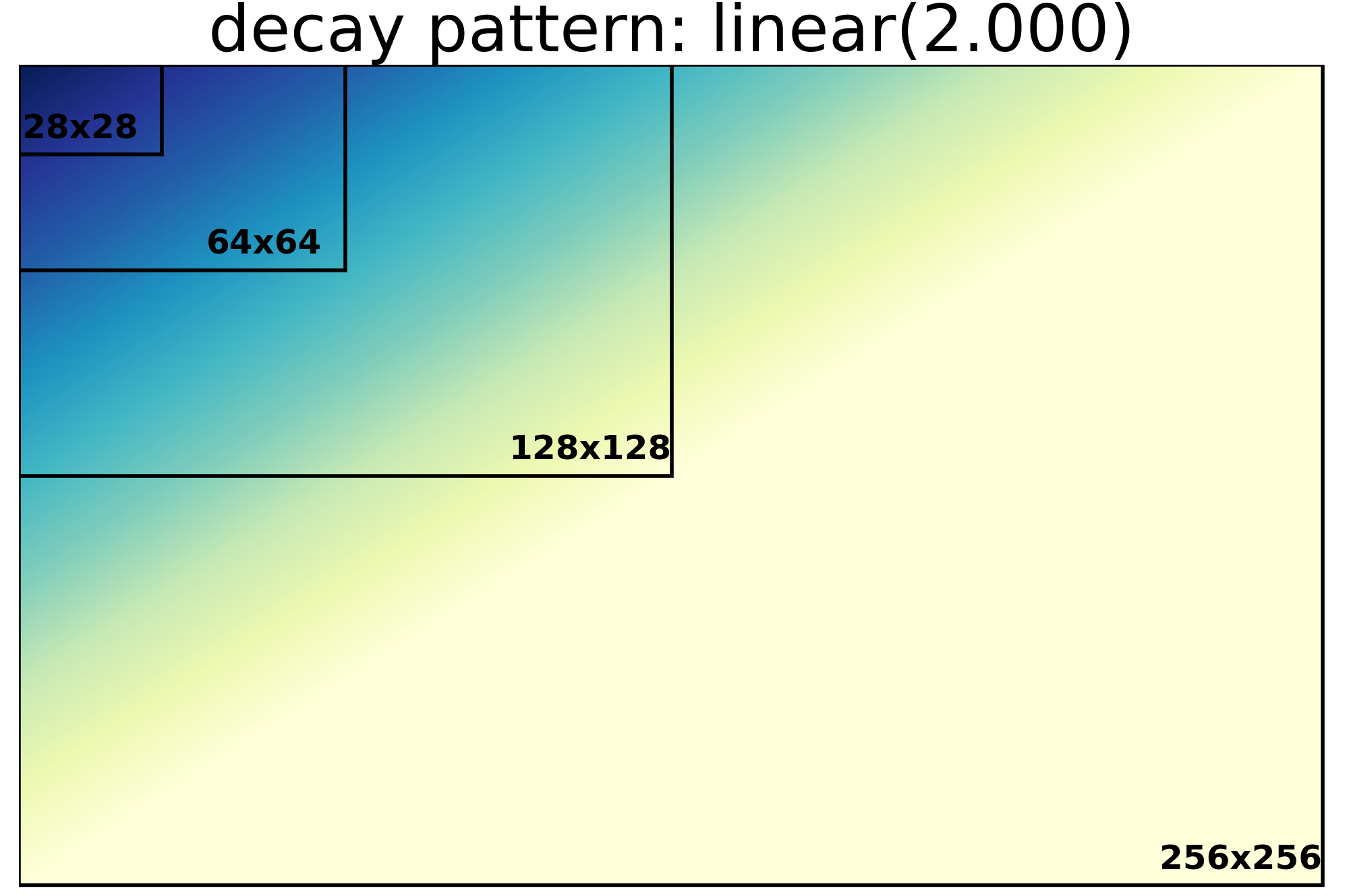}
    \end{subfigure}
    \begin{subfigure}
        \centering
        \includegraphics[height=1.3in]{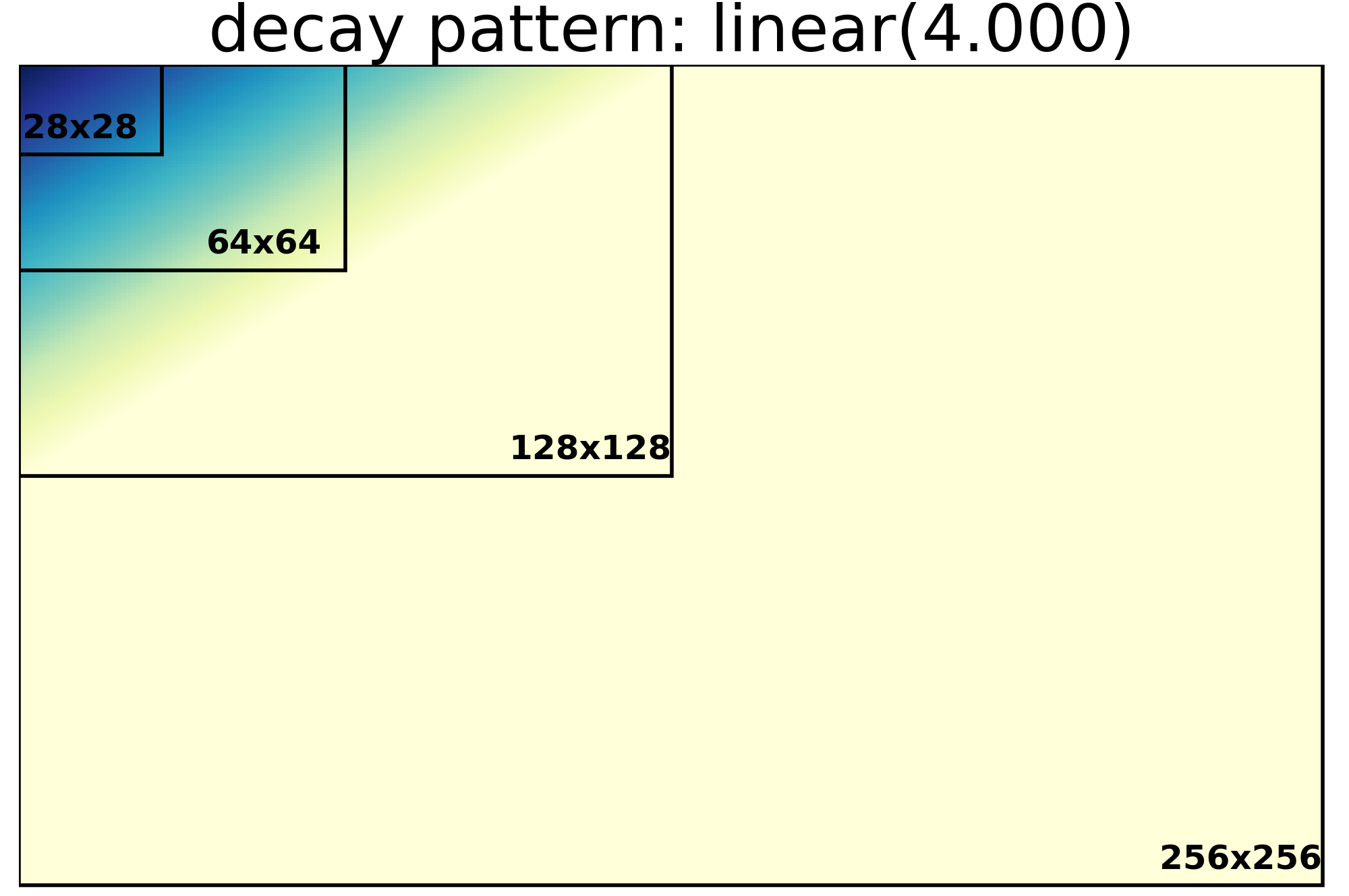}
    \end{subfigure}
    \caption{Simulation setting in Section \ref{sec:sim} with different decay patterns and dimensions/shapes (28x28 - 256x256). \textbf{Upper panel.} Heatmaps for the simulated probabilities with step decay ($\beta=0.1, 0.3, 0.5$). \textbf{Middle panel.} Heatmaps for the simulated probabilities with exponential decay (base=1.01, 1.05, 1.10). \textbf{Lower panel.} Heatmaps for the simulated probabilities with linear decay (slope=1, 2, 4). The performance for the proposed \textit{RankDice} and \textit{thresholding}-based frameworks is summarized in Table \ref{tab:sim}.}
    \label{fig:sim}
\end{figure*}

\begin{table}
    \centering
        \begin{minipage}[t]{0.49\textwidth}
            \centering
            \scalebox{.65}{
                \begin{tabular}[t]{@{}cccccccccccccccccc@{}} \toprule
                 Decay & \phantom{a} & Shape & \phantom{a} &  Threshold (at 0.5) & \phantom{a} & RankDice (our) \\
                 \midrule
                  step(0.1) && 28x28 && 0.049(.000) && 0.274(.001) \\
                  && 64x64 && 0.083(.000) && 0.279(.000) \\
                  && 128x128 && 0.081(.000) && 0.278(.000) \\
                  && 256x256 && 0.089(.000) && 0.279(.000) \\
                    \cmidrule{3-7}
                  step(0.3) && 28x28 && 0.022(.001) && 0.499(.001) \\ 
                  && 64x64 && 0.038(.000) && 0.517(.001) \\
                  && 128x128 && 0.036(.000) && 0.518(.000) \\
                  && 256x256 && 0.040(.000) && 0.518(.000) \\
                    \cmidrule{3-7}
                  step(0.5) && 28x28 && 0.708(.000) && 0.708(.000) \\
                  && 64x64 && 0.707(.000) && 0.707(.000) \\
                  && 128x128 && 0.708(.000) && 0.708(.000) \\
                  && 256x256 && 0.708(.000) && 0.708(.000) \\
                   \midrule
                 exp(1.01) && 28x28 && 0.870(.000) && 0.870(.000)  \\
                 && 64x64 && 0.669(.000) && 0.714(.000)  \\
                 && 128x128 && 0.410(.000) && 0.551(.000)  \\
                 && 256x256 && 0.286(.000) && 0.450(.000)  \\
                    \cmidrule{3-7}
                 exp(1.05) && 28x28 && 0.427(.001) && 0.551(.001)  \\
                 && 64x64 && 0.296(.001) && 0.446(.001)  \\
                 && 128x128 && 0.276(.001) && 0.427(.001)  \\
                 && 256x256 && 0.274(.001) && 0.427(.001)  \\
                   \cmidrule{3-7}
                 exp(1.10) && 28x28 && 0.332(.002) && 0.467(.002)  \\
                 && 64x64 && 0.301(.001) && 0.439(.002)  \\
                 && 128x128 && 0.300(.002) && 0.438(.002)  \\
                 && 256x256 && 0.298 (.002) && 0.436(.002)  \\
                \bottomrule
                \end{tabular}}
        \end{minipage}
        \hfill
        \begin{minipage}[t]{0.49\hsize}\centering
            \scalebox{.65}{
                \begin{tabular}[t]{@{}cccccccccccccccccc@{}} \toprule
                 Decay & \phantom{a} & Shape & \phantom{a} &  Threshold (at 0.5) & \phantom{a} & RankDice (our) \\
                 \midrule
                 linear(1.00) && 28x28 && 0.679(.001) && 0.717(.001)  \\
                 && 64x64 && 0.672(.000) && 0.711(.000)  \\
                 && 128x128 && 0.669(.000) && 0.709(.000)  \\
                 && 256x256 && 0.668(.000) && 0.707(.000)  \\
                \cmidrule{3-7}
                 linear(2.00) && 28x28 && 0.578(.001) && 0.647(.001)  \\
                 && 64x64 && 0.575(.001) && 0.642(.001)  \\
                 && 128x128 && 0.573(.000) && 0.638(.000)  \\
                 && 256x256 && 0.573(.000) && 0.637(.000)  \\
                \cmidrule{3-7}
                 linear(4.00) && 28x28 && 0.588(.003) && 0.663(.002)  \\
                 && 64x64 && 0.580(.001) && 0.646(.001)  \\
                 && 128x128 && 0.575(.001) && 0.642(.001)  \\
                 && 256x256 && 0.574(.000) && 0.639(.000)  \\
                \bottomrule
                \end{tabular}}
            \caption{The averaged Dice metrics and its standard errors (in parentheses) of the proposed \textit{RankDice} framework and the \textit{thresholding}-based (or \textit{argmax}-based) framework in Example 1 (see Fig \ref{fig:sim}) in Section \ref{sec:sim}.}
            \label{tab:sim}
        \end{minipage}
\end{table}

\begin{minipage}{\textwidth}
    \begin{minipage}[b]{0.49\textwidth}
        \scalebox{.85}{
        \begin{tabular}{@{}cccccccccccccccccc@{}} \toprule
        & Framework & \phantom{a} & Threshold & \phantom{a} &  Dice  \\
         \midrule
         & \textit{threshold}-based && 0.1 && 0.481(.005) \\
         &                             && 0.2 && 0.560(.005) \\
         &                             && 0.3 && 0.560(.005) \\
         &                             && 0.4 && 0.560(.005) \\
         &                             && 0.5 && 0.560(.005) \\
         &                             && 0.6 && 0.528(.005) \\
         &                             && 0.7 && 0.471(.005) \\
         &                             && 0.8 && 0.377(.005) \\
         &                             && 0.9 && 0.230(.004) \\
         \cmidrule{1-2}
         & \textit{RankDice}(our)      && --- && \textbf{0.601(0.005)} \\
        \bottomrule
        \end{tabular}}
        \captionof{table}{The averaged Dice metrics and its standard errors (in parentheses) of the proposed \textit{RankDice} framework and the \textit{thresholding}-based (with different thresholds) framework in Example 2 in Section \ref{sec:sim}.}
        \label{tab:sim2}
    \end{minipage}
    \hfill
    \begin{minipage}[b]{0.49\textwidth}
        \centering
        \includegraphics[scale=0.16]{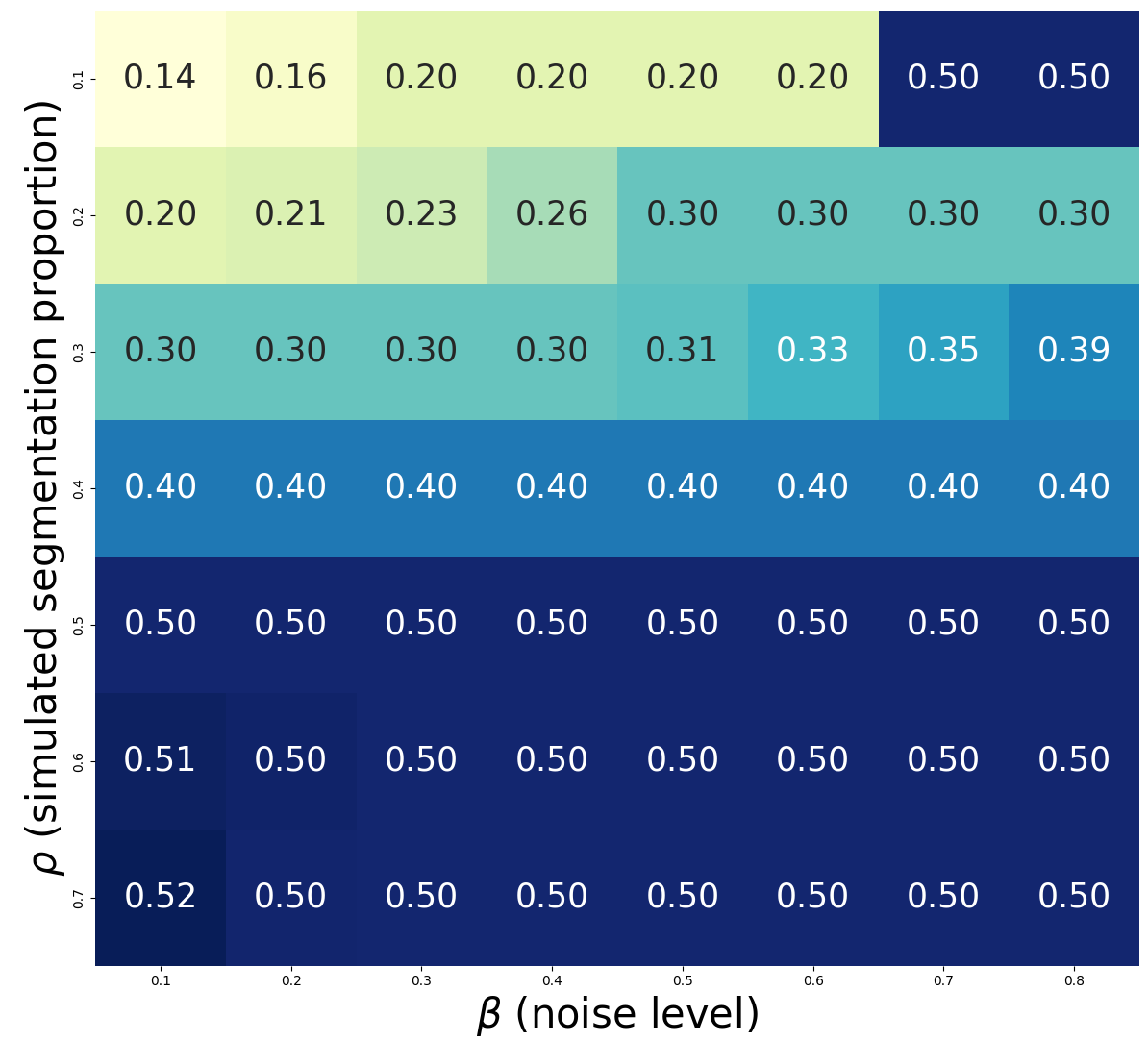}
        \captionof{figure}{The optimal thresholds for different images with various generating parameters $(\beta, \rho)$ in Example 2 in Section \ref{sec:sim}.}
        \label{fig:sim2}
    \end{minipage}
\end{minipage}

\subsection{Implementation details}

The experiment protocol of our numerical sections basically follows a well-developed Github repository \textsc{pytorch-segmentation} \citep{torch_segmentation}. The major difference lies in the empirical evaluation of the Dice and IoU metrics. In our experiments, we report the unbiased evaluations $\widehat{\mDice}_\gamma(\cdot)$ and $\widehat{\mIoU}_\gamma(\cdot)$, yet the biased evaluations $\overbar{\mDice}_\gamma(\cdot)$ and $\overbar{\mIoU}_\gamma(\cdot)$ are usually used for existing literature, see more discussion and definitions in Appendix \ref{sec:E-Dice}.

To justify the effectiveness of our experiment protocol, we also report $\overbar{\mDice}_\gamma(\cdot)$ and $\overbar{\mIoU}_\gamma(\cdot)$ under our setting based on the \textit{argmax}-based framework and compare the performance with the existing benchmarks. Specifically, in our setting, the performance is, \textbf{DeepLab}: (i) CityScapes: $\overbar{\mIoU}$ 63.20\%; $\overbar{\mDice}$ 76.10\%; (ii) VOC: $\overbar{\mIoU}$ 74.40\%; $\overbar{\mDice}$ 84.30\%; \textbf{PSPNet}: (i) CityScapes: $\overbar{\mIoU}$ 65.20\%; $\overbar{\mDice}$ 77.60\%; (ii) VOC \textit{test}: $\overbar{\mIoU}$ 79\% (which is provided by \citet{torch_segmentation} with the same configuration expect \texttt{batch\_size=16}); \textbf{FCN8}: VOC: $\overbar{\mIoU}$ 55.60\%; $\overbar{\mDice}$ 70.40\%.

The experiment protocol, including \texttt{learning\_rate}, \texttt{crop\_size}, \texttt{backbone}, and \texttt{batch\_size}, on the existing networks are summarized as follows.

\medbreak

\noindent \textbf{DeepLab \citep{chen2018encoder}.} The experiment on the Fine-annotated CityScapes dataset is set as follows: \texttt{backbone} is ``Xception-65''. The final $\overbar{\mIoU}$ is 79.14\%; 
The experiment on the PASCAL VOC 2012 dataset is set as follows: \texttt{learning\_rate} is 0.007 with \texttt{poly} schedule; \texttt{crop\_size} is 513x513, \texttt{backbone} is ``resnet101'', \texttt{batch\_size} is 16. The final $\overbar{\mIoU}$ is 78.21\%; 

\noindent \textbf{PSPNet \citep{zhao2017pyramid}.} The experiment on the Fine-annotated CityScapes dataset is set as follows: \texttt{learning\_rate} is 0.01. The final $\overbar{\mIoU}$ is 78.4\%; 
The experiment on the PASCAL VOC 2012 dataset is set as follows: \texttt{learning\_rate} is 0.01, \texttt{batch\_size} is 16. The final $\overbar{\mIoU}$ based on the VOC \textit{test} datset is 82.6\%; 

\noindent \textbf{FCN8 \citep{long2015fully}.}  
The experiment on the PASCAL VOC 2012 dataset is set as follows: \texttt{learning\_rate} is 0.0001, \texttt{batch\_size} is 20, \texttt{backbone} is ``VGG16''. The final $\overbar{\mIoU}$ is 62.2\%; 

Note that the suboptimal performance of our experiment may be caused by a small batch/crop size, different specified backbone models, and other fitting hyperparameters.  The current experiment can be further improved by carefully tuning the hyperparameters, yet it provides a fair numerical comparison of all frameworks (threshold-based, argmax-based, and the proposed \textit{RankDice}).

\section{Technical proofs}

\subsection{Proof of Theorem \ref{thm:Dice_bayes}}
\begin{proof}
It suffices to consider the point-wise maximization:
$$
 \pmb{\delta}^*(\mb{x}) = \argmax_{\mb{v} \in \{0,1\}^d} \ \Dice_\gamma(\mb{v}|\mb{x}), \quad \Dice_\gamma(\mb{v}|\mb{x}) = \mathbb{E} \Big( \frac{2 \mb{Y}^\intercal \mb{v} + \gamma }{ \| \mb{Y} \|_1 + \| \mb{v} \|_1 + \gamma } \Big| \mb{X} = \mb{x} \Big).
$$
Let $\mb{y}_{\sm j} = (y_1, \cdots, y_{j-1}, y_{j+1}, \cdots, y_d)^\intercal$, $I(\mb{v}) = I(\pmb{\delta}(\mb{x})) = \{j: v_{j} = 1\}$ be the index set of segmented features by $\pmb{\delta}(\mb{x})$, and $\| \mb{v} \|_1 = \tau$, we have
$$
\Dice_\gamma(\mb{v}|\mb{x}) = \mathbb{E} \Big( \frac{2 \mb{Y}^\intercal \mb{v} }{ \| \mb{Y} \|_1 + \tau + \gamma } \Big| \mb{X} = \mb{x} \Big) + \mathbb{E} \Big( \frac{ \gamma }{ \| \mb{Y} \|_1 + \tau + \gamma } \Big| \mb{X} = \mb{x} \Big).
$$
Note that the second term is only related to $\tau$, and the first term can be rewritten as:
\begin{align}
\label{eqn:additive}
\mathbb{E} \Big( \frac{2 \mb{Y}^\intercal \mb{v} }{ \| \mb{Y} \|_1 + \tau + \gamma } & \Big| \mb{X} = \mb{x} \Big) = \sum_{\mb{y} \in \{0,1\}^d} \frac{2 \mb{y}^\intercal \mb{v} \mathbb{P}(\mb{Y} = \mb{y} | \mb{x}) }{  \tau + \| \mb{y} \|_1 + \gamma } = \sum_{\mb{y} \in \{0,1\}^d} \sum_{j=1}^d \frac{2 y_j v_j \mathbb{P}(\mb{Y} = \mb{y} | \mb{x}) }{ \tau + \| \mb{y} \|_1 + \gamma } \nonumber \\
& = \sum_{j \in I(\mb{v}) } \sum_{\mb{y} \in \{0,1\}^d}  \frac{2 y_j \mathbb{P}(\mb{Y} = \mb{y} | \mb{x}) }{ \tau + \| \mb{y} \|_1 + \gamma } = \sum_{j \in I(\mb{v}) } \sum_{\substack{\mb{y}_{\sm j} \in \{0,1\}^{d-1} \\  y_j = 1}}  \frac{2 \mathbb{P}(\mb{Y} = \mb{y} | \mb{x}) }{  \tau + \| \mb{y} \|_1 + \gamma } = \sum_{j \in I(\mb{v}) } s_j(\tau).
\end{align}
As indicated in \eqref{eqn:additive}, when $\tau$ is given, $\Dice_\gamma(\mb{v}|\mb{x})$ is an additive function with respect to $j \in I(\mb{v})$. Therefore, maximizing $\Dice_\gamma(\mb{v}|\mb{x})$ suffices to find the indices of top $\tau$ largest $s_j(\tau)$. Toward this end, we consider the differenced score function:
\begin{align}
\label{pf:diff}
D_{jj'}(\tau) & = s_j(\tau) - s_{j'}(\tau) = \sum_{\substack{\mb{y}_{\sm  j} \in \{0,1\}^{d-1} \\  y_j = 1}}  \frac{2 \mathbb{P}(\mb{Y} = \mb{y} | \mb{x}) }{  \tau + \| \mb{y} \|_1 + \gamma } - \sum_{\substack{\mb{y}_{\sm  j'} \in \{0,1\}^{d-1} \\  y_{j'} = 1}}  \frac{2 \mathbb{P}(\mb{Y} = \mb{y} | \mb{x}) }{  \tau + \| \mb{y} \|_1 + \gamma } \nonumber \\
& = \sum_{\substack{\mb{y}_{\sm  jj'} \in \{0,1\}^{d-2} \\  y_j = 1; y_{j'} = 0}}  \frac{2 \mathbb{P}(\mb{Y} = \mb{y} | \mb{x}) }{  \tau + \| \mb{y} \|_1 + \gamma } - \sum_{\substack{\mb{y}_{\sm jj'} \in \{0,1\}^{d-2} \\ y_{j} = 0, y_{j'} = 1}}  \frac{2 \mathbb{P}(\mb{Y} = \mb{y} | \mb{x}) }{  \tau + \| \mb{y} \| + \gamma } \nonumber \\
& = \sum_{\substack{\mb{y}_{\sm jj'} \in \{0,1\}^{d-2} \\  y_j = 1; y_{j'} = 0}}  \frac{2 \prod_{i \neq \{ j, j' \} } \mathbb{P}(Y_i = y_i | \mb{x}) \mathbb{P}(Y_j = 1 | \mb{x}) \mathbb{P}(Y_{j'} = 0 | \mb{x}) }{  \tau + 1 + \| \mb{y}_{\sm jj'} \|_1 + \gamma } \nonumber \\
& \hspace{3cm} - \sum_{\substack{\mb{y}_{\sm jj'} \in \{0,1\}^{d-2} \\  y_j = 0; y_{j'} = 1}}  \frac{2 \prod_{i \neq \{ j, j'\} } \mathbb{P}(Y_i = y_i | \mb{x}) \mathbb{P}(Y_j = 0 | \mb{x}) \mathbb{P}(Y_{j'} = 1 | \mb{x}) }{  \tau + 1 + \| \mb{y}_{\sm jj'} \|_1 + \gamma } \nonumber \\
& = \big( \mathbb{P}(Y_j = 1 | \mb{x}) - \mathbb{P}(Y_{j'} = 1 | \mb{x}) \big) \sum_{\substack{\mb{y}_{\sm jj'} \in \{0,1\}^{d-2} }}  \frac{2 \prod_{i \neq j, j'} \mathbb{P}(Y_i = y_i | \mb{x}) }{  \tau + 1 + \| \mb{y}_{\sm  jj'} \|_1 + \gamma },
\end{align}
where $\mb{y}_{\sm  jj'}$ is $\mb{y}$ removing $y_j$ and $y_{j'}$, the second last equality follows from the conditional independence of $Y_j$ and $Y_{j'}$ given $\mb{X}$ for any pair $j$ and $j'$. According to \eqref{pf:diff}, we have 
$$
D_{jj'}(\tau) \geq 0 \quad \iff \quad \mathbb{P}(Y_j = 1 | \mb{x}) - \mathbb{P}(Y_{j'} = 1 | \mb{x}) \geq 0,
$$
thus sorting $s_j(\tau)$ is equivalent to sorting $\mathbb{P}(Y_j = 1 | \mb{x})$ for any given $\tau$. Let $J_{\tau} = \big \{ j : \sum_{j'=1}^d \mb{1} \big( \mathbb{P}(Y_{j'} = 1 | \mb{x}) \geq \mathbb{P}(Y_{j} = 1 | \mb{x}) \big) \leq \tau \big\}$ be the index set of the $\tau$-largest conditional probabilities, it suffices to solve
\begin{align*}
\tau^* & = \argmax_{\tau = 0, \cdots, d} \sum_{j \in J_{\tau}} \mathbb{E} \big( \frac{2Y_j}{ \|\mb{Y}\|_1 + \tau + \gamma} \big) + \mathbb{E}\big( \frac{\gamma}{ \|\mb{Y}\|_1 + \tau + \gamma} \big) \\ 
& = \argmax_{\tau = 0, \cdots, d} \sum_{j \in J_{\tau}} \sum_{l=0}^{d-1} \frac{2 p_j(\mb{x}) }{\tau + l + \gamma + 1} \mathbb{P} \big( \| \mb{Y}_{\sm j} \|_1 = l | \mb{X} = \mb{x} \big) + \sum_{l=0}^d \frac{\gamma}{\tau + l + \gamma}\mathbb{P} \big( \| \mb{Y} \|_1 = l | \mb{X} = \mb{x} \big),
\end{align*}
where $\| \mb{Y}_{\sm j} \|_1 = \sum_{j' \neq j} Y_{j'}$ is a Poisson-binomial random variable with success probabilities $\mb{p}_{\sm j}(\mb{x})$, since $Y_j$ ($j = 1, \cdots, d$) is an independent Bernoulli random variable given $\mb{X} = \mb{x}$. The desirable result then follows.
\end{proof}

\subsection{Proof of Lemma \ref{lem:shrinkage}}
\begin{proof}
To proceed, we denote $\xi_l(\mb{x}) = \mathbb{P}( \widehat{\Gamma}(\mb{x}) = l )$, and $\xi_{jl}(\mb{x}) = \mathbb{P}( \widehat{\Gamma}_{ \sm j}(\mb{x}) = l )$, and $\overbar{\pi}_\tau = \overbar{\omega}_\tau(\mb{x}) + \overbar{\nu}_\tau(\mb{x})$. For simplicity in presentation, we assume that $\widehat{q}_1(\mb{x}) \geq \cdots \geq \widehat{q}_d(\mb{x})$. Then, for any $\tau' > \tau$, since $\overbar{\nu}_\tau \geq \overbar{\nu}_{\tau'}$, and 
\begin{align}
    \frac{\overbar{\pi}_\tau(\mb{x}) - \overbar{\pi}_{\tau'}(\mb{x})}{ 2(\tau' - \tau) } & \geq \sum_{j=1}^\tau \widehat{q}_j(\mb{x}) \sum_{l=0}^{d-1} \frac{ \xi_{jl}(\mb{x}) }{(\tau + l + 1 + \gamma)(\tau' + l + 1 + \gamma)}  - \frac{1}{\tau' - \tau} \sum_{j=\tau+1}^{\tau'} \widehat{q}_j(\mb{x}) \sum_{ l=0 }^{d-1} \frac{\xi_{jl}(\mb{x})}{ \tau' + l + \gamma + 1}   \nonumber \\
    & \geq \sum_{j=1}^\tau \widehat{q}_j(\mb{x}) \sum_{l=0}^{d-1} \frac{ \xi_{\tau l}(\mb{x}) }{(\tau + l + \gamma + 1)(\tau' + l + \gamma + 1)}  - \widehat{q}_{\tau + 1}(\mb{x}) \sum_{ l=0 }^{d-1} \frac{\xi_{\tau l}(\mb{x})}{ \tau' + l + \gamma + 1} \nonumber \\
    & \geq \widehat{q}_{\tau + 1}(\mb{x}) \sum_{l=0}^{d-1} \frac{ (\tau + \gamma + d) \xi_{\tau l}(\mb{x}) }{(\tau + l + \gamma + 1)(\tau' + l + \gamma + 1)}  - \widehat{q}_{\tau + 1}(\mb{x}) \sum_{ l=0 }^{d-1} \frac{\xi_{\tau l}(\mb{x})}{ \tau' + l + \gamma + 1} \geq 0, \nonumber \\
\end{align}
where the second inequality follows from Lemma \ref{lem:de_score} with $\zeta_l = (\tau + l + \gamma + 1)(\tau' + l + \gamma + 1)$ and $\zeta_l = \tau' + l + \gamma + 1$, and $\widehat{q}_1(\mb{x}) \geq \cdots \geq \widehat{q}_\tau(\mb{x}) \geq \cdots \geq \widehat{q}_{\tau'}(\mb{x})$, and the third inequality follows from the condition that $\sum_{j = 1}^\tau \widehat{q}_{j}(\mb{x}) / \widehat{q}_{\tau+1}(\mb{x}) \geq \tau + \gamma + d$. 
The desirable result then follows. The upper bound provided by Lemma \ref{lem:shrinkage} is illustrated in Figure \ref{fig:shrinkage} based on a random example of Example 1.
\begin{figure}[h]
    \centering
    \includegraphics[scale=0.26]{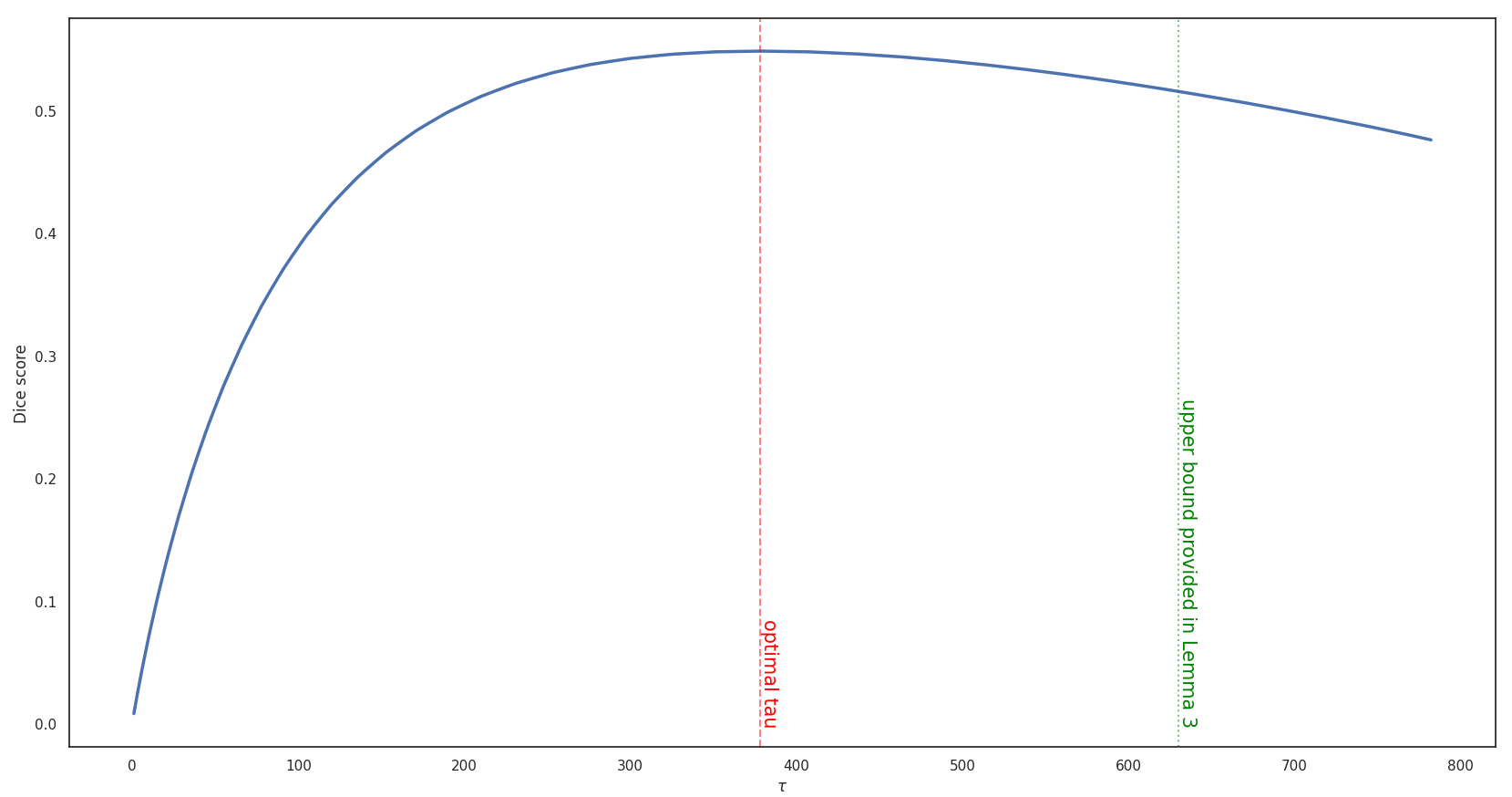}
    \caption{Dice score vs. $\tau$ based on a random example of Example 1. Note that Lemma 3 indicates the optimal $\tau$ (red line) is always obtained before the upper bound (green line). Thus, the searching region of $\tau$ can be shrunk.}
    \label{fig:shrinkage}
\end{figure}
\end{proof}

\subsection{Proof of Lemma \ref{lem:app}}
\begin{proof}
Without loss of generality, assume $\mathcal{L}(\epsilon) \subset \{ 0, \cdots, d-1\}$. Denote $l_L = \floor{\widehat{\sigma}(\mb{x}) \Psi^{-1}(\epsilon) + \widehat{\mu}(\mb{x})} - 1$ and $l_U = \ceil{ - \widehat{\sigma}(\mb{x}) \Psi^{-1}(\epsilon) + \widehat{\mu}(\mb{x})}$, $\xi_l = \mathbb{P}\big( \widehat{\Gamma} (\mb{x}) = l \big)$, $\widetilde{\xi}_l = \widetilde{\mathbb{P}}\big( \widehat{\Gamma} (\mb{x}) = l \big)$, $\xi_{\sm j,l} = \mathbb{P}\big( \widehat{\Gamma}_{\sm j} (\mb{x}) = l \big)$, $\widetilde{\xi}_{\sm j,l} = \widetilde{\mathbb{P}}\big( \widehat{\Gamma}_{ \sm j} (\mb{x}) = l \big)$, we treat $\big| \widetilde{\omega}_\tau - \overbar{\omega}_\tau \big|$ and $\big| \widetilde{\nu}_\tau - \overbar{\nu}_\tau \big|$ separately. First,
\begin{align*}
    \big| \widetilde{\omega}_\tau - \overbar{\omega}_\tau \big| & \leq \sum_{ l > l_U  } \frac{2}{ \tau + l + \gamma + 1} \omega_{\tau, l} + \sum_{0 \leq l < l_L  } \frac{2}{ \tau + l + \gamma + 1} \omega_{\tau, l} +  \sum_{ l_L \leq l < l_U  } \frac{2}{ \tau + l + \gamma + 1} \Big| \omega_{\tau, l} - \widetilde{\omega}_{\tau, l} \Big| \\
    & =: S_1 + S_2 + S_3.
\end{align*}
Next, we turn to bound $S_1$ - $S_3$ separately. 
\begin{align*}
    S_1 & = \sum_{ l > l_U } \frac{2}{ \tau + l + \gamma + 1} \sum_{s=1}^\tau \widehat{q}_{j_s}(\mb{x}) \mathbb{P}\big( \widehat{\Gamma}_{ \sm j_s} (\mb{x}) = l \big) \leq \frac{2}{ \tau + l_U + \gamma + 1}  \sum_{s=1}^\tau \sum_{l > l_U} \mathbb{P}\big( \widehat{\Gamma}_{ \sm j_s} (\mb{x}) = l \big) \\
    & \leq \frac{2}{ \tau + l_U + \gamma + 1}  \sum_{s=1}^\tau  \big( 1 - \mathbb{P}\big( \widehat{\Gamma}_{ \sm j_s} (\mb{x}) \leq l_U \big) \big) \leq \frac{2 \tau}{ \tau + l_U + \gamma + 1} \mathbb{P}\big( \widehat{\Gamma} (\mb{x}) > l_U \big),
\end{align*}
where the last inequality follows from Lemma \ref{lem:pb_sandwich}. For $S_2$, we have
\begin{align*}
    S_2 & \leq \frac{2}{ \tau + \gamma + 1} \sum_{s=1}^\tau \sum_{0 \leq l < l_L } \mathbb{P}\big( \widehat{\Gamma}_{ \sm j_s} (\mb{x}) = l \big) \leq \frac{2 \tau}{ \tau + \gamma + 1} \mathbb{P}\big( \widehat{\Gamma} (\mb{x}) \leq l_L + 1 \big),
\end{align*}
where the last inequality follows from Lemma \ref{lem:pb_sandwich}. Next, according to Theorem 1.1 in \citet{neammanee2005refinement},
$$
\mathbb{P}\big( \widehat{\Gamma} (\mb{x}) \leq l_L + 1 \big) \leq \mathbb{P}\big( Z \leq \Psi^{-1}(\epsilon) \big) + \frac{C_0}{\widehat{\sigma}^2(\mb{x})} = \epsilon + \frac{C_0}{\widehat{\sigma}^2(\mb{x})},
$$
and 
$$
\mathbb{P}\big( \widehat{\Gamma} (\mb{x}) > l_U \big) \leq   \mathbb{P}\big( Z \geq \Psi^{-1}(1-\epsilon) \big) + \frac{C_0}{\widehat{\sigma}^2(\mb{x})} = \epsilon + \frac{C_0}{\widehat{\sigma}^2(\mb{x})},
$$
where $Z$ is a random variable following the refined normal distribution. For $S_3$,
\begin{align*}
    S_3 & \leq \sum_{l=0}^{d-1} \frac{2}{\tau + l + \gamma + 1} \sum_{s=1}^\tau \widehat{\mb q}_{j_s}(\mb{x}) | \widetilde{\xi}_{\sm j, l} - \xi_{\sm j, l} | \leq \sum_{s=1}^\tau \widehat{\mb q}_{j_s}(\mb{x}) \sum_{l=0}^{d-1} \frac{2}{\tau + l + \gamma + 1} \max_{j = 1, \cdots, \tau} | \widetilde{\xi}_{\sm j, l} - \xi_{\sm j, l} | \\
    & \leq \min( \widehat{\mu}(\mb{x}), \tau ) \sum_{l=0}^{d-1} \frac{2}{\tau + l + \gamma + 1} \max_{s = 1, \cdots, \tau} \frac{C_0}{ \widehat{\sigma}^2(\mb{x}) - \widehat{q}_{j_s}(\mb{x})\big(1 - \widehat{q}_{j_s}(\mb{x}) \big) } \\
    & \leq  \frac{2C_0 \min( \widehat{\mu}(\mb{x}), \tau )}{ \widehat{\sigma}^2(\mb{x}) - 1/4 } \sum_{l=1}^{d} \frac{1}{\tau + l + \gamma} = \frac{2C_0 \min( \widehat{\mu}(\mb{x}), \tau )}{ \widehat{\sigma}^2(\mb{x}) - 1/4 } \big( H_{\tau+d+\gamma} - H_{\tau + \gamma} \big) \\
    & \leq \frac{2C_0 \min( \widehat{\mu}(\mb{x}), \tau )}{ \widehat{\sigma}^2(\mb{x}) - 1/4 } \Big( \log\big( 1 + \frac{d}{\tau + \gamma} \big) + 1 - \frac{1}{\tau + \gamma} \Big),
\end{align*}
where $H_K = \sum_{k=1}^K 1/k$ is the harmonic number, and the last inequality follows from the fact that $ \log(K) + 1 / K \leq H_K \leq \log(K) + 1$. Taken together, 
$$
\big| \widetilde{\omega}_\tau - \overbar{\omega}_\tau \big| \leq \frac{4 \tau}{ \tau + \gamma + 1} \big( \epsilon + \frac{C_0}{ \widehat{\sigma}^2(\mb{x})} \big) + \frac{2C_0 \min( \widehat{\mu}(\mb{x}), \tau )}{ \widehat{\sigma}^2(\mb{x}) - 1/4 } \Big( \log\big( 1 + \frac{d}{\tau + \gamma} \big) + \frac{\tau + \gamma - 1}{\tau + \gamma} \Big).
$$
Similarly, for $|\widetilde{\nu}_\tau - \overbar{\nu}_\tau|$, 
\begin{align*}
    |\widetilde{\nu}_\tau - \overbar{\nu}_\tau| \leq \frac{2\gamma}{\tau + \gamma} \epsilon + \frac{\gamma C_0}{ \widehat{\sigma}^2(\mb{x})} \Big( \log\big( 1 + \frac{d}{\tau + \gamma - 1} \big) + 1 - \frac{1}{\tau + \gamma - 1} \Big).
\end{align*}
This completes the proof.
\end{proof}

\subsection{Proof of Lemma \ref{lem:blind_approx}}
\begin{proof}
With the same argument in the proof of Lemma \ref{lem:app}, it suffices to consider
\begin{align*}
    \big| \widetilde{\mathbb P}( \widehat{\Gamma}(\mb{x}) = l ) - \widetilde{\mathbb P}( \widehat{\Gamma}_{\sm j}(\mb{x}) = l ) \big| \leq \sum_{l'=l-1}^{l} \big| \widetilde{\mathbb P}( \widehat{\Gamma}(\mb{x}) \leq l' ) - \widetilde{\mathbb P}( \widehat{\Gamma}_{\sm j}(\mb{x}) \leq l') \big|.
\end{align*}
Denote $I = \widehat{\sigma}(\mb{x})^{-1}( l + 1/2 -\widehat{\mu}(\mb{x}))$ and $I_{\sm j} = \widehat{\sigma}_{\sm j}(\mb{x})^{-1}( l + 1/2 -\widehat{\mu}_{\sm j}(\mb{x}))$, we have
\begin{align*}
    \big| \widetilde{\mathbb P}( \widehat{\Gamma}(\mb{x}) \leq l ) & - \widetilde{\mathbb P}( \widehat{\Gamma}_{\sm j}(\mb{x}) \leq l) \big| = \big| \Psi ( I ) - \Psi_{\sm j}(I_{\sm j}) \big| \\
    & \leq \big| \Phi(I) - \Phi(I_{\sm j}) \big| + \frac{1}{6} \big| \widehat{\eta}(\mb{x}) ( 1 - I^2) \phi(I) - \widehat{\eta}_{\sm j}(\mb{x}) ( 1 - I_{\sm j}^2) \phi(I_{\sm j}) \big| =: T_1 + \frac{1}{6} T_2.
\end{align*}
Next, we turn to treat $T_1$ and $T_2$ separately. Without loss generalization, we assume $|I_{\sm j}| \geq |I|$, then
\begin{align*}
    T_1 & \leq \big| \int_{I}^{I_{\sm j}} \phi(\mb{x}) d\mb{x} \big| \leq | I_{\sm j} - I | \phi(I) \\
    & = \Big(| \widehat{\sigma}^{-1}(\mb{x}) - \widehat{\sigma}_{\sm j}^{-1}(\mb{x}) | \big| l + 1/2 - \widehat{\mu}(\mb{x}) \big| + \widehat{\sigma}_{\sm j}^{-1}(\mb{x}) \big| \widehat{\mu}(\mb{x}) - \widehat{\mu}_{\sm j}(\mb{x}) \big| \Big) \phi(I) \\
    & = \widehat{\sigma}^{-1}_{\sm j}(\mb{x}) ( \widehat{\sigma}(\mb{x}) - \widehat{\sigma}_{\sm j}(\mb{x}) )  |I| \phi(I) + \widehat{\sigma}_{\sm j}^{-1}(\mb{x}) \widehat{q}_{j}(\mb{x}) \phi(I) \\
    & \leq \frac{\widehat{q}_{j}(\mb{x})(1 - \widehat{q}_{j}(\mb{x}))}{ \widehat{\sigma}_{\sm j}(\mb{x}) \big(\widehat{\sigma}_{\sm j}(\mb{x}) + \widehat{\sigma}(\mb{x}) \big) }  |I| \phi(I) + \frac{1}{\sqrt{2\pi} \widehat{\sigma}_{\sm j}(\mb{x})} \\
    & \leq \frac{1}{4 \sqrt{2 \pi}} \Big( \frac{1}{2\sqrt{e}( \widehat{\sigma}^2(\mb{x}) -1/4)} + \frac{4}{\sqrt{\widehat{\sigma}^2(\mb{x})-1/4}}\Big),
\end{align*}
where the last inequality follows the fact that $|u| \phi(u) \leq 1/\sqrt{2e\pi}$ and $0 \leq \phi(u) \leq 1/\sqrt{2\pi}$. For $T_2$, let $g(u) = (1-u^2)\phi(u)$, we have
\begin{align*}
    T_2 & \leq \big| \widehat{\eta}(\mb{x}) ( 1 - I^2) \phi(I) \big| + \big| \widehat{\eta}_{\sm j}(\mb{x}) ( 1 - I_{\sm j}^2) \phi(I_{\sm j}) \big| \leq \frac{1}{\sqrt{2\pi}} \frac{\widehat{m}_3(\mb{x})}{ (\widehat{\sigma}^2(\mb{x})-1/4)^{3/2} },
\end{align*}
where the last inequality follows from the fact that $|g(u)| \leq 1/\sqrt{2\pi}$, and $\widehat{m}_3(\mb{x}) = \sum_{j=1}^d \widehat{p}_j(\mb{x})( 1 - \widehat{p}_j(\mb{x}) ) ( 1 - 2\widehat{p}_j(\mb{x}))$. Taken together, using the same argument in the proof of Lemma \ref{lem:app}, we have
\begin{align*}
    & \big| \widetilde{\omega}^b_\tau - \widetilde{\omega}_\tau \big| \leq \frac{1}{4 \sqrt{2\pi}} \Big( \frac{1/(2\sqrt{e})}{ \widehat{\sigma}^2(\mb{x}) - 1/4 } + \frac{4}{\sqrt{\widehat{\sigma}^2(\mb{x}) - 1/4}} + \frac{4 \widehat{m}_3(\mb{x})}{ (\widehat{\sigma}^2(\mb{x}) - 1/4)^{3/2} } \Big) \big( \log\big( 1 + \frac{d}{\tau + \gamma} \big) + 1 \big).
\end{align*}
Combining Lemma \ref{lem:app}, the desirable result then follows.
\end{proof}

\subsection{Proof of Lemma \ref{lem:overlapping}}
\begin{proof}
Denote $\mathcal{K}_+ = \{ 1 \leq k \leq K: \alpha_k > 0 \}$. We first prove the necessity. Suppose $\pmb{\Delta}^*_k$ is a global minimizer of ${\Dice}_k(\cdot)$, for $k \in \mathcal{K}_+$. Then for any $\pmb{\Delta} = ( \pmb{\Delta}_1, \cdots, \pmb{\Delta}_K )$, we have 
$$
\mDice_\gamma(\pmb{\Delta}^*) = \sum_{k \in \mathcal{K}_+} \alpha_k {\Dice}_{\gamma, k}(\pmb{\Delta}^*_k) \leq \sum_{k \in \mathcal{K}_+} \alpha_k {\Dice}_{\gamma, k}(\pmb{\Delta}_k),
$$
yields that $\pmb{\Delta}^*$ is a global minimizer of $\mDice_\gamma(\cdot)$. We next prove the sufficiency by contradiction. Suppose $\pmb{\Delta}^*$ is a global minimizer of $\mDice_\gamma(\cdot)$, yet there exists $k_0 \in \mathcal{K}_+$ such that $\pmb{\Delta}_{k_0}^*$ is not a minimizer of ${\Dice}_{\gamma, k_0}(\cdot)$. Thus, there exists a segmentation rule $\widetilde{\pmb{\Delta}}$ such that $ {\Dice}_{\gamma, k_0}( \widetilde{\pmb{\Delta}}) <  {\Dice}_{\gamma, k_0}( \pmb{\Delta}_{k_0}^* )$, then let $\widetilde{\pmb{\Delta}} = ( \pmb{\Delta}^*_1, \cdots, \pmb{\Delta}^*_{k_0 - 1}, \widetilde{\pmb{\Delta}}_{k_0}, \pmb{\Delta}^*_{k_0 + 1}, \cdots, \pmb{\Delta}^*_K )$ 
$$
\mDice_\gamma( \pmb{\Delta}^*) = \sum_{k \in \mathcal{K}_+} \alpha_k {\Dice}_{\gamma, k}(\pmb{\Delta}^*_k) > \sum_{k \in \mathcal{K}_+} \alpha_k {\Dice}_{\gamma, k}(\widetilde{\pmb{\Delta}}) = \mDice_\gamma( \widetilde{\pmb{\Delta}}),
$$
which leads to contradiction of that $\pmb{\Delta}^*$ is a global minimizer of $\mDice_\gamma(\cdot)$. The desirable result then follows.
\end{proof}

\subsection{Proof of Lemma \ref{lem:nonoverlap_bayes_rule}}
\begin{proof}
Given a segmentation rule $\pmb{\Delta}(\mb{X}) = ( \pmb{\Delta}_1(\mb{X}), \cdots, \pmb{\Delta}_K(\mb{X}) )$, $\mDice_\gamma(\cdot)$ can be rewritten as 
$$
\mDice_\gamma(\pmb{\Delta}) = \sum_{k=1}^K \alpha_k {\Dice}_{\gamma, k} (\pmb{\Delta}_k(\mb{X})).
$$
Similarly, it is equivalent to consider the point-wise minimization conditional on $\mb{X} = \mb{x}$:
\begin{align*}
 \pmb{\Delta}^*(\mb{x}) & = \argmax_{\mb{V} \in \{0,1\}^{d \times K}} \ \mDice_\gamma(\mb{V}|\mb{x}), \quad \text{s.t.} \sum_{k=1}^K \mb{V}_k = \mb{1}_d, \quad \\
 & \mDice_\gamma(\mb{V}|\mb{x}) = \sum_{k=1}^K \alpha_k {\Dice}_{\gamma,k} (\mb{V}_k | \mb{x}),
\end{align*}
where $\mb{V}_k$ is the $k$-th column of $\mb{V}$.
According to \eqref{eqn:additive}, we have
\begin{align*}
\mDice_\gamma(\mb{V}|\mb{x}) & = \sum_{k=1}^K \alpha_k \sum_{j \in I({\mb{V}_k}) } \sum_{\substack{\mb{y}_{\sm j,k} \in \{0,1\}^{d-1} \\ y_{j,k} = 1}}  \frac{2 \mathbb{P}(\mb{Y}_k = \mb{y}_k | \mb{x}) }{  \tau_k + \| \mb{y}_k \|_1 + \gamma } +  \sum_{k=1}^K \alpha_k \mathbb{E} \Big( \frac{  \gamma }{ \| \mb{Y}_{\cdot k} \|_1 + \tau_k + \gamma } \Big) \\
& = \sum_{k=1}^K \sum_{j=1}^d  R_{jk}(\tau_k) v_{jk} + \sum_{k=1}^K \alpha_k \overbar{\nu}(\tau_k), 
\end{align*}
where $v_{jk} \in \{0,1\}$ is the segmentation indicator of the $j$-th feature under the class-$k$, and $R_{jk}(\cdot)$ is a reward function defined as: 
$$R_{jk}(\tau_k) = \alpha_k  \sum_{\substack{\mb{y}_{\sm j,k} \in \{0,1\}^{d-1} \\ y_{j,k} = 1}}  \frac{2 \mathbb{P}(\mb{Y}_k = \mb{y}_k | \mb{x}) }{  \tau_k + \| \mb{y}_k \|_1 }.$$
Now, suppose the optimal volume function $\pmb{\tau}^*(\mb{x}) = ( \tau_1^*(\mb{x}), \cdots, \tau_{K}^*(\mb{x}))^\intercal$ is given, then $\overbar{\nu}(\tau^*_k)$ becomes a constant, and the point-wise maximization on $\mDice_\gamma$ is equivalent to:
\begin{align}
    \label{eqn:assignment}
    \max_{\mb{V} \in \{0,1\}^{d \times K} } & \quad \sum_{k=1}^K \sum_{j=1}^d R^*_{jk} v_{jk}, \nonumber \\ 
    \text{subject to} &  \quad  \sum_{j=1}^d v_{jk} = \tau^*_k(\mb{x}), \quad \sum_{k=1}^K v_{jk} = 1, \nonumber \\
    \quad & \text{for } k = 1, \cdots, K; \quad \text{for } j = 1, \cdots, d,
\end{align}
where $R^*_{jk} = R_{jk}(\tau^*_k(\mb{x}))$ is the reward under $\tau_k^*(\mb{x})$. Note that \eqref{eqn:assignment} is the formulation for the assignment problem \citep{kuhn1955hungarian}. This completes the proof. 
\end{proof}

\subsection{Proof of Lemma \ref{lem:Dice-calibrated}}
\begin{proof}
For simplicity, we construct a counter example based on $\gamma = 0$ and $d=2$, that is, $\mb{Y} = (Y_1, Y_2)^\intercal$ and $\mb{p}(\mb{x}) = (p_1(\mb{x}), p_2(\mb{x}))^\intercal$. Without loss generality, we assume $p_1(\mb{x}) > p_2(\mb{x})$. 

First, we derive the Bayes rule in Theorem \ref{thm:Dice_bayes} for this case. Note that it suffices to compare the scores for $\tau = 1$ ($\mb{v} = (1, 0)^\intercal$) and $\tau = 2$ $(\mb{v} = (1, 1)^\intercal)$.
\begin{align*}
    \Dice( (1, 0)^\intercal | \mb{x} ) = p_1(\mb{x}) - \frac{1}{3} p_1(\mb{x}) p_2(\mb{x}), \quad \Dice( (1, 1)^\intercal | \mb{x} ) = \frac{2}{3}( p_1(\mb{x}) + p_2(\mb{x}) ) - \frac{1}{3} p_1(\mb{x}) p_2(\mb{x}).
\end{align*}
Therefore, the Bayes rule for Dice optimal segmentation is:
$$
\pmb{\delta}^*(\mb{x}) = (1, 0)^\intercal, \text{ if } \frac{1}{2}p_1(\mb{x}) > p_2(\mb{x}), \quad \pmb{\delta}^*(\mb{x}) = (1, 1)^\intercal, \text{ otherwise}.
$$

\smallbreak

Now, we check the Dice-calibrated for classification-calibrated losses. 
For example, $p_1(\mb{x}) = 0.45$ and $p_2(\mb{x}) = 0.44$, then 
$$\widetilde{\pmb{\delta}}(\mb{x}) = \mb{1}(\mb{p}(\mb{x}) \geq 0.5) = (0, 0)^\intercal \neq (1, 1)^\intercal = \pmb{\delta}^*(\mb{x}),$$ where the first equality follows from the definition of a classification-calibrated loss: the decision rule must agree with the conditional probabilities. Therefore, $\Dice( \widetilde{\pmb{\delta}} ) < \Dice( \pmb{\delta}^* )$ yields that a classification-calibrated loss with thresholding at 0.5 is not Dice-calibrated.

\smallbreak

\end{proof}

\subsection{Proof of Lemma \ref{lem:RankDice-calibrated}}
\begin{proof}
By the definition of a strictly proper loss, see \cite{gneiting2007strictly} and references herein, and the formulation in \eqref{eqn:pop_Rankdice}, we have $\widehat{\mb{q}}(\mb{x}) = \mb{p}(\mb{x}) = \big( \mathbb{P}(Y_1=1 | \mb{X} = \mb{x}), \cdots, \mathbb{P}(Y_d =1 | \mb{X} = \mb{x}) \big)^\intercal$. Then the estimation of $\widehat{\tau}(\mb{x})$ and $\widehat{\pmb{\delta}}(\mb{x})$ agrees with the definition of $\tau^*(\mb{x})$ and $\pmb{\delta}^*(\mb{x})$. This completes the proof.
\end{proof}

\subsection{Proof of Theorem \ref{thm:risk_bound}}
\begin{proof} 
First, we consider point-wise approximation of the Dice metric under probabilities $\mb{p}$ and $\widehat{\mb{q}}$. For any $\pmb{\delta}$, such that $\delta_j(\mb{x}) = 1$ for $j = j_1, \cdots, j_\tau$, and $\delta_j(\mb{x}) = 0$ otherwise. Define
\begin{align*}
\widehat{\Dice}_\gamma({\pmb{\delta}}(\mb{x}) | \mb{X} = \mb{x}) & := \sum_{s=1}^{\tau} \sum_{l=0}^{d-1} \frac{2 \widehat{q}_{j_s}(\mb{x}) \mathbb{P} \big( \widehat{\Gamma}_{\sm j_s}(\mb{x}) = l \big)}{\tau + l + \gamma + 1} + \sum_{l=0}^d \frac{\gamma \mathbb{P}\big( \widehat{\Gamma}(\mb{X}) = l \big)}{\tau + l + \gamma}  \\
& = \sum_{s=1}^{\tau} 2\widehat{q}_{j_s}(\mb{x}) \mathbb{E}\big( \frac{1}{ \tau + \gamma + 1 + \widehat{\Gamma}_{\sm j_s}(\mb{x}) } \big) + \gamma \mathbb{E} \big( \frac{1}{ \tau + \gamma + \widehat{\Gamma}(\mb{x})} \big) \\
\Dice_\gamma({\pmb{\delta}}(\mb{x})| \mb{X} = \mb{x}) & := \sum_{s=1}^{\tau} \sum_{l=0}^{d-1} \frac{2 p_{j_s}(\mb{x}) \mathbb{P} \big( \Gamma_{\sm j_s}(\mb{x}) = l \big)}{\tau + l + \gamma + 1} + \sum_{l=0}^d \frac{\gamma \mathbb{P}\big( {\Gamma}(\mb{X}) = l \big)}{\tau + l + \gamma}  \\
& = \sum_{s=1}^{\tau} 2 p_{j_s}(\mb{x}) \mathbb{E}\big( \frac{1}{ \tau + \gamma + 1 + \Gamma_{\sm j_s}(\mb{x}) } \big) + \gamma \mathbb{E} \big( \frac{1}{ \tau + \gamma + \Gamma(\mb{x})} \big). \\
\end{align*}
Now, we have
\begin{align*}
    & \big| \widehat{\Dice}_\gamma({\pmb{\delta}}(\mb{x}) | \mb{X} = \mb{x}) - \Dice_\gamma({\pmb{\delta}}(\mb{x})| \mb{X} = \mb{x}) \big| \\
    & \leq  \Big| 2\sum_{s=1}^{\tau} \Big( \widehat{q}_{j_s}(\mb{x}) \mathbb{E}\big( \frac{1}{ \tau + \gamma + 1 + \widehat{\Gamma}_{\sm j_s}(\mb{x}) } \big) - p_{j_s}(\mb{x}) \mathbb{E}\big( \frac{1}{ \tau + \gamma + 1 + \Gamma_{\sm j_s}(\mb{x}) } \big) \Big) \Big| \\
    & \qquad + \gamma \Big|  \mathbb{E} \big( \frac{1}{ \tau + \gamma + \widehat{\Gamma}(\mb{x})} \big) - \mathbb{E} \big( \frac{1}{ \tau + \gamma + \Gamma(\mb{x})} \big) \Big| \\
    & \leq  \Big| 2\sum_{s=1}^{\tau} \widehat{q}_{j_s}(\mb{x}) \Big( \mathbb{E}\big( \frac{1}{ \tau + \gamma + 1 + \widehat{\Gamma}_{\sm j_s}(\mb{x}) } \big) - \mathbb{E}\big( \frac{1}{ \tau + \gamma + 1 + \Gamma_{\sm j_s}(\mb{x}) } \big)  \Big) \Big| \\ 
    & \quad + \Big| 2\sum_{s=1}^{\tau} \big( \widehat{q}_{j_s}(\mb{x}) - p_{j_s}(\mb{x})\big) \mathbb{E}\big( \frac{1}{ \tau + \gamma + 1 + \Gamma_{\sm j_s}(\mb{x}) } \big) \Big) \Big| + \gamma \Big|  \mathbb{E} \big( \frac{1}{ \tau + \gamma + \widehat{\Gamma}(\mb{x})} \big) - \mathbb{E} \big( \frac{1}{ \tau + \gamma + \Gamma(\mb{x})} \big) \Big| \\
    & \leq 2 \sum_{s=1}^{\tau} \Big| \frac{ \mathbb{E}\big( \Gamma_{\sm j_s}(\mb{x}) - \widehat{\Gamma}_{\sm j_s}(\mb{x})\big)}{ (\tau + \gamma + 1)^2 } \Big| + 2 \sum_{s=1}^{\tau} \frac{\big| \widehat{q}_{j_s}(\mb{x}) - p_{j_s}(\mb{x})\big|}{\tau + \gamma + 1 } + \gamma \Big| \frac{ \mathbb{E}\big( \Gamma(\mb{x}) - \widehat{\Gamma}(\mb{x})\big)}{ (\tau + \gamma)^2 } \Big| \\
    & \leq 2 \sum_{s=1}^{\tau} \frac{ | \|\widehat{\mb{q}}(\mb{x})\|_1 - \|\mb{p}(\mb{x})\|_1 | + | \widehat{q}_{j_s}(\mb{x}) - p_{j_s}(\mb{x}) | }{ (\tau + \gamma + 1)^2 } + 2 \sum_{s=1}^{\tau} \frac{\big| \widehat{q}_{j_s}(\mb{x}) - p_{j_s}(\mb{x})\big|}{\tau + \gamma + 1 } + \gamma \frac{ | \|\widehat{\mb{q}}(\mb{x})\|_1 - \|\mb{p}(\mb{x})\|_1 | }{ (\tau + \gamma)^2 } \\
    & \leq \big( \frac{3}{{2(1 + \gamma)}} + c_1 \big) { \|\widehat{\mb{q}}(\mb{x}) - \mb{p}(\mb{x}) \|_1 },
\end{align*}
where the second inequality follows from the triangle inequality, $c_1 = 0$ if $\gamma = 0$, $c_1 = 1/\gamma$ if $\gamma > 0$. Therefore, 
\begin{align*}
    \Dice_\gamma(\pmb{\delta}^*) & - \Dice_\gamma(\widehat{\pmb{\delta}}) = \Dice_\gamma(\pmb{\delta}^*) - \widehat{\Dice}_\gamma(\pmb{\delta}^*) + \widehat{\Dice}_\gamma(\pmb{\delta}^*) - \widehat{\Dice}_\gamma(\widehat{\pmb{\delta}}) +  \widehat{\Dice}_\gamma(\widehat{\pmb{\delta}}) - \Dice_\gamma(\widehat{\pmb{\delta}}) \\
    & \leq \Dice_\gamma(\pmb{\delta}^*) - \widehat{\Dice}_\gamma(\pmb{\delta}^*) +  \widehat{\Dice}_\gamma(\widehat{\pmb{\delta}}) - \Dice_\gamma(\widehat{\pmb{\delta}}) \\
    & \leq \mathbb{E}_\mb{X}\big| \widehat{\Dice}_\gamma({\pmb{\delta}^*}(\mb{X}) | \mb{X}) - \Dice_\gamma({\pmb{\delta}^*}(\mb{X})| \mb{X}) \big| + \mathbb{E}_\mb{X}\big| \widehat{\Dice}_\gamma(\widehat{\pmb{\delta}}(\mb{X}) | \mb{X}) - \Dice_\gamma(\widehat{\pmb{\delta}}(\mb{X})| \mb{X}) \big| \\
    & \leq \big( \frac{3}{1 + \gamma} + 2c_1 \big) \mathbb{E} { \|\widehat{\mb{q}}(\mb{X}) - \mb{p}(\mb{X}) \|_1 },
\end{align*}
where the first inequality follows from the definition of $\widehat{\pmb \delta}$ such that $\widehat{\Dice}_\gamma(\pmb{\delta}^*) - \widehat{\Dice}_\gamma(\widehat{\pmb{\delta}}) \leq 0$. This completes the proof.
\end{proof}

\subsection{Proof of Corollary \ref{thm:rate}}
\begin{proof} 
According to the Pinsker's inequality, we have 
    \begin{align*}
        \mathbb{E}_\mb{X} & { \|\widehat{\mb{q}}(\mb{X}) - \mb{p}(\mb{X}) \|_1 } = \sum_{j=1}^d \mathbb{E}_\mb{X} { |\widehat{{q}}_j(\mb{X}) - {p}_j(\mb{X})| } \leq \sum_{j=1}^d \mathbb{E}_\mb{X} \sqrt{ \frac{1}{2} \KL \big( \mathbb{P}(Y_j|\mb{X}), \widehat{\mathbb{P}}(Y_j|\mb{X}) \big)} \\
        & \leq \sqrt{ \frac{d}{2} \sum_{j=1}^d \mathbb{E}_\mb{X} \KL \big( \mathbb{P}(Y_j|\mb{X}), \widehat{\mathbb{P}}(Y_j|\mb{X}) \big)} \\
        & = \sqrt{\frac{d}{2}} \sqrt{ \mathbb{E}\Big( l_{\text{CE}}\big(\mb{Y}, \widehat{\mb{q}}(\mb{X}) \big) \Big) - \mathbb{E}\Big( l_{\text{CE}}\big(\mb{Y}, \mb{p}(\mb{X}) \big) \Big)},
    \end{align*}
where the last inequality follows from the Cauchy-Schwarz inequality and the Jensen's inequality, and $\KL\big( \mathbb{P}(Y_j|\mb{x}), \widehat{\mathbb{P}}(Y_j|\mb{x}) \big) :=  p_j(\mb{x}) \log \big( p_j(\mb{x})/ \widehat{q}_j(\mb{x}) \big) + ( 1 - p_j(\mb{x}) ) \log ( (1 - p_j(\mb{x})) / (1 - \widehat{q}_j(\mb{x})))$ is the KL divergence between $\mathbb{P}(Y_j|\mb{x})$ under $\mb{p}$ and $\widehat{\mathbb{P}}(Y_j|\mb{x})$ under $\widehat{\mb{q}}$. The desirable result then follows by combining \eqref{eqn:risk_bound} in Theorem \ref{thm:risk_bound}. This completes the proof.
\end{proof}

\subsection{Proof of Lemma \ref{lem:IoU_bayes}}
\begin{proof}
Denote $\mb{y}_I = (y_j: j\in I)^\intercal$, $\mb{y}_{\sm I} = (y_j : j\notin I )^\intercal$, and let $I(\mb{v}) = I(\pmb{\delta}(\mb{x})) = \{j: v_{j} = 1\}$ be a segmentation index set, and $\| \mb{v} \|_1 = \tau$. Again, consider the point-wise maximization:
$$
 \pmb{\delta}^*(\mb{x}) = \argmax_{\mb{v} \in \{0,1\}^d} \ {\IoU}_\gamma(\mb{v}|\mb{x}),
$$
where ${\IoU}_\gamma(\mb{v}|\mb{x})$ is defined as
\begin{align*}
    {\IoU}_\gamma(\mb{v}|\mb{x}) 
    &= \mathbb{E} \Big( \frac{ \|\mb{Y}_{I(\mb{v})}\|_1 + \gamma }{ \| \mb{Y}_{\sm I(\mb{v})} \|_1 + \| \mb{v} \|_1 + \gamma } \Big| \mb{X} = \mb{x} \Big)\\
    &= \mathbb{E} \Big( \frac{\|\mb{Y}_{I(\mb{v})}\|_1 }{ \| \mb{Y}_{\sm I(\mb{v})} \|_1 + \tau + \gamma } \Big| \mb{X} = \mb{x} \Big) + \mathbb{E} \Big( \frac{ \gamma }{ \| \mb{Y}_{\sm I(\mb{v})} \|_1 + \tau + \gamma } \Big| \mb{X} = \mb{x} \Big)\\
    &= \Big(\mathbb{E}(\|\mb{Y}_{I(\mb{v})}\|_1 | \mb{X}=\mb{x}) + \gamma\Big) \mathbb{E}\Big( \frac{1 }{ \| \mb{Y}_{\sm I(\mb{v})} \|_1 + \tau + \gamma } \Big| \mb{X} = \mb{x} \Big),
\end{align*}
where the last equality follows from the fact that $\mb{Y}_{I(\mb{v})} \perp \mb{Y}_{\sm I(\mb{v})} \mid \mb{X}$. Fix $\tau = \|\mb{v}\|_1$. The first term is maximized at $\mb{v}^* = (v_1^*,\cdots,v_d^*)$ with
\begin{align*}
    {v}^*_j = 
    \begin{cases}
        1 & \text{if } p_j(\mb{x}) \text{ ranks top }\tau, \\
        0 & \text{otherwise},
    \end{cases}
\end{align*}
and the maximum value is $\sum_{j\in J_{\tau}(\mb{x})} p_j(\mb{x}) + \gamma$.
The second term is 
\begin{align*}
    \mathbb{E}\Big( \frac{1 }{ \| \mb{Y}_{\sm I(\mb{v})} \|_1 + \tau + \gamma } \Big| \mb{X} = \mb{x} \Big)
    = \mathbb{E}\Big( \frac{1 }{ \Gamma_{\sm I(\mb{v})}(\mb{x}) + \tau + \gamma } \Big).
\end{align*}
Given $I(\mb{v})\neq I(\mb{v}^*)$, we have $\Gamma_{\sm I(\mb{v})}(\mb{x}) = \sum_{j=1}^{d-\tau} B_{j}$ and $\Gamma_{\sm I(\mb{v}^*)}(\mb{x}) = \sum_{j=1}^{d-\tau} B^*_{j}$, where $B_{j}$ and $B^*_j$ are independent Bernoulli variables with success probability $p_j\geq p^*_j$, respectively, for $j=1,\cdots,d-\tau$. 
By Theorem 2.2.9 of \cite{belzunce2015introduction}, $\Gamma_{\sm I(\mb{v})}(\mb{x})$ is stochastically greater than $\Gamma_{\sm I(\mb{v}^*)}(\mb{x})$, namely 
\begin{equation*}
    \mathbb{P}(\Gamma_{\sm I(\mb{v})}(\mb{x})\leq \varsigma) \leq 
    \mathbb{P}(\Gamma_{\sm I(\mb{v}^*)}(\mb{x})\leq \varsigma) 
\end{equation*}
for any $\varsigma\in\mathbb{R}$. Thus, 
\begin{align*}
    \mathbb{E}\Big( \frac{1 }{ \Gamma_{\sm I(\mb{v})}(\mb{x}) + \tau + \gamma } \Big) 
    &= \int_{0}^\infty \mathbb{P}\Big( \frac{1 }{ \Gamma_{\sm I(\mb{v})}(\mb{x}) + \tau + \gamma } \geq \varsigma \Big)d\varsigma \\
    &\leq \int_{0}^\infty \mathbb{P}\Big( \frac{1 }{ \Gamma_{\sm I(\mb{v}^*)}(\mb{x}) + \tau + \gamma } \geq \varsigma \Big)d\varsigma = \mathbb{E}\Big( \frac{1}{ \Gamma_{\sm I(\mb{v}^*)}(\mb{x}) + \tau + \gamma } \Big).
\end{align*}
As a result, the second term is also maximized at $\mb{v}^*$. Therefore, $\mb{v}^*$ maximizes $\IoU_{\gamma}(\mb{v}|\mb{x})$ for a fixed $\tau$. The desired result follows.
\end{proof}

\subsection{Proof of Lemma \ref{lem:shrinkage_IoU}}

\begin{proof}
    Without loss of generality, assuming $\widehat{q}_1(\mb{x})\geq \cdots \geq \widehat{q}_{d}(\mb{x})$ are fixed, then for any $\tau'>\tau$, 
    \begin{equation*}
        \begin{split}
            &\varpi_{\tau}(\mb{x}) - \varpi_{\tau'}(\mb{x}) \\
            &= \Big( \sum_{j=1}^{\tau} \widehat{q}_{j}(\mb{x}) + \gamma \Big) \mathbb{E}\Big( \frac{1}{\widehat{\Gamma}_{\sm J_{\tau}(\mb{x})} + \tau + \gamma} \Big) - \Big( \sum_{j=1}^{\tau'} \widehat{q}_{j}(\mb{x}) + \gamma \Big) \mathbb{E}\Big( \frac{1}{\widehat{\Gamma}_{\sm J_{\tau'}(\mb{x})} + \tau' + \gamma} \Big)\\
            &= \Big( \sum_{j=1}^{\tau} \widehat{q}_{j}(\mb{x}) + \gamma \Big) \mathbb{E}\Big( \frac{1}{\widehat{\Gamma}_{\sm J_{\tau'}(\mb{x})} + \sum_{j=\tau+1}^{\tau'} \widehat{B}_j(\mb{x}) + \tau + \gamma} \Big) - \Big( \sum_{j=1}^{\tau'} \widehat{q}_{j}(\mb{x}) + \gamma \Big) \mathbb{E}\Big( \frac{1}{\widehat{\Gamma}_{\sm J_{\tau'}(\mb{x})} + \tau' + \gamma} \Big)\\
            &= \mathbb{E}\Big( \frac{(\sum_{j=1}^{\tau} \widehat{q}_j(\mb{x}) + \gamma )(\widehat{\Gamma}_{\sm J_{\tau'}(\mb{x})} + \tau' + \gamma) - (\sum_{j=1}^{\tau'} \widehat{q}_{j}(\mb{x}) + \gamma)(\widehat{\Gamma}_{\sm J_{\tau'}(\mb{x})} + \sum_{j=\tau+1}^{\tau'} \widehat{B}_j(\mb{x}) + \tau + \gamma) }{(\widehat{\Gamma}_{\sm J_{\tau'}(\mb{x})} + \sum_{j=\tau+1}^{\tau'} \widehat{B}_j(\mb{x}) + \tau + \gamma)(\widehat{\Gamma}_{\sm J_{\tau'}(\mb{x})} + \tau' + \gamma)}\Big) \\
            &= \mathbb{E}\Big( \frac{(\sum_{j=1}^{\tau} \widehat{q}_j(\mb{x}) + \gamma )(\tau' - \tau - \sum_{j=\tau+1}^{\tau'} \widehat{B}_j(\mb{x}) )}{(\widehat{\Gamma}_{\sm J_{\tau'}(\mb{x})} + \sum_{j=\tau+1}^{\tau'} \widehat{B}_j(\mb{x}) + \tau + \gamma)(\widehat{\Gamma}_{\sm J_{\tau'}(\mb{x})} + \tau' + \gamma)}\Big) - \mathbb{E}\Big(\frac{\sum_{j=\tau+1}^{\tau'} \widehat{q}_{j}(\mb{x}) }{\widehat{\Gamma}_{\sm J_{\tau'}(\mb{x})} + \tau' + \gamma}\Big) \\
            &\geq \frac{(\sum_{j=1}^{\tau} \widehat{q}_j(\mb{x}) + \gamma )(\tau' - \tau - \sum_{j=\tau+1}^{\tau'} \widehat{q}_j(\mb{x}) )}{( (d-\tau')\widehat{q}_{\tau+1}(\mb{x}) + \tau' + \gamma)^2} - \frac{\sum_{j=\tau+1}^{\tau'} \widehat{q}_{j}(\mb{x})}{\tau' + \gamma}\\
            &\geq \frac{(\sum_{j=1}^{\tau} \widehat{q}_j(\mb{x}) + \gamma )(\tau' - \tau)(1- \widehat{q}_{\tau+1}(\mb{x}) )}{( (d-\tau')\widehat{q}_{\tau+1}(\mb{x}) + \tau' + \gamma)^2} - \frac{(\tau'-\tau) \widehat{q}_{\tau+1}(\mb{x})}{\tau' + \gamma} \geq 0,
        \end{split}
    \end{equation*}
    whenever 
    \begin{equation*}
        \begin{split}
            \sum_{j=1}^{\tau}\widehat{q}_j(\mb{x}) + \gamma 
            &\geq \frac{\widehat{q}_{\tau+1}(\mb{x})}{1 - \widehat{q}_{\tau+1}(\mb{x})} \max_{\tau'>\tau} \frac{( (d-\tau')\widehat{q}_{\tau+1}(\mb{x}) + \tau' + \gamma)^2}{\tau' + \gamma} \\
            &= \frac{\widehat{q}_{\tau+1}(\mb{x})}{1 - \widehat{q}_{\tau+1}(\mb{x})} \max\Big(d + \gamma, \frac{( (d-\tau)\widehat{q}_{\tau+1}(\mb{x}) + \tau + \gamma)^2}{\tau + \gamma}\Big).
        \end{split}
    \end{equation*}
    This completes the proof.
\end{proof}

\subsection{Proof of Theorem \ref{thm:risk_bound_IoU}}

\begin{proof}
    Similar to the proof of Theorem \ref{thm:risk_bound}, we consider point-wise approximation of IoU metric under $\mb{p}$ and $\widehat{\mb{q}}$. For any $\mb{\delta}$ such that $\delta_j(\mb{x}) = 1$ for $j\in J_{\tau}(\mb{x})$ and $\delta_j(\mb{x}) = 0$ otherwise. Define 
    \begin{align*}
        \widehat{\IoU}_\gamma({\pmb{\delta}}(\mb{x}) | \mb{X} = \mb{x}) & := \Big(\sum_{s=1}^{\tau} \widehat{q}_{j_s}(\mb{x}) + \gamma\Big) \sum_{l=0}^{d-\tau} \frac{ \mathbb{P} \big( \widehat{\Gamma}_{\sm J_{\tau}(\mb{x})}(\mb{x}) = l \big)}{\tau + l + \gamma} \\
        & = \Big(\sum_{s=1}^{\tau} \widehat{q}_{j_s}(\mb{x}) + \gamma\Big) \mathbb{E}\big( \frac{1}{ \tau + \gamma + \widehat{\Gamma}_{\sm J_{\tau}(\mb{x})}(\mb{x}) } \big), \\
        \IoU_\gamma({\pmb{\delta}}(\mb{x})| \mb{X} = \mb{x}) & := \Big(\sum_{s=1}^{\tau} p_j(\mb{x}) + \gamma \Big) \sum_{l=0}^{d-\tau} \frac{\mathbb{P} \big( \Gamma_{\sm J_{\tau}(\mb{x})}(\mb{x}) = l \big)}{\tau + l + \gamma} \\
        & = \Big(\sum_{s=1}^{\tau} p_j(\mb{x}) + \gamma \Big) \mathbb{E}\big( \frac{1}{ \tau + \gamma + \Gamma_{\sm J_{\tau}(\mb{x})}(\mb{x}) } \big). 
        \end{align*}
        We have
        \begin{align*}
            & \big| \widehat{\IoU}_\gamma({\pmb{\delta}}(\mb{x}) | \mb{X} = \mb{x}) - \IoU_\gamma({\pmb{\delta}}(\mb{x})| \mb{X} = \mb{x}) \big| \\
            & \leq  \Big| \sum_{s=1}^{\tau} \Big( \widehat{q}_{j_s}(\mb{x}) \mathbb{E}\big( \frac{1}{ \tau + \gamma + \widehat{\Gamma}_{\sm J_{\tau}(\mb{x})}(\mb{x}) } \big) - p_{j_s}(\mb{x}) \mathbb{E}\big( \frac{1}{ \tau + \gamma + \Gamma_{\sm J_{\tau}(\mb{x})}(\mb{x}) } \big) \Big) \Big| \\
            & \qquad + \gamma \Big|  \mathbb{E} \big( \frac{1}{ \tau + \gamma + \widehat{\Gamma}_{\sm J_{\tau}(\mb{x})}(\mb{x})} \big) - \mathbb{E} \big( \frac{1}{ \tau + \gamma + \Gamma_{\sm J_{\tau}(\mb{x})}(\mb{x})} \big) \Big| \\
            & \leq  \Big| \sum_{s=1}^{\tau} \widehat{q}_{j_s}(\mb{x}) \Big( \mathbb{E}\big( \frac{1}{ \tau + \gamma + 1 + \widehat{\Gamma}_{\sm J_{\tau}(\mb{x})}(\mb{x}) } \big) - \mathbb{E}\big( \frac{1}{ \tau + \gamma + 1 + \Gamma_{\sm J_{\tau}(\mb{x})}(\mb{x}) } \big)  \Big) \Big| \\ 
            & \qquad + \Big| \sum_{s=1}^{\tau} \big( \widehat{q}_{j_s}(\mb{x}) - p_{j_s}(\mb{x})\big) \mathbb{E}\big( \frac{1}{ \tau + \gamma + 1 + \Gamma_{\sm J_{\tau}(\mb{x})}(\mb{x}) } \big) \Big) \Big| \\ 
            & \qquad + \gamma \Big|  \mathbb{E} \big( \frac{1}{ \tau + \gamma + \widehat{\Gamma}_{\sm J_{\tau}(\mb{x})}(\mb{x})} \big) - \mathbb{E} \big( \frac{1}{ \tau + \gamma + \Gamma_{\sm J_{\tau}(\mb{x})}(\mb{x})} \big) \Big| \\
            & \leq \sum_{s=1}^{\tau} \Big| \frac{ \mathbb{E}\big( \Gamma_{\sm J_{\tau}(\mb{x})}(\mb{x}) - \widehat{\Gamma}_{\sm J_{\tau}(\mb{x})}(\mb{x})\big)}{ (\tau + \gamma + 1)^2 } \Big| + \sum_{s=1}^{\tau} \frac{\big| \widehat{q}_{j_s}(\mb{x}) - p_{j_s}(\mb{x})\big|}{\tau + \gamma + 1 } + \gamma \Big| \frac{ \mathbb{E}\big( \Gamma_{\sm J_{\tau}(\mb{x})}(\mb{x}) - \widehat{\Gamma}_{\sm J_{\tau}(\mb{x})}(\mb{x})\big)}{ (\tau + \gamma)^2 } \Big| \\
            & \leq \sum_{s=1}^{\tau} \frac{ \|\widehat{\mb{q}}(\mb{x})- \mb{p}(\mb{x})\|_1}{ (\tau + \gamma + 1)^2 } + \sum_{s=1}^{\tau} \frac{\big| \widehat{q}_{j_s}(\mb{x}) - p_{j_s}(\mb{x})\big|}{\tau + \gamma + 1 } + \gamma \frac{ \|\widehat{\mb{q}}(\mb{x}) - \mb{p}(\mb{x})\|_1 }{ (\tau + \gamma)^2 } \\
            & \leq C_2 { \|\widehat{\mb{q}}(\mb{x}) - \mb{p}(\mb{x}) \|_1 },
        \end{align*}
        where $C_2$ is a constant that may depend on $\gamma$. Then the rest of the proof follows from the same arguments for Theorems \ref{thm:risk_bound} and \ref{thm:rate} with \Dice{} replaced by \IoU{}, and is omitted.
\end{proof}

\section{Auxiliary Lemmas}

\begin{lemma}
    \label{lem:pb_sandwich}
    Suppose $\Gamma = \sum_{j=1}^d B_j$ is a Poisson binomial random variable, and $B_j (j=1,\cdots, d)$ are independent Bernoulli trials with success probabilities $p_1, \cdots, p_d$. Then, for any $j = 1, \cdots, d$, we have
    $$
    \mathbb{P}\big( \Gamma_{\sm j} \leq l - 1 \big) \leq \mathbb{P} \big( \Gamma \leq l \big) \leq \mathbb{P}\big( \Gamma_{\sm j} \leq l \big),
    $$
    where $\Gamma_{\sm j} = \sum_{j' \neq j} B_{j'}$.
\end{lemma}
\begin{proof}
Note that $\Gamma = \Gamma_{\sm j} + B_j$, then
\begin{align*}
    \mathbb{P}( \Gamma \leq l ) & = \mathbb{P}( \Gamma_{\sm j} + B_j \leq l ) = \mathbb{P}( \Gamma_{\sm j} \leq l \mid B_j = 0 ) (1 - p_j) + \mathbb{P}( \Gamma_{\sm j} \leq l - 1 \mid B_j = 1 ) p_j \\
    & = \mathbb{P}( \Gamma_{\sm j} \leq l ) (1 - p_j) + \mathbb{P}( \Gamma_{\sm j} \leq l - 1 ) p_j,
\end{align*}
where the last equality follows from the fact that $B_j \perp B_{j'}$ for $j \neq j'$. The desirable result then follows since $\mathbb{P}( \Gamma_{\sm j} \leq l ) \geq \mathbb{P}( \Gamma_{\sm j} \leq l -1 )$.
\end{proof}

\begin{lemma}
    \label{lem:de_score}
    Suppose $\Gamma = \sum_{j=1}^d B_j$ is a Poisson binomial random variable, and $B_j (j=1,\cdots, d)$ are independent Bernoulli trials with success probabilities $p_1 \geq p_2 \geq \cdots \geq p_d$. Then, for an arbitrary positive non-decreasing sequence $\zeta_l$ $(l = 0, \cdots, d - 1)$,
    $$
    Q_j = \sum_{l=0}^{d - 1} \frac{1}{ \zeta_l } \mathbb{P}( \Gamma_{\sm j} = l )
    $$
    is a non-increasing function with respect to $j = 1, \cdots, d$.
\end{lemma}

\begin{proof}
Note that for any $j' \neq j$, $\Gamma_{\sm j} = \sum_{i \neq j} B_{i} = \sum_{i \neq j, j'} B_i + B_{j'} =: \Gamma_{\sm jj'} + B_{j'}$. Denote $\xi_{\sm jj', l} = \mathbb{P}\big( \Gamma_{\sm jj'} = l \big)$, we have
\begin{align*}
    \mathbb{P}\big( \Gamma_{\sm j} = l \big) =  (1 - p_{j'}) \xi_{\sm jj', l} + p_{j'} \xi_{\sm jj', l-1}.
\end{align*}
Then, 
\begin{align*}
    Q_j - Q_{j'} & = \sum_{l = 0}^d \frac{1}{ \zeta_l } \big(  (1 - p_{j'})\xi_{\sm jj', l} + p_{j'} \xi_{\sm jj', l-1} - (1 - p_{j})\xi_{\sm jj', l} - p_{j} \xi_{\sm jj', l-1} \big) \\
    & = ( p_j - p_{j'} ) \Big( \sum_{l = 0}^d \frac{1}{\zeta_l} \xi_{\sm jj', l} - \sum_{l = 0}^d \frac{1}{\zeta_l} \xi_{\sm jj', l-1} \Big) = ( p_j - p_{j'} ) \Big( \sum_{l = 0}^d \big( \frac{1}{\zeta_l} - \frac{1}{\zeta_{l+1}} ) \xi_{\sm jj', l} \Big),
\end{align*}
where the first equality follows from the fact that $\mathbb{P}( \Gamma_{\sm j} = d ) = 0$, and the last equality follows from the fact that $\xi_{\sm jj', l} = 0$ for $l < 0$ or $l=d$. Hence, we have $Q_j - Q_{j'}$ has the same sign with $p_j - p_{j'}$, and the desirable result then follows. 
\end{proof}

\begin{lemma}
    \label{lem:pb_bounds}
    Let $\Gamma$ be a Poisson-binomial random variable with the success probability $(p_j)_{j=1, \cdots, d}$, then for any $\tau \geq 1$, we have
    \begin{equation*}
         (\sum_{j=1}^d p_j + \tau)^{-1} \leq \mathbb{E}\big( \frac{1}{ \Gamma + \tau} \big) \leq (\frac{d+1}{d} \sum_{j=1}^d p_j + \tau - 1)^{-1}.
    \end{equation*}
\end{lemma}

\begin{proof}
According to the corollary in \cite{chao1972negative}, we have
\begin{align}
    \mathbb{E}\big( \frac{1}{ \Gamma + \tau } \big) & = \int_0^1 t^{\tau - 1} P_\Gamma(t) dt = \int_0^1 t^{\tau - 1} \big( \prod_{j=1}^d ( 1 - p_j + p_j t) \big) dt \leq \int_0^1 t^{\tau - 1} ( 1 - \bar{p} + \bar{p} t)^d dt \nonumber \\
    & = \mathbb{E}\big( \frac{1}{ \Lambda + \tau } \big) \leq (\frac{d+1}{d} \sum_{j=1}^d p_j + \tau - 1)^{-1}, \nonumber
\end{align}
where $\bar{p} = d^{-1} \sum_{j=1}^d p_j$, the first inequality follows from the inequality of arithmetic and geometric means, the last equality follows from Section 3.1 of \cite{chao1972negative}, and the last inequality follows from (25) in \cite{wooff1985bounds}. This completes the proof.
\end{proof}

\newpage

\vskip 0.2in
\bibliography{detect}

\end{document}